%% file: main.tex
\documentclass[11pt]{article} 

\pdfoutput=1

\usepackage{etoolbox}
\newtoggle{arxiv}
\toggletrue{arxiv}

\input{preamble/packages} 
\input{preamble/macros} 
\input{preamble/andrew_macros}
\input{preamble/mdp_macros_moca}

\newlength\tindent
\allowdisplaybreaks

\newcommand{\icmledit}[1]{#1}

\newcommand{\loose}{\looseness=-1}

\title{Leveraging Offline Data in Online Reinforcement Learning}
\author{Andrew Wagenmaker\footnote{University of Washington, Seattle. Email: \texttt{ajwagen@cs.washington.edu}} \and Aldo Pacchiano\footnote{Microsoft Research, New York. Email: \texttt{apacchiano@microsoft.com}}}
\date{}

\begin{document}
\maketitle

\begin{abstract}
\input{body/abstract}

\end{abstract}

\input{body/introduction}

\input{body/results}

\input{body/conclusion}

\newpage
\bibliographystyle{icml2021}
\bibliography{bibliography.bib}

\newpage
\appendix
\input{body/linear_mdp_upper_bound}

\input{body/examples}

\input{body/lower_bounds}

\end{document}

%% file: preamble/packages.tex
\usepackage[utf8]{inputenc}
\usepackage{mathrsfs}
\usepackage{microtype}
\usepackage{subfigure}
\usepackage{booktabs} 
\usepackage{longtable,makecell}
\usepackage{amsmath, amsfonts, amsthm, amssymb, graphicx}
\usepackage{mathtools,verbatim}
\usepackage{enumitem}
\usepackage{bm}
\usepackage{thm-restate}
\usepackage{natbib}
\usepackage{xspace}
\usepackage[dvipsnames]{xcolor}

\usepackage{color-edits}

\usepackage[noend]{algpseudocode}
\usepackage{algorithm}

\iftoggle{arxiv}
{
	\usepackage[scaled=.90]{helvet}
	
	\usepackage[vmargin=1.00in,hmargin=1.00in,centering,letterpaper]{geometry}
	\usepackage{fancyhdr, lastpage}

	\setlength{\headsep}{.10in}
	\setlength{\headheight}{15pt}
	\cfoot{}
	\lfoot{}
	\rfoot{\sc Page\ \thepage\ of\ \protect\pageref{LastPage}}
	\renewcommand\headrulewidth{0.4pt}
	\renewcommand\footrulewidth{0.4pt}
}

\usepackage{hyperref}
\usepackage{cleveref}
\hypersetup{colorlinks,citecolor=blue}



%% file: preamble/macros.tex
\newcommand{\poly}{\mathrm{poly}}

\numberwithin{equation}{section}

\newcommand{\op}{\mathrm{op}}

{
 \theoremstyle{plain}
      \newtheorem{asm}{Assumption}
}
\creflabelformat{asm}{#2#1#3}

\theoremstyle{plain}
\newtheorem{thm}{Theorem}
\newtheorem{lem}{Lemma}[section]

\newtheorem{prop}[thm]{Proposition}
\theoremstyle{definition}
\newtheorem{defn}{Definition}[section]

\theoremstyle{plain}
\newtheorem{theorem}{Theorem}
\newtheorem{lemma}{Lemma}[section]
\newtheorem{corollary}{Corollary}

\theoremstyle{definition}

\newtheorem{example}{Example}[section]

\newtheorem{remark}{Remark}[section]

\renewcommand{\Pr}{\mathbb{P}}
\newcommand{\Exp}{\mathbb{E}}

\newcommand{\KL}{\mathrm{KL}}

\newcommand{\TV}{\mathrm{TV}}
\newcommand{\rmd}{\mathrm{d}}

\newcommand{\tr}{\mathrm{tr}}

\newcommand{\R}{\mathbb{R}}
\newcommand{\I}{\mathbb{I}}

\DeclareMathOperator*{\argmax}{arg\,max}
\DeclareMathOperator*{\argmin}{arg\,min}

\newcommand{\simplex}{\triangle}
\newcommand{\cOtil}{\widetilde{\cO}}

\def\ddefloop#1{\ifx\ddefloop#1\else\ddef{#1}\expandafter\ddefloop\fi}
\def\ddef#1{\expandafter\def\csname bb#1\endcsname{\ensuremath{\mathbb{#1}}}}
\ddefloop ABCDEFGHIJKLMNOPQRSTUVWXYZ\ddefloop

\def\ddefloop#1{\ifx\ddefloop#1\else\ddef{#1}\expandafter\ddefloop\fi}
\def\ddef#1{\expandafter\def\csname fr#1\endcsname{\ensuremath{\mathfrak{#1}}}}
\ddefloop ABCDEFGHIJKLMNOPQRSTUVWXYZ\ddefloop
\def\ddefloop#1{\ifx\ddefloop#1\else\ddef{#1}\expandafter\ddefloop\fi}
\def\ddef#1{\expandafter\def\csname scr#1\endcsname{\ensuremath{\mathscr{#1}}}}
\ddefloop ABCDEFGHIJKLMNOPQRSTUVWXYZ\ddefloop
\def\ddefloop#1{\ifx\ddefloop#1\else\ddef{#1}\expandafter\ddefloop\fi}
\def\ddef#1{\expandafter\def\csname b#1\endcsname{\ensuremath{\mathbf{#1}}}}
\ddefloop ABCDEFGHIJKLMNOPQRSTUVWXYZ\ddefloop
\def\ddef#1{\expandafter\def\csname c#1\endcsname{\ensuremath{\mathcal{#1}}}}
\ddefloop ABCDEFGHIJKLMNOPQRSTUVWXYZ\ddefloop
\def\ddef#1{\expandafter\def\csname h#1\endcsname{\ensuremath{\widehat{#1}}}}
\ddefloop ABCDEFGHIJKLMNOPQRSTUVWXYZ\ddefloop
\def\ddef#1{\expandafter\def\csname t#1\endcsname{\ensuremath{\widetilde{#1}}}}
\ddefloop ABCDEFGHIJKLMNOPQRSTUVWXYZ\ddefloop

\def\ddefloop#1{\ifx\ddefloop#1\else\ddef{#1}\expandafter\ddefloop\fi}
\def\ddef#1{\expandafter\def\csname mat#1\endcsname{\ensuremath{\mathbf{#1}}}}
\ddefloop abcdefgijklmnopqrstuvwxyzABCDEFGHIJKLMNOPQRSTUVWXYZ\ddefloop


%% file: preamble/andrew_macros.tex
\newcommand{\inner}[2]{\langle #1, #2 \rangle}

\newcommand{\optcov}{\textsc{OptCov}\xspace}
\newcommand{\pedel}{\textsc{Pedel}\xspace}
\newcommand{\force}{\textsc{Force}\xspace}
\newcommand{\fwregret}{\textsc{FWRegret}\xspace}
\newcommand{\condcov}{\textsc{ConditionedCov}\xspace}
\newcommand{\algname}{\textsc{FTPedel}\xspace}
\newcommand{\algnamese}{\textsc{FTPedel-SE}\xspace}
\newcommand{\paradigmname}{\textsf{FineTuneRL}\xspace}

\newcommand{\thetast}{\bm{\theta}^\star}

\newcommand{\Delmin}{\Delta_{\min}}

\newcommand{\pist}{\pi^\star}
\newcommand{\bLambda}{\bm{\Lambda}}
\newcommand{\piexp}{\pi_{\mathrm{exp}}}
\newcommand{\bphi}{\bm{\phi}}
\newcommand{\Vst}{V^\star}

\newcommand{\xst}{\bm{x}^\star}
\newcommand{\bx}{\bm{x}}
\newcommand{\Toff}{T_{\mathrm{off}}}
\newcommand{\bz}{\bm{z}}

\newcommand{\bthetast}{\thetast}
\newcommand{\Thetaaltst}{\Theta_{\mathrm{alt}}^{\star}}
\newcommand{\btheta}{\bm{\theta}}
\newcommand{\ton}{t^{\mathrm{on}}}
\newcommand{\bu}{\bm{u}}
\newcommand{\Piexp}{\Pi_{\mathrm{exp}}}
\newcommand{\Pidet}{\Pi_{\mathsf{det}}}
\newcommand{\pitil}{\widetilde{\pi}}
\newcommand{\ptil}{\widetilde{p}}
\newcommand{\pihat}{\widehat{\pi}}
\newcommand{\bphitil}{\widetilde{\bphi}}
\newcommand{\be}{\bm{e}}
\newcommand{\tst}{t^\star}

\newcommand{\bthetahat}{\widehat{\bm{\theta}}}
\newcommand{\cThat}{\widehat{\cT}}
\newcommand{\bv}{\bm{v}}
\newcommand{\lammin}{\lambda_{\min}}
\newcommand{\frakD}{\mathfrak{D}}
\newcommand{\bmu}{\bm{\mu}}
\newcommand{\cEest}{\cE_{\mathrm{est}}}
\newcommand{\bphihat}{\widehat{\bphi}}
\newcommand{\bOmega}{\bm{\Omega}}

\newcommand{\bLamhat}{\widehat{\bm{\Lambda}}}

\newcommand{\bSighat}{\widehat{\bm{\Sigma}}}
\newcommand{\ihat}{\widehat{i}}
\newcommand{\bLamoff}{\bLambda_{\mathrm{off}}}
\newcommand{\tsum}{{\textstyle \sum}}
\newcommand{\Vhat}{\widehat{V}}
\newcommand{\lamminst}{\lambda_{\min}^\star}
\newcommand{\lamun}{\underline{\lambda}}
\newcommand{\lambar}{\bar{\lambda}}
\newcommand{\cEexp}{\cE_{\mathrm{exp}}}
\newcommand{\Vpi}{V^\pi}
\newcommand{\pitilst}{\widetilde{\pi}^\star}
\newcommand{\iotaalg}{\iota_0}

\newcommand{\bGamma}{\bm{\Gamma}}

\newcommand{\Ton}{T_{\mathrm{on}}}
\newcommand{\Tonbar}{\bar{T}_{\mathrm{on}}}
\newcommand{\Tontil}{\widetilde{T}_{\mathrm{on}}}

\newcommand{\Nse}{T_{\mathrm{se}}}
\newcommand{\logs}{\mathsf{logs}}
\newcommand{\frakDoff}{\mathfrak{D}_{\mathrm{off}}}
\newcommand{\Cst}{C^\star}
\newcommand{\frakM}{\mathfrak{M}}
\newcommand{\ones}{\bm{1}}
\newcommand{\bvtil}{\widetilde{\bv}}
\newcommand{\pilog}{\pi^{\mathrm{log}}}
\newcommand{\Picov}{\Pi_{\mathsf{cov}}}
\newcommand{\Ncov}{\mathsf{N}_{\mathsf{cov}}}
\newcommand{\bthetatil}{\widetilde{\btheta}}
\newcommand{\zeros}{\bm{0}}
\newcommand{\unif}{\mathrm{unif}}

\newcommand{\ba}{\bm{a}}
\newcommand{\Cpi}{C^\pi}
\newcommand{\Coto}{C_{\mathsf{o2o}}}
\newcommand{\Pilsm}{\Pi_{\mathsf{lsm}}}
\newcommand{\bw}{\bm{w}}
\newcommand{\Clot}{C_{\mathsf{lot}}}
\newcommand{\Nver}{T_{\mathsf{ver}}}
\newcommand{\epsver}{\epsilon_{\mathsf{ver}}}
\newcommand{\Noto}{T_{\mathsf{o2o}}}

\newcommand{\Kon}{T_{\mathrm{on}}}
\newcommand{\tauver}{\tau_{\mathsf{ver}}}
\newcommand{\tauuver}{\tau_{\mathsf{uver}}}
\newcommand{\Tfw}{T}

%% file: preamble/mdp_macros_moca.tex
\renewcommand{\hbar}{\bar{h}}
\newcommand{\bern}{\mathrm{Bernoulli}}

\newcommand{\Qpi}{Q^\pi}

\newcommand{\epsexp}{\epsilon_{\mathrm{exp}}}

%% file: body/abstract.tex

\iftoggle{arxiv}{
Two central paradigms have emerged in the reinforcement learning (RL) community: online RL and offline RL. In the online RL setting, the agent has no prior knowledge of the environment, and must interact with it in order to find an $\epsilon$-optimal policy. In the offline RL setting, the learner instead has access to a fixed dataset to learn from, but is unable to otherwise interact with the environment, and must obtain the best policy it can from this offline data. 
\iftoggle{arxiv}{
Practical scenarios often motivate an intermediate setting: if we have some set of offline data and, in addition, may also interact with the environment, how can we best use the offline data to minimize the number of online interactions necessary to learn an $\epsilon$-optimal policy?
}{
Practical scenarios often motivate an intermediate setting: if we have some set of offline data and may also interact with the environment, how can we best use the offline data to minimize the number of online interactions necessary to learn an $\epsilon$-optimal policy?
}

In this work, we consider this setting, which we call the \textsf{FineTuneRL} setting, for MDPs with linear structure. We characterize the necessary number of online samples needed in this setting given access to some offline dataset, and develop an algorithm, \textsc{FTPedel}, which is provably optimal, up to $H$ factors. We show through an explicit example that combining offline data with online interactions can lead to a provable improvement over either purely offline or purely online RL. Finally, our results illustrate the distinction between \emph{verifiable} learning, the typical setting considered in online RL, and \emph{unverifiable} learning, the setting often considered in offline RL, and show that there is a formal separation between these regimes. }
{
Two central paradigms have emerged in the reinforcement learning (RL) community: online RL and offline RL. In the online RL setting, the agent has no prior knowledge of the environment, and must interact with it in order to find an $\epsilon$-optimal policy. In the offline RL setting, the learner instead has access to a fixed dataset to learn from, but is unable to otherwise interact with the environment, and must obtain the best policy it can from this offline data. 
Practical scenarios often motivate an intermediate setting: if we have some set of offline data and may also interact with the environment, how can we best use the offline data to minimize the number of online interactions necessary to learn an $\epsilon$-optimal policy?
In this work, we consider this setting, which we call the \textsf{FineTuneRL} setting, for MDPs with linear structure. We characterize the necessary number of online samples needed in this setting given access to some offline dataset, and develop an algorithm, \textsc{FTPedel}, which is provably optimal, up to $H$ factors. We show through an explicit example that combining offline data with online interactions can lead to a provable improvement over either purely offline or purely online RL. Finally, our results illustrate the distinction between \emph{verifiable} learning, the typical setting considered in online RL, and \emph{unverifiable} learning, the setting often considered in offline RL, and show that there is a formal separation between these regimes.
}

%% file: body/introduction.tex

\section{Introduction}

Many important learning problems in adaptive environments can be mapped into the reinforcement learning (RL) paradigm. Recent years have seen an impressive set of results deploying RL algorithms in a variety of domains such as healthcare~\citep{yu2021reinforcement}, robotics~\citep{kober2013reinforcement}, and games~\citep{silver2016mastering}. Typically, in such RL settings, the goal of the learner is to find a \emph{policy} that maximizes the expected reward that can be obtained in a given environment. Motivated by the practical successes of RL, developing efficient approaches to \emph{policy optimization} has been a question of much interest in the machine learning community in recent years.
Broadly speaking, policy optimization algorithms can be divided into two camps: online RL, where the learner has no prior knowledge of the environment and must simply interact with it to learn a good policy, and offline RL, where the agent has access to a set of offline data from interactions with the environment, but is otherwise unable to interact with it. 

Online methods \citep{brafman2002r,azar2017minimax,auer2008near} rely on the continuous deployment of policies to collect data. These policies are computed at every step by the algorithm and make use of the information that has been collected so far. Unfortunately, online methods can be very data inefficient. In the absence of a sufficiently exploratory baseline policy, they may require an extremely large number of samples to gather sufficient data to learn a near-optimal policy. Offline RL methods~\citep{levine2020offline} mitigate some of these shortcomings. 
For example, if the offline data was originally collected by running a hand-crafted expert policy in the environment, or by running a known safe exploration strategy in a production system, it could be sufficient to learn a near-optimal policy without requiring any additional interactions with the environment. 
Unfortunately, offline methods are very sensitive to the coverage of the logged data.
Namely, the quality of the candidate policy generated by an offline algorithm will strongly depend on how well the available data covers the true optimal policy's support~\citep{zhan2022offline}. Moreover, since no further interactions with the environment are allowed, offline RL algorithms may not have any way of knowing whether or not their candidate policy is near-optimal---they cannot \emph{verify} the optimality of the policy.

In this work, we aim to bridge the gap between online and offline reinforcement learning and consider an intermediate setting, which we call \paradigmname, where the algorithm has access to an offline dataset but can also augment this for further fine-tuning via online interactions with the environment. Here, the goal is to minimize the number of \emph{online} interactions---we assume the offline data is available ``for free''.
We believe that this is often a more realistic form of training than the rigid set of assumptions imposed by pure online or offline scenarios, \icmledit{and has been the subject of much recent attention in the applied RL literature \citep{ball2023efficient,nakamoto2023cal,zheng2023adaptive}.}
This reflects the fact that in practical problems, we may indeed have access to large amounts of cheaply available offline data, 
which we wish to use to minimize the number of---much more difficult to acquire---online interactions.
In \paradigmname, then, the offline data can be used to bootstrap the online exploration procedure, reducing the online complexity, and the online interaction rounds can be used to optimally refine the candidate policy that would have resulted from only using the offline data. 

The \paradigmname paradigm setting takes some inspiration from the emerging need to devise optimal ways to fine-tune large models. Just as in the case in large language models~\citep{brown2020language} or in image generation tasks~\citep{ramesh2021zero}, accessibility to large amounts of offline data may allow for the creation of pre-trained models that must be adapted online to solve specific tasks. We hope that by introducing this RL paradigm and by laying the groundwork for analyzing the complexity of \paradigmname, more work can be dedicated to this important yet relatively unexplored feedback model in RL.

\subsection{Summary of Contributions}

In addition to introducing the \paradigmname setting, we make the following contributions:
\begin{enumerate}[leftmargin=*]
\item We introduce a new notion of concentrability coefficient in the setting of linear MDPs, which we call the \emph{Offline-to-Online Concentrability Coefficient} (\Cref{def:offline_to_online_concentrability}), defined as, for each step $h$:
\begin{align*}
\Coto^h(\frakDoff,\epsilon,T) :=  \inf_{\piexp}  \max_{\pi \in \Pilsm} \frac{\|  \bphi_{\pi,h} \|_{( T\bLambda_{\piexp}^h + \bLamoff^h )^{-1}}^2}{(\Vst_0 - V_0^\pi)^2 \vee \epsilon^2}.
\end{align*}
Here $\bphi_{\pi,h}$ denotes the ``average feature vector'' of policy $\pi$ at step $h$, $\bLamoff^h$ are the offline covariates at step $h$ for offline dataset $\frakDoff$, $\bLambda_{\piexp}^h$ denotes the expected covariates induced by policy $\piexp$ at step $h$, and $T$ denotes the number of episodes of online exploration.
\iftoggle{arxiv}{In words, t}{T}his quantifies the total coverage of our combined offline and online data, if we run for $T$ online episodes, playing the exploration policy $\piexp$ that optimally explores the regions of feature-space left unexplored by the offline data.
\item We show there exists an algorithm, \algname, which, up to lower-order terms, only collects, for each step $h$,
\begin{align*}
\min_{\Kon} \Kon \quad \text{s.t.} \quad \Coto^h(\frakDoff,\epsilon,\Kon) \le \frac{1}{\beta}
\end{align*}
online episodes---\iftoggle{arxiv}{the minimal number of online episodes necessary to ensure the offline-to-online concentrability coefficient is sufficiently bounded}{the minimal number of online episodes which ensures the offline-to-online concentrability coefficient is sufficiently small}---and returns a policy that is $\epsilon$-optimal. Furthermore, we show that this complexity is necessary---no algorithm can collect fewer online samples and return a policy guaranteed to be $\epsilon$-optimal.
\item Finally, we study the question of \emph{verifiability} in RL. We note that many existing approaches in offline RL, especially those relying on pessimism, give guarantees that are \emph{unverifiable}---an algorithm may return a near-optimal policy but it has no way of verifying it is near-optimal. We show that coverage conditions necessary for unverifiable RL are insufficient for verifiable RL, and propose stronger coverage conditions to ensure verifiability.
\end{enumerate}

\icmledit{While our work focuses on understanding the statistical complexity of online RL with access to offline data, it motivates a simple, intuitive, and broadly applicable algorithmic principle: direct online exploration to cover (relevant) regions of the feature space not covered by the offline data. Our algorithm, \algname, instantiates this principle in the setting of linear MDPs, and we hope inspires further work in leveraging offline data in online RL in more general settings.\loose 
}

\iftoggle{arxiv}{
The remainder of this work is organized ad follows. In \Cref{sec:related} we outline related work in both online and offline RL. In \Cref{sec:prelim} we formally define our problem setting, \paradigmname. Next, in \Cref{sec:results} we outline our main results, and in \Cref{sec:coverage} we consider the question of verifiability in RL.
Finally, in \Cref{sec:alg_description} we describe our algorithmic approach, and in \Cref{sec:conclusion} offer some potential directions for future work.
}{}

\iftoggle{arxiv}{
\section{Related Work}\label{sec:related}

Since our work is theoretical in nature, we restrict ourselves to drawing connections between our work and previous works in this space.  Three lines of work are relevant when placing our contributions in the context of the existing literature: Online RL, Offline RL, and works that lie at the intersection of these regimes. 

\paragraph{Online RL.} Much work has been dedicated to designing sample efficient algorithms for online RL, where the agent has access to the environment while learning. A significant portion of this work has focused on designing algorithms for tabular MDPs with finitely many states and actions~\citep{brafman2002r,azar2017minimax,auer2008near,jin2018q,dann2017unifying,kearns2002near,agrawal2017optimistic,simchowitz2019non,pacchiano2021towards,wagenmaker2022beyond}. 
Moving beyond the tabular setting~\citep{yang2020reinforcement,jin2020provably} propose sample efficient no-regret algorithms for MDPs with linear features. More precisely, the authors of~\citep{jin2020provably} prove that an optimistic modification of Least-Squares Value Iteration (LSVI) achieves a sublinear regret in Linear MDPs. In contrast with the tabular setting, the regret in these linear models is upper bounded by a function that depends polynomially on the dimension parameter even when the state space may be extremely large or even infinite. This work has been subsequently built on by a large number of additional works on RL with function approximation \citep{zanette2020frequentist,zanette2020learning,ayoub2020model,weisz2021exponential,zhou2020nearly,zhou2021provably,du2021bilinear,jin2021bellman,foster2021statistical,wagenmaker2022reward}.

Regret is not the only objective that one may consider in RL. Finding an optimal or near-optimal policy without regard to the reward acquired during exploration is in many cases a more desirable objective, and is the one considered in this work. Methods to achieve this fall under the umbrella of PAC RL, for which there is a vast literature that spans over two decades~\citep{kearns2002near,kakade2003sample}. 

Not all MDP instances are equally difficult. The majority of existing work in RL has focused on obtaining algorithms that are \emph{worst-case} optimal, scaling with the complexity of the hardest instance in a given problem class. Such guarantees, however, fail to take into account that some instances may be significantly ``easier'' than others.
While several classical works consider \emph{instance-dependent} bounds \citep{auer2008near,tewari2007optimistic}---bounds scaling with the difficulty of learning on a given problem instance---the last several years have witnessed significant progress in obtaining such guarantees, both in the tabular setting \citep{ok2018exploration,zanette2019tighter,simchowitz2019non,yang2021q,dann2021beyond,xu2021fine,wagenmaker2022beyond} as well as the function approximation setting \citep{he2020logarithmic,wagenmaker2022first,wagenmaker2022instance,wagenmaker2023instance}. Our work builds on this line of instance-dependent guarantees, in particular the work of \cite{wagenmaker2022instance}, and we aim to obtain an instance-dependent complexity in the \paradigmname setting.

\paragraph{Offline RL.} In contrast to the online setting, where the agent has full access to the environment while learning, in the offline or batch RL setting, the agent only has access to a set of logged data of interaction with the environment. Early theoretical works in offline RL focus on the setting where the offline data is assumed to have global coverage. This is the case for algorithms such as FQI \citep{munos2008finite,chen2019information} or DAgger for Agnostic MBRL~\citep{ross2012agnostic}. While these approaches are shown to find near-optimal policies, with the aid of either a least squares or a model-fitting oracle, they require that the logged data covers all states and actions.

Towards relaxing such strong coverage conditions, more recent works have developed algorithms for offline RL where the offline data has only partial coverage. This is addressed by either imposing constraints at the policy level, preventing the policy from visit states and actions where the offline data coverage is poor~\citep{fujimoto2019off,liu2020provably,kumar2019stabilizing,wu2019behavior}, or by relying on the principle of ``pessimism'' and acting conservatively when learning from offline data~\citep{kumar2020conservative,yu2020mopo,kidambi2020morel,jin2021pessimism,yin2021near,rashidinejad2021bridging}. In algorithms relying on pessimism, the error of the resulting candidate policy scales with a quantity characterizing the coverage of the offline dataset under the visitation distribution of the optimal policy. We build on this principle in this work in the \paradigmname setting, augmenting the coverage of the offline data by performing online exploration to ensure near-optimal policies are covered. 
Other works study different aspects of the problem of offline RL in linear MDPs. For example~\citep{zhang2022corruption} study the setting of learning in the presence of corrupted data, while~\citep{chang2021mitigating} studies the problem of offline imitation learning.  Recent works have also extended the study of offline RL to more general function approximation settings~\citep{jiang2020minimax,uehara2021pessimistic,xie2021bellman,chen2022offline,zhan2022offline,yin2022near}.

\paragraph{Bridging Online and Offline RL.} 
While there exist empirical works considering the setting of online RL where the learner also has some form of access to logged data \citep{rajeswaran2017learning,nair2018overcoming,hester2018deep,ball2023efficient,nakamoto2023cal,zheng2023adaptive}, to our knowledge, 
only several existing works offer formal guarantees in this setting \citep{ross2012agnostic,xie2021policy,song2022hybrid,tennenholtz2021bandits}. 
\cite{tennenholtz2021bandits} consider a linear bandit setting where they have access to offline data, but where the features of the offline data are only partially observed, a somewhat different setting than what we consider.
Both \cite{ross2012agnostic} and \cite{xie2021policy} consider the setting where the learner has access to some logging policy $\mu$ rather than a fixed set of logged data, and at the start of every episode can choose whether to play $\mu$, or to play any other online policy of their choosing. 
\icmledit{In many respects, this setting is much more akin to online RL than offline RL. All data available to the learner is collected in an online fashion, either by rolling out $\mu$ or another policy, and the sample complexity bounds} are then obtained in terms of the total number of rollouts---both of $\mu$ or alternate online policies played---and are shown to scale with the coverage of $\mu$. 
\cite{xie2021policy} prove that in a minimax sense, in this setting 
there does not exist an approach which can have a strictly better sample complexity than either using purely online RL algorithms (ignoring $\mu$), or collecting data only by playing $\mu$. 

In contrast to this setting, in this work we assume the learner is simply given access to some offline dataset which could be generated arbitrarily, rather than being given access to a logging policy, and is then able to interact with the environment in an online fashion by playing any policy it desires, ultimately using the combination of the offline data and online interactions to learn a near-optimal policy. Our goal is to minimize the number of online interactions---the offline data is ``free'', and does not count towards the total number of samples collected. In contrast to \cite{xie2021policy}, we show that there is a provable gain in combining offline data with online interactions in this setting, over either purely offline or purely online RL (\Cref{prop:ex_offline_helps}).

\newcommand{\dbi}{d_{\mathrm{bi}}}
Concurrent to this work, \cite{song2022hybrid} propose a setting very similar to ours, which they call ``Hybrid RL''.
They propose the Hybrid Q-learning algorithm (Hy-Q), a simple adaptation of fitted Q-iteration for low bilinear rank MDPs~\citep{du2021bilinear}. 
Letting $\dbi$ denote the MDP's bilinear rank, up to logarithmic factors, Hy-Q can be used to find an $\epsilon$-optimal policy in a total number of samples (online + offline) of order $\mathcal{O}\left(\max\{ (\Cst)^2, 1 \} \cdot \mathrm{poly}( H ,\dbi ) /\epsilon^2\right)$, where $\Cst$ serves as a measure of how well the offline data covers the optimal policy.
Furthermore, by having access to an offline dataset and allowing for online deployments, Hy-Q avoids the use of potentially intractable exploration strategies (such as optimism) and therefore is the first computationally efficient algorithm (assuming access to a least squares regression oracle) for low bilinear rank MDPs.

In the setting where $\Cst$ is large---the offline data does not cover a near-optimal policy---the guarantee obtained by Hy-Q could be much worse than that obtained in the purely online setting \citep{du2021bilinear}. One might hope to instead obtain a guarantee that is never worse than the purely online guarantee and, even in the setting of poor offline data coverage, that some useful information may still be extracted from the offline data. We are able to obtain a guarantee of precisely this form and quantify, in a sharp, instance-dependent manner, how the coverage of offline data reduces the online exploration needed. \icmledit{See \Cref{rem:song_comparison} for further comparison with \cite{song2022hybrid}.}

Finally, we mention the recent work \cite{xie2022role}. While this work is purely online (it does not assume access to any offline data), it shows that online guarantees can be obtained in terms of the concentrability coefficient parameter introduced in the offline RL literature, providing a bridge between the analysis techniques of offline and online RL.

}{
\section{Related Work}\label{sec:related}

Three lines of existing work are particularly relevant to our work: Online RL, Offline RL, and works that lie at the intersection of these regimes. 

\paragraph{Online RL.} Much work has been dedicated to designing sample efficient algorithms for online RL. A significant portion of this work has focused on designing algorithms for tabular MDPs with finitely many states and actions~\citep{brafman2002r,azar2017minimax,auer2008near,jin2018q,dann2017unifying,kearns2002near,agrawal2017optimistic,simchowitz2019non,pacchiano2021towards,wagenmaker2022beyond}. 
Moving beyond the tabular setting~\citep{yang2020reinforcement,jin2020provably} propose sample efficient no-regret algorithms for MDPs with linear features, work that has been subsequently built on by a large number of additional works on RL with function approximation \citep{zanette2020frequentist,zanette2020learning,ayoub2020model,weisz2021exponential,zhou2020nearly,zhou2021provably,du2021bilinear,jin2021bellman,foster2021statistical,wagenmaker2022reward}.

Not all MDP instances are equally difficult. The majority of existing work in RL has focused on obtaining algorithms that are \emph{worst-case} optimal, scaling with the complexity of the hardest instance in a given problem class. Such guarantees, however, fail to take into account that some instances may be significantly ``easier'' than others.
While several classical works consider \emph{instance-dependent} bounds \citep{auer2008near,tewari2007optimistic}---bounds scaling with the difficulty of learning on a given problem instance---the last several years have witnessed significant progress in obtaining such guarantees, both in the tabular setting \citep{ok2018exploration,zanette2019tighter,simchowitz2019non,yang2021q,dann2021beyond,xu2021fine,wagenmaker2022beyond} as well as the function approximation setting \citep{he2020logarithmic,wagenmaker2022first,wagenmaker2022instance,wagenmaker2023instance}. Our work builds on this line of work, in particular \cite{wagenmaker2022instance}, and we aim to obtain an instance-dependent guarantee in the \paradigmname setting.

\paragraph{Offline RL.} 
Early theoretical works in offline RL focus on the setting where the offline data is assumed to have global coverage. This is the case for algorithms such as FQI \citep{munos2008finite,chen2019information} or DAgger for Agnostic MBRL~\citep{ross2012agnostic}. While these approaches are shown to find near-optimal policies, with the aid of either a least squares or a model-fitting oracle, they require that the logged data covers all states and actions.

Towards relaxing such strong coverage conditions, more recent works have developed algorithms for offline RL where the offline data has only partial coverage. This is addressed by either imposing constraints at the policy level, preventing the policy from visit states and actions where the offline data coverage is poor~\citep{fujimoto2019off,liu2020provably,kumar2019stabilizing,wu2019behavior}, or by relying on the principle of ``pessimism'' and acting conservatively when learning from offline data~\citep{kumar2020conservative,yu2020mopo,kidambi2020morel,jin2021pessimism,yin2021near,rashidinejad2021bridging}. In algorithms relying on pessimism, the error of the resulting candidate policy scales with a quantity characterizing the coverage of the offline dataset under the visitation distribution of the optimal policy. 
Recent works have also extended the study of offline RL to more general function approximation settings~\citep{jiang2020minimax,uehara2021pessimistic,xie2021bellman,chen2022offline,zhan2022offline,yin2022near}, as well as in the presence of corrupted data \citep{zhang2022corruption}, or offline imitation learning \citep{chang2021mitigating}.

\paragraph{Bridging Online and Offline RL.} 
While there exist empirical works considering the setting of online RL where the learner also has some form of access to logged data \citep{rajeswaran2017learning,nair2018overcoming,hester2018deep,ball2023efficient,nakamoto2023cal,zheng2023adaptive}, to our knowledge, 
only several works offer formal guarantees in this setting \citep{ross2012agnostic,xie2021policy,song2022hybrid,tennenholtz2021bandits}. 
\cite{tennenholtz2021bandits} consider a linear bandit setting where they have access to offline data, but where the offline data features are only partially observed, a different setting than what we consider.
Both \cite{ross2012agnostic} and \cite{xie2021policy} consider the setting where the learner has access to a logging policy $\mu$ rather than a fixed set of logged data, and at the start of every episode can choose whether to play $\mu$, or to play any other online policy. 
\icmledit{In many respects, this setting is much more akin to online RL than offline RL. All data available to the learner is collected in an online fashion, either by rolling out $\mu$ or another policy, and the sample complexity bounds} are then obtained in terms of the total number of rollouts---both of $\mu$ or alternate online policies played---and are shown to scale with the coverage of $\mu$. 
\icmledit{\cite{xie2021policy} prove that in a minimax sense, in this setting 
there does not exist an approach which can have a strictly better sample complexity than either using purely online RL algorithms (ignoring $\mu$), or collecting data only by playing $\mu$. }
In contrast to this setting, in this work we assume the learner is simply given access to some offline dataset which could be generated arbitrarily, rather than being given access to a logging policy, and is then able to interact with the environment in an online fashion by playing any policy it desires, ultimately using the combination of the offline data and online interactions to learn a near-optimal policy. \icmledit{Our goal is to minimize the number of online interactions---the offline data is ``free'', and does not count towards the total number of samples collected. In contrast to \cite{xie2021policy}, we show that there is a provable gain in combining offline data with online interactions in this setting, over either purely offline or purely online RL (\Cref{prop:ex_offline_helps}).}

\newcommand{\dbi}{d_{\mathrm{bi}}}
Concurrent to this work, \cite{song2022hybrid} propose a setting similar to ours, which they call ``Hybrid RL''.
They propose the Hybrid Q-learning algorithm (Hy-Q), a simple adaptation of fitted Q-iteration for low bilinear rank MDPs~\citep{du2021bilinear}. 
Letting $\dbi$ denote the MDP's bilinear rank, up to logarithmic factors, Hy-Q can be used to find an $\epsilon$-optimal policy in a total number of samples (online + offline) of order $\mathcal{O}\left(\max\{ (\Cst)^2, 1 \} \cdot \mathrm{poly}( H ,\dbi ) /\epsilon^2\right)$, where $\Cst$ serves as a measure of how well the offline data covers the optimal policy. In the setting where $\Cst$ is large---the offline data does not cover a near-optimal policy---the guarantee obtained by Hy-Q could be much worse than that obtained in the purely online setting \citep{du2021bilinear}. One might hope to obtain a guarantee never worse than the purely online guarantee and, even in the setting of poor offline data coverage, that some useful information may still be extracted from the offline data---precisely the guarantee we obtain. 

Finally, we mention the recent work \cite{xie2022role}. While this work is purely online and does not assume access to offline data, it shows that online guarantees can be obtained in terms of the concentrability coefficient parameter introduced in the offline RL literature, providing a bridge between the analysis techniques of offline and online RL. 
}

%% file: body/results.tex

\section{Preliminaries}\label{sec:prelim}
 \paragraph{Notation.}
We let $\| \bv \|_{\bLambda}^2 = \bv^\top \bLambda \bv$. $a \vee b$ denotes $\max \{ a,b \}$. $\cS^{d-1}$ denotes the unit sphere in $d$ dimensions. $\simplex$ denotes the simplex. $\logs(\cdot)$ denotes some function which depends at most logarithmically on its arguments: $\logs(x_1,\ldots,x_n) = \sum_{i=1}^n c_i \log(e + x_i)$ for $x_i \ge 0$ and absolute constants $c_i \ge 0$.
We let $\Pr_\cM[\cdot]$ and $\Exp_\cM[\cdot]$ denote the measure and expectation induced by MDP $\cM$, and $\Pr_\pi[\cdot]$ and $\Exp_\pi[\cdot]$ the measure and expectation induced playing policy $\pi$ on our MDP. \iftoggle{arxiv}{Throughout,}{} $C$ and $c$ denote universal constants.

\paragraph{Markov Decision Processes.}
In this work we study episodic Markov Decision Processes (MDPs). In the episodic setting, an MDP is denoted by a tuple $\cM = (\cS,\cA,H,\{ P_h \}_{h=1}^H, \{ \nu_h \}_{h=1}^H)$ for $\cS$ the set of states, $\cA$ the set of actions, $H$ the horizon, $\{ P_h \}_{h=1}^H$ the probability transition kernels, and $\{ \nu_h \}_{h=1}^H$ the reward distributions, which we assume are supported on $[0,1]$. 
Each episode begins at some fixed state $s_1$. The agent then takes some action $a_1 \in \cA$, transitions to $s_2 \sim P_1(\cdot | s_1, a_1)$, and receives reward $r_1(s_1,a_1) \sim \nu_1(s_1,a_1)$. This repeats for $H$ steps at which point the episode terminates and the process restarts. We assume $\{ P_h \}_{h=1}^H$ and $\{ \nu_h \}_{h=1}^H$ are initially unknown.

\iftoggle{arxiv}{
We denote a policy by $\pi : \cS \times [H] \rightarrow \simplex_{\cA}$. We are typically interested in finding policies with large expected reward. We can quantify this via the \emph{value function}. In particular, the $Q$-value function for policy $\pi$ is defined as
\begin{align*}
\Qpi_h(s,a) = \Exp_{\pi} \left [ \sum_{h' = h}^H r_{h'}(s_{h'},a_{h'}) \mid s_h = s, a_h = a \right ].
\end{align*}
In words, $\Qpi_h(s,a)$ denotes the expected reward we will receive from playing action $a$ in state $s$ at step $h$ and then playing policy $\pi$. We can similarly define the value function in terms of the $Q$-value function as $\Vpi_h(s) = \Exp_{a \sim \pi_h(\cdot | s)}[\Qpi_h(s,a)]$---the expected reward we will receive playing policy $\pi$ from state $s$ at step $h$ on. We denote the \emph{value of a policy} by $\Vpi_0 := \Vpi_1(s_1)$, which is the total expected reward policy $\pi$ will achieve. We denote the value of the optimal policy by $\Vst_0 := \sup_\pi \Vpi_0$ and similarly denote an optimal policy---any policy $\pi$ with $\Vpi_0 = \Vst_0$---by $\pist$. For some set of policies $\Pi$, we denote the value of the best policy in $\Pi$ as $\Vst_0(\Pi) := \sup_{\pi \in \Pi} \Vpi_0$.
}{
We denote a policy by $\pi : \cS \times [H] \rightarrow \simplex_{\cA}$, and the $Q$-value function for policy $\pi$ as
\begin{align*}
\textstyle \Qpi_h(s,a) = \Exp_{\pi} \big [ \sum_{h' = h}^H r_{h'}(s_{h'},a_{h'}) \mid s_h = s, a_h = a \big ].
\end{align*}
We define the value function as $\Vpi_h(s) = \Exp_{a \sim \pi_h(\cdot | s)}[\Qpi_h(s,a)]$. We denote the \emph{value of a policy} by $\Vpi_0 := \Vpi_1(s_1)$, the total expected reward policy $\pi$ will acquire, and $\Vst_0 := \sup_\pi \Vpi_0$. We let $\pist$ denote an optimal policy---any policy with $\Vpi_0 = \Vst_0$.
}

\iftoggle{arxiv}{
\paragraph{PAC Reinforcement Learning.}
In this work we are interested primarily in the PAC (Probably Approximately Correct) RL setting. In PAC RL the goal is to find a policy $\pihat$ such that, with probability at least $1-\delta$,
\begin{align}\label{eq:pac_def}
V_0^{\pihat} \ge \Vst_0 - \epsilon. 
\end{align}
We say a policy $\pihat$ satisfying \eqref{eq:pac_def} is $\epsilon$-optimal. }{}

\subsection{Linear MDPs}
In order to allow for efficient learning in MDPs with large state spaces---where $|\cS|$ is extremely large or even infinite---further assumptions must be made on the structure of the MDP. One such formulation is the \emph{linear MDP} setting, which we consider in this work.

\begin{defn}[Linear MDPs \citep{jin2020provably}]\label{defn:linear_mdp}
We say that an MDP is a $d$-\emph{dimensional linear MDP}, if there exists some (known) feature map $\bphi(s,a) : \cS \times \cA \rightarrow \R^d$, $H$ (unknown) signed vector-valued measures $\bmu_h \in \R^d$ over $\cS$, and $H$ (unknown) reward vectors $\btheta_h \in \R^d$, such that:
\begin{align*}
P_h(\cdot | s,a) = \inner{\bphi(s,a)}{\bmu_h(\cdot)}, \quad \Exp[\nu_h(s,a)] = \inner{\bphi(s,a)}{\btheta_h}. 
\end{align*}
We will assume $\| \bphi(s,a) \|_2 \le 1$ for all $s,a$; and for all $h$, $\| |\bmu_h|(\cS) \|_2 = \| \int_{s \in \cS} | \rmd \bmu_h(s) | \|_2  \le \sqrt{d}$ and $\| \btheta_h \|_2 \le \sqrt{d}$.
\end{defn}

Note that our definition of linear MDPs allows the reward to be random\iftoggle{arxiv}{---we simply assume their means are linear}{}. While linear MDPs encompass settings such as tabular MDPs---where $\bphi(s,a)$ are then simply taken to be the standard basis vectors---they also encompass more complex settings where generalization across states is possible. Indeed, several recent works have demonstrated that complex real-world environments can be modeled as linear MDPs to allow for sample efficient learning \citep{ren2022free,zhang2022making}.

\iftoggle{arxiv}{
We introduce several additional pieces of a notation in the linear MDP setting. For some policy $\pi$, we let $\bphi_{\pi,h} := \Exp_\pi[\bphi(s_h,a_h)]$ denote the expected feature vector of policy $\pi$ at step $h$. This generalizes the state-action visitation frequencies often found in the tabular RL literature---in a tabular MDP, this definition would give $[\bphi_{\pi,h}]_{(s,a)} = \Pr_\pi[s_h = s,a_h = a]$. We also denote the average feature vector in a particular state as $\bphi_{\pi,h}(s) := \Exp_{a \sim \pi_h(\cdot |s)}[\bphi(s,a)]$, and the expected covariates at step $h$ generated by playing policy $\pi$ as $\bLambda_{\pi,h} := \Exp_\pi [ \bphi(s_h,a_h) \bphi(s_h,a_h)^\top]$. Let $\lambda_{\min,h}^\star = \sup_\pi \lammin(\bLambda_{\pi,h})$, the largest achievable minimum eigenvalue at step $h$, and $\lamminst = \min_h \lambda_{\min,h}^\star$. We will assume the following.
}{
We introduce several additional pieces of a notation in the linear MDP setting. For policy $\pi$, let $\bphi_{\pi,h} := \Exp_\pi[\bphi(s_h,a_h)]$ denote the expected feature vector at step $h$, which generalizes the state-action visitation frequencies often found in the tabular RL literature. Denote the expected covariates at step $h$ generated by playing policy $\pi$ as $\bLambda_{\pi,h} := \Exp_\pi [ \bphi(s_h,a_h) \bphi(s_h,a_h)^\top]$. Let $\lamminst = \min_h \sup_\pi \lammin(\bLambda_{\pi,h})$, the largest achievable minimum eigenvalue. We assume the following.
}

\begin{asm}[Full Rank Covariates]\label{asm:full_rank_cov}
\iftoggle{arxiv}{In our MDP,}{} $\lamminst > 0$.
\end{asm}
\noindent Note that \Cref{asm:full_rank_cov} is similar to other explorability assumptions found in the RL \iftoggle{arxiv}{with function approximation}{} literature \citep{zanette2020provably,hao2021online,agarwal2021online,wagenmaker2022instance,yin2022near}. For the remainder of this work, we assume \Cref{asm:full_rank_cov} holds for the MDP under consideration.

We will be interested in optimizing over covariance matrices \iftoggle{arxiv}{in order to reduce uncertainty in specific directions of interest. To}{in this work, and to} this end, define
\begin{align}\label{eq:cov_set_defn}
\bOmega_h := \{ \Exp_{\pi \sim \omega}[\bLambda_{\pi,h}] \ : \ \omega \in \bOmega_\pi \}
\end{align} 
for $\bOmega_\pi$ the set of all valid distributions over Markovian policies (both deterministic and stochastic). $\bOmega_h$ then denotes the set of all covariance matrices realizable by distributions over policies at step $h$.

\paragraph{Policy Cover.}
The learning approach we propose is \emph{policy-based}, in that it learns over a set of policies $\Pi$, with the goal of finding the best policy in the class. \iftoggle{arxiv}{In the PAC setting, where our goal will be to find some policy which is $\epsilon$-optimal, we must construct some $\Pi$ guaranteed to contain an $\epsilon$-optimal policy on any MDP. To this end, we consider the class of \emph{linear softmax policies}.

\begin{defn}[Linear Softmax Policy]
We say a policy $\pi$ is a \emph{linear softmax policy} with parameters $\eta$ and $\{ \bw_h\}_{h=1}^H$ if
\begin{align*}
\pi_h(a|s) = \frac{e^{\eta \inner{\bphi(s,a)}{\bw_h}}}{\sum_{a' \in \cA} e^{\eta \inner{\bphi(s,a')}{\bw_h}}}.
\end{align*}
\end{defn}

\noindent We have the following result.

\begin{prop}[Lemma A.14 of \cite{wagenmaker2022instance}]\label{prop:lsm_coverage}
Fix $\epsilon > 0$. Then there exists some choice of $\eta$ and set of parameter vectors $\cW$ such that the set of linear softmax policies defined with $\eta$ and over the set $\cW$, $\Pilsm$, is guaranteed to contain an $\epsilon$-optimal policy on any linear MDP. 
\end{prop}

Henceforth, we let $\Pilsm$ refer to the set of linear softmax policies defined in \Cref{prop:lsm_coverage}. Note that $\Pilsm$ depends on $\epsilon$, but we suppress this dependence for simplicity. Intuitively, $\Pilsm$ can be thought of as a \emph{policy-cover}---in some sense, it covers the space of policies we hope to learn over. }{
In particular, we consider the class of \emph{linear softmax policies}.

\begin{defn}[Linear Softmax Policy]
A policy $\pi$ is a \emph{linear softmax policy} with parameters $\eta$ and $\{ \bw_h\}_{h=1}^H$ if
\begin{align*}
\pi_h(a|s) = \frac{e^{\eta \inner{\bphi(s,a)}{\bw_h}}}{\sum_{a' \in \cA} e^{\eta \inner{\bphi(s,a')}{\bw_h}}}, \quad \forall s,a,h.
\end{align*}
\end{defn}

It can be shown that there exists some choice of $\eta$ and set of parameter vectors $\cW$ such that the set of linear softmax policies corresponding to $\eta$ and $\cW$ contains an $\epsilon$-optimal policy on any linear MDP. Henceforth we refer to this set of policies as $\Pilsm$.
}

\subsection{Offline Reinforcement Learning and \paradigmname}
\iftoggle{arxiv}{
In this work we are interested in the setting where we have access to some set of offline data. We will denote such a dataset by $\frakDoff = \{ (s_{h(\tau)}^\tau, a_{h(\tau)}^\tau, r_{h(\tau)}^\tau, s_{h(\tau)+1}^\tau) \}_{\tau = 1}^{\Toff}$, where here $h(\tau)$ denotes the step of the $\tau$th sample. We make the following assumption on this data.

\begin{asm}[Offline Data]\label{asm:offline_data}
Let $\frakDoff$ be some offline dataset and $\cM$ our underlying MDP. Then for each $t \le \Toff$, we have
\begin{align*}
& \Pr_{\frakDoff}[(r_{h(t)}^t,s_{h(t)+1}^t) \in A \times B \mid \{ (s_{h(\tau)}^\tau, a_{h(\tau)}^\tau, r_{h(\tau)}^\tau, s_{h(\tau)+1}^\tau) \}_{\tau = 1}^{t-1}, s_{h(t)} = s_{h(t)}^t, a_{h(t)} = a_{h(t)}^t] \\
& \qquad = \Pr_{\cM}[(r_{h(t)}(s_h,a_h),s_{h(t)+1}) \in A \times B \mid s_{h(t)} = s_{h(t)}^t, a_{h(t)} = a_{h(t)}^t ]
\end{align*}
for all $A \subseteq [0,1]$ and $B \subseteq \cS$.
\end{asm}
}{
In this work we are interested in the setting where we have access to some set of offline data. Let $z^\tau := (s_{h(\tau)}^\tau, a_{h(\tau)}^\tau, r_{h(\tau)}^\tau, s_{h(\tau)+1}^\tau)$ and denote such a dataset by $\frakDoff = \{ z^\tau \}_{\tau = 1}^{\Toff}$, where here $h(\tau)$ denotes the step of the $\tau$th sample. We make the following assumption on this data.

\begin{asm}[Offline Data]\label{asm:offline_data}
Let $\frakDoff$ be an offline dataset and $\cM$ our underlying MDP. Then for each $t \le \Toff$:
\begin{align*}
& \Pr_{\frakDoff}[(r_{h(t)}^t,s_{h(t)+1}^t) \in A \times B \mid \{ z^\tau \}_{\tau = 1}^{t-1}, s_{h(t)}^t,  a_{h(t)}^t  ] \\
& \ \ \  = \Pr_{\cM}[(r_{h(t)}(s_h,a_h),s_{h(t)+1}) \in A \times B \mid  s_{h(t)}^t,  a_{h(t)}^t ]
\end{align*}
for all $A \subseteq [0,1]$ and $B \subseteq \cS$.
\end{asm}
}

\Cref{asm:offline_data} is similar to existing assumptions on offline data found in the offline RL literature, for instance the \emph{compliance} condition of \cite{jin2021pessimism}.
\Cref{asm:offline_data} implies that the distribution of the reward and next state in $\frakDoff$ matches the distribution induces by our MDP $\cM$. However, it allows for correlations between steps $\tau$ (e.g. the data could be collected by an adaptive policy) and, additionally, does not even require that the dataset contain full trajectories. For some dataset $\frakDoff$, we define $\frakDoff^h := \cup_{\tau = 1, h(\tau) = h}^{\Toff} \{ s_{h(\tau)}^\tau, a_{h(\tau)}^\tau, r_{h(\tau)}^\tau, s_{h(\tau)+1}^\tau) \}$ 
\iftoggle{arxiv}{
the subset of $\frakDoff$ with $h(\tau) = h$, and
\begin{align*}
\bLamoff^h(\frakDoff) = \sum_{\tau = 1}^{\Toff} \I \{ h(\tau) = h \} \cdot \bphi(s_{h(\tau)}^\tau, a_{h(\tau)}^\tau) \bphi(s_{h(\tau)}^\tau, a_{h(\tau)}^\tau)^\top
\end{align*}
}{
and, for $\bphi_\tau := \bphi(s_{h(\tau)}^\tau, a_{h(\tau)}^\tau)$, 
\begin{align*}
\textstyle \bLamoff^h = \sum_{\tau = 1}^{\Toff} \I \{ h(\tau) = h \} \cdot \bphi_\tau \bphi_\tau^\top
\end{align*}
}
\iftoggle{arxiv}{
the covariates collected at step $h$. In general, we will abbreviate $\bLamoff^h := \bLamoff^h(\frakDoff)$. Finally, we recall the definition of the concentrability coefficient, $C^{\pi}$, from the offline RL literature. While various notions of concentrability have been proposed, we are primarily interested in those specialized to the setting of linear MDPs, and consider in particular the following definition given in \cite{zanette2021provable}:
\begin{align}\label{eq:offline_conc}
\Cpi(\frakDoff) := \sum_{h=1}^H \| \bphi_{\pi,h} \|_{(\bLamoff^h)^{-1}}.
\end{align}
}{
the covariates collected at step $h$. Finally, we recall the definition of the concentrability coefficient, $C^{\pi}$, from the offline RL literature \cite{zanette2021provable}:
\begin{align}\label{eq:offline_conc}
\textstyle \Cpi(\frakDoff) := \sum_{h=1}^H \| \bphi_{\pi,h} \|_{(\bLamoff^h)^{-1}}.
\end{align}
}
\iftoggle{arxiv}{
We let $\Cst(\frakDoff) := C^{\pist}(\frakDoff)$. Existing works in offline RL have shown that if $\Cst(\frakDoff)$ is bounded, it is possible to obtain a near-optimal policy via pessimism using only offline data, and that furthermore this is a necessary coverage condition (see for example \cite{zanette2021provable} or \cite{jin2021pessimism}).
}{
We let $\Cst(\frakDoff) := C^{\pist}(\frakDoff)$. Existing work has shown that having $\Cst(\frakDoff)$ bounded is a necessary and sufficient condition to obtain a near-optimal using offline data \citep{zanette2021provable,jin2021pessimism}.
}

\paragraph{Bridging Offline and Online RL.}
Given the previous definitions, we are now ready to formally define our learning setting, \paradigmname.

\paragraph{Problem Definition (\paradigmname).}
For any linear MDP $\cM$ satisfying \Cref{asm:full_rank_cov}, given access to some dataset $\frakDoff$ which satisfies \Cref{asm:offline_data} on MDP $\cM$ as well as the ability to interact online with $\cM$, return some policy $\pihat$ such that $\Pr_{\cM}[V_0^{\pihat} \ge \Vst_0 - \epsilon] \ge 1-\delta$, using as few online interactions as possible.

\section{Main Results}\label{sec:results}
\iftoggle{arxiv}{The notion of concentrability has played a key role in the offline RL literature.}{} In the setting of \paradigmname, we are interested in generalizing the notion of concentrability to account not just for the offline data available, but how this data can be augmented by online exploration to improve coverage. To this end, we introduce the following notion of \emph{offline-to-online concentrability}:

\begin{defn}[Offline-to-Online Concentrability Coefficient]\label{def:offline_to_online_concentrability}
Given step $h$, offline dataset $\frakDoff$, desired tolerance of learning $\epsilon > 0$, and number of online samples $T$, we define the \emph{offline-to-online concentrability coefficient} as:
\begin{align*}
\Coto^h(\frakDoff,\epsilon,T) :=  \inf_{\bLambda \in \bOmega_h}  \max_{\pi \in \Pilsm} \frac{\|  \bphi_{\pi,h} \|_{( T\bLambda + \bLamoff^h )^{-1}}^2}{(\Vst_0 - V_0^\pi)^2 \vee \epsilon^2}.
\end{align*}
\end{defn}

Intuitively, we can think of $\Coto^h(\frakDoff,\epsilon,T)$ as generalizing offline concentrability to the setting where we can augment the offline data by collecting $T$ samples of online data as well, thereby improving the coverage of the data. In particular, we note that the coverage term, $\|  \bphi_{\pi,h} \|_{( T\bLambda + \bLamoff^h )^{-1}}^2$, bears a direct resemblance to the offline concentrability coefficient, \eqref{eq:offline_conc}, but instead of scaling only with the offline data $\bLamoff^h$, it also scales with $T$ samples of online data---denoted by $T\bLambda$, the covariates we can collect in $T$ online interactions with the environment. Note that the online-to-offline concentrability scales with the \emph{best-case} online covariates realizable on our MDP---the best possible online data we could collect to cover our policy space given our offline data and $T$ episodes of online exploration. 

We weight this coverage term by the optimality of the policy under consideration, scaling it by the minimum of inverse squared gap for policy $\pi$, $\Vst_0 - \Vpi_0$, and $\epsilon^{-2}$. This quantifies the fact that for very suboptimal policies, we should not need to collect a significant amount of data, as they can easily be shown to be suboptimal. Note that $\Coto^h$ corresponds to a somewhat stronger notion of coverage than what has recently been considered in the offline RL literature---rather than simply covering the optimal policy, the offline-to-online concentrability scales with the coverage of \emph{every} policy, weighted by each policy's optimality. As we discuss in more detail in \Cref{sec:coverage}, this stronger notion of coverage is necessary if we care about \emph{verifiable} learning.

Existing work in the offline RL literature shows that efficient learning is possible if the concentrability coefficient is bounded. We will take a similar approach in this work, and aim to collect enough online data so that the offline-to-online concentrability coefficient is sufficiently small.
To this end, we introduce the following notion of minimal online exploration for coverage, quantifying the minimal number of online samples, $T$, that must be collected in order to ensure the offline-to-online concentrability is less than some desired threshold.

\begin{defn}[Minimal Online Samples for Coverage]
For some desired tolerance $\beta$, we define the \emph{minimal online samples for coverage} as:
\begin{align*}
\Noto^h(\frakDoff, \epsilon; \beta) := \min_T T \quad \text{s.t.} \quad \Coto^h(\frakDoff,\epsilon,T) \le \frac{1}{\beta}.
\end{align*}
\end{defn}

Note that if our goal is to bound the offline-to-online concentrability coefficient at step $h$, $\Noto^h(\frakDoff, \epsilon; \beta)$ is essentially the minimum number of online interactions that would be required to do so. \icmledit{Intuitively, this corresponds to the minimum number of online interactions needed to cover relevant regions of the feature space not sufficiently covered by the offline data.}
The following lower bound shows that this quantity is fundamental. 

\begin{theorem}\label{prop:minimax_lb}
Fix $\Toff \ge 0$ and $\epsilon > 0$. Then there exists some class of MDPs $\frakM$ and some offline data $\frakDoff$ with $|\frakDoff| = \Toff$, $\lammin(\bLamoff^h) \ge \Omega(\Toff/d), \forall h \in [H]$, such that any algorithm must collect at least 
\iftoggle{arxiv}{
\begin{align*}
\sum_{h=1}^H \Noto^h(\frakDoff, \epsilon; c \cdot dH)
\end{align*} }{
\begin{align*}
\tsum_{h=1}^H \Noto^h(\frakDoff, \epsilon; c \cdot dH)
\end{align*} }
online episodes on some instance $\cM \in \frakM$ in order to identify an $\epsilon$-optimal policy with constant probability on $\cM$, for universal constant $c$. 
\end{theorem}

\Cref{prop:minimax_lb} illustrates that there exists a setting where collecting at least $\Noto^h$ online samples is necessary if our goal is to learn an $\epsilon$-optimal policy---given tolerance $\epsilon > 0$ and offline dataset size $\Toff \ge 0$, we can construct a class of instances and offline dataset of size $\Toff$ such that any algorithm must collect at least $\sum_{h=1}^H \Noto^h(\frakDoff, \epsilon; c \cdot dH)$ online episodes on some instance in the class in order to learn an $\epsilon$-optimal policy. The construction used and proof of \Cref{prop:minimax_lb} are given in \Cref{sec:minimax_lb_pf}.

\subsection{Efficient Learning in \paradigmname}

\iftoggle{arxiv}{
While \Cref{prop:minimax_lb} shows that $\Noto^h$ is a \emph{necessary} measure of the number of samples that must be collected in order to ensure learning, the natural next question is whether $\Noto^h$ is also sufficient. As the following result illustrates, this is indeed the case.
}{
The following result shows that $\Noto^h$ is also a sufficient measure of the number of samples needed to ensure learning. 
}

\begin{theorem}\label{cor:main_complexity}
Fix $\epsilon  > 0$, and assume we have access to some offline dataset $\frakDoff$ satisfying \Cref{asm:offline_data}. Then there exists an algorithm, \algname, which with probability at least $1-\delta$ returns an $\epsilon$-optimal policy and collects at most
\begin{align*}
 \iotaalg \cdot \sum_{h=1}^H \Noto^h(\frakDoff, \epsilon; \beta) + \frac{\Clot}{\epsilon^{8/5}}
\end{align*}
online episodes, for lower-order constant $\Clot := \poly \big ( d, H, \log \frac{1}{\delta}, \frac{1}{\lamminst}, \log \frac{1}{\epsilon}, \log \Toff \big )$, $\iotaalg = \cO(\log \frac{1}{\epsilon})$, and 
\begin{align*}
\beta := d H^5 \cdot \logs \left (d, H, \Toff, \frac{1}{\lamminst}, \frac{1}{\epsilon}, \log \frac{1}{\delta} \right ) + c H^4 \cdot \log \frac{1}{\delta}.
\end{align*}
\end{theorem}

\Cref{cor:main_complexity} shows that, up to $H$ factors and lower-order terms, $\Noto^h$ is a sufficient measure for the number of online samples that must be collected, and that our algorithm, \algname, achieves this complexity. Furthermore,
\Cref{prop:minimax_lb} shows that the leading-order term in \Cref{cor:main_complexity} is unimprovable in a minimax sense (up to $H$ factors). In addition, in \Cref{sec:instance_lb} we  present an \emph{instance-dependent} lower bound for the $\epsilon = 0$ case---rather than scaling with the worst-case complexity over a class of instances, it scales with the complexity necessary on a particular instance---which holds for \emph{any} offline dataset and shows that the $\log 1/\delta$ dependence of \Cref{cor:main_complexity} is also necessary.
\icmledit{We emphasize that, while the sample complexity of \algname corresponds to realizing a stronger coverage condition than simply ensuring $\Cst(\frakDoff)$ is bounded, our lower bounds show that this stronger condition is necessary if our goal is verifiable learning.}
The proof of \Cref{cor:main_complexity} is given in \Cref{sec:upper_proof}, and a description of our algorithm \algname is given in \Cref{sec:alg_description}.

We next provide the following guarantee which shows that, even in the case when $\frakDoff = \emptyset$ or when $\frakDoff$ has poor coverage, \algname does essentially no worse than the \pedel algorithm of \cite{wagenmaker2022instance} which, up to $H$ factors and lower-order terms, is the tightest known complexity bound for online PAC RL in linear MDPs.

\begin{corollary}\label{cor:worst_case}
Regardless of $\frakDoff$, with probability at least $1-\delta$, \algname collects at most
\begin{align}\label{eq:worst_case1}
\iotaalg \beta \cdot \sum_{h=1}^H \inf_{\bLambda \in \bOmega_h}  \max_{\pi \in \Pilsm} \frac{\|  \bphi_{\pi,h} \|_{\bLambda^{-1}}^2}{(\Vst_0 - V_0^\pi)^2 \vee \epsilon^2}  + \frac{\Clot}{\epsilon^{8/5}}
\end{align}
online episodes. Furthermore, \eqref{eq:worst_case1} is always bounded by
\begin{align}\label{eq:worst_case2}
\cOtil \left (\frac{d H^5 (dH + \log 1/\delta)}{\epsilon^2} + \frac{\Clot}{\epsilon^{8/5}} \right ).
\end{align}
\end{corollary}

The main complexity measure given in \Cref{cor:worst_case} matches almost exactly the complexity measure of \pedel given in \cite{wagenmaker2022instance}, up to log factors and lower-order terms, implying that \algname loses virtually nothing from incorporating offline data, as compared to a purely online approach. Furthermore, \eqref{eq:worst_case2} shows that \algname hits the worst-case optimal online rate, up to $H$ factors and lower-order terms \citep{wagenmaker2022reward}.

\icmledit{
\begin{remark}[Comparison to \citep{song2022hybrid}]\label{rem:song_comparison}
Instantiating the bound given in \citep{song2022hybrid} in our setting, we see that, in order to find an $\epsilon$-optimal policy, they require collecting at least $\cOtil ( \frac{\max \{ (\Cst)^2, 1 \} \cdot d^3 H^5 \log 1/\delta}{\epsilon^2})$ online episodes. 
In comparison, our worst-case bound, \Cref{cor:worst_case}, improves on this complexity by a factor of $d$ (though is a factor of $H$ worse) implying that, \emph{even when we have access to no offline data}, \algname obtains a better online sample complexity than the algorithm of \citep{song2022hybrid}, up to a factor of $H$, even if the algorithm of \citep{song2022hybrid} has access to an arbitrarily large amount of offline data. 
We remark as well that the dependence on $\Cst$ can only hurt the sample complexity given in \citep{song2022hybrid}---since their complexity scales as $\max \{ (\Cst)^2, 1 \}$, their approach is unable to benefit from small $\Cst$, while if $\Cst$ is large (the offline data coverage is poor), their complexity could be significantly worse than, for example, our worst-case complexity of \Cref{cor:worst_case}. In contrast, our complexity will only improve as the coverage of the offline data improves. 
\end{remark}}

\begin{remark}[Scaling of $\Noto^h$ and $\epsilon$ Dependence]
Note that $\Noto^h(\frakDoff,\epsilon;\beta)$ will typically scale linearly in $\beta$ and, except in cases when the offline data coverage is extremely rich,
as $\cO(\frac{1}{\epsilon^2})$. In general, then, the $\frac{\Clot}{\epsilon^{8/5}}$ term will be lower-order, scaling with a smaller power of $\epsilon$. Intuitively, this term corresponds to the cost of \emph{learning to explore}---learning the set of actions that must be taken to obtain the optimal online covariates which reduce uncertainty. We leave further reducing the $\epsilon$ dependence in this term for future work. 
\end{remark}

\subsection{Leveraging Offline Data Yields a Provable Improvement}

We next show that there exist settings where complementing the offline data with online exploration yields a provable improvement over either (a) relying purely on the offline data without online exploration or (b) ignoring the offline data and using only data collected online. 

\begin{prop}\label{prop:ex_offline_helps}
Fix $\epsilon \le 1/20$. Then there exist two MDPs $\cM^1$ and $\cM^2$, and some dataset $\frakDoff$ that satisfies \Cref{asm:offline_data} on both $\cM^1$ and $\cM^2$, such that:
\begin{itemize}
\item Any algorithm which returns some policy $\pihat$ without further online exploration must have:
\begin{align*}
\max_{i \in \{ 1,2 \}} \Exp_{\frakDoff \sim \cM^i}[\Vst_0(\cM^i) - V_0^{\pihat}(\cM^i)] \ge \Omega(\sqrt{\epsilon}) .
\end{align*}
\item To identify an $\epsilon$-optimal policy on either $\cM^1$ or $\cM^2$ with constant probability, any algorithm which does not use $\frakDoff$ must collect at least $\Omega(\frac{1}{\epsilon^2})$ online samples.
\item \algname will return an $\epsilon$-optimal policy with constant probability after collecting at most $\cO(\frac{1}{\epsilon^{8/5}})$ online episodes.
\end{itemize}
\end{prop}

As \Cref{prop:ex_offline_helps} shows, we can construct a dataset which does not contain enough information itself to allow us to identify an $\epsilon$-optimal policy, but coupled with a small amount of online exploration, reduces the cost of pure online RL needed to identify an $\epsilon$-optimal policy \icmledit{by a factor of $1/\epsilon^{2/5}$. This illustrates that, for example, using offline data we can beat the standard $\Omega(1/\epsilon^2)$ online lower bounds found throughout the RL literature}. 
Furthermore, it shows that \algname is able to properly leverage this offline data to reduce the number of online samples it must collect.

\icmledit{
\begin{proof}[Proof Sketch of \Cref{prop:ex_offline_helps}]
We briefly sketch the proof of \Cref{prop:ex_offline_helps}---see \Cref{sec:improvement} for a full proof. \Cref{prop:ex_offline_helps} is proved by constructing a family of MDPs with three states and three actions: $s_0$ is a fixed starting state, from which the learner transitions to either $s_1$ or $s_2$, and the episode terminates. To identify the optimal action in $s_1$, at least $\Omega(1/\epsilon^2)$ samples are needed from each action; however, the offline dataset contains enough information from $s_1$ that the optimal action in this state can be identified from only the offline data. In contrast, in state $s_2$, the optimal action can be identified by playing each action $\Omega(1)$ times, yet the offline data only contains observations from one of the three actions in $s_2$. 
Thus, using only the offline data, the learner is unable to find the optimal action in $s_2$, and will therefore be unable to find a policy that $\epsilon$-optimal. 
However, by using the offline data, it does not need to collect any additional samples from $s_1$ (avoiding the $\Omega(1/\epsilon^2)$ samples it would otherwise need to collect from $s_1$), and must only collect a constant number of samples from $s_2$ to identify an $\epsilon$-optimal policy, reducing the sample complexity of \Cref{cor:main_complexity} to only the cost of learning-to-explore. 
\end{proof}}

\iftoggle{arxiv}{}{
\subsection{Algorithm Description}\label{sec:alg_description}
Our algorithm, \algname, is based on the \pedel algorithm of \cite{wagenmaker2022instance}. 
We provide a brief description of \algname in what follows and a full definition in \Cref{sec:upper_proof}. We refer the reader to \cite{wagenmaker2022instance} for an in-depth discussion of \pedel.  

\iftoggle{arxiv}{}{\algrenewcommand\algorithmicindent{0.8em}}
\begin{algorithm}[h]
\begin{algorithmic}[1]
\State \textbf{input:} tolerance $\epsilon$, confidence $\delta$, offline data $\frakDoff$
\State $\Pi_0 \leftarrow \Pilsm$, $\epsilon_\ell \leftarrow 2^{-\ell}$
\For{$\ell = 1,2,\ldots, \cO(\log \frac{1}{\epsilon}) $}
\For{$h=1,2,\ldots,H$}
\State Solve online experiment design (\Cref{alg:regret_data_fw}) to collect online covariates $\bLambda_{h,\ell}$ satisfying \eqref{eq:alg_description_data_cond}
\For{$\pi \in \Pi_\ell$} 
	\State $\bphihat_{\pi,h+1}^\ell \leftarrow $ estimate of feature visitation for $\pi$
\EndFor
\State $\bthetahat_h^\ell \leftarrow$ estimate of reward vector
\State $\Vhat_0^\pi \leftarrow \tsum_{h=1}^H \inner{\bphihat_{\pi,h}^\ell}{\bthetahat_h^\ell}$
	\EndFor
\State $\Pi_{\ell+1} \leftarrow  \Pi_\ell \backslash  \{ \pi \in \Pi_\ell \ : \ \Vhat_0^\pi < \sup_{\pi' \in \Pi_\ell} \Vhat_0^{\pi'} - 2\epsilon_\ell  \} $
	\If{$|\Pi_{\ell+1}| = 1$} \textbf{return} $\pi \in \Pi_{\ell+1}$\EndIf
\EndFor
\State \textbf{return} any $\pi \in \Pi_{\ell+1}$
\end{algorithmic}
\caption{\textbf{F}ine-\textbf{T}uning \textbf{P}olicy Learning via \textbf{E}xperiment \textbf{De}sign in \textbf{L}inear MDPs (\algname, informal)}
\label{alg:offline_pedel_pseudocode}
\end{algorithm}

\algname is a \emph{policy elimination}-based algorithm. It begins with some initial set of policies $\Pi_0$, and then gradually refines this set---at epoch $\ell$, maintaining a set of policies $\Pi_\ell$ which are guaranteed to be at most $\cO(\epsilon_\ell)$-suboptimal. The key property of \algname is its ability to carefully explore, directing online exploration only to regions that are both relevant to learn about the set of active policies, and that have not yet been sufficiently covered by the offline data. In particular, at step $h$ of epoch $\ell$, it aims to collect online covariates $\bLambda_{h,\ell}$ that satisfy
\begin{align}\label{eq:alg_description_data_cond}
\max_{\pi \in \Pi_\ell} \| \bphihat_{\pi,h}^\ell \|_{(\bLambda_{h,\ell} + \bLamoff^h)^{-1}}^2 \le H^2 \epsilon_\ell^2/\beta.
\end{align} 
We show that if our collected covariates satisfy \eqref{eq:alg_description_data_cond}, our estimate of the value of each $\pi \in \Pi_\ell$, $\Vhat_0^\pi$, will be within a factor of $\cO(\epsilon_\ell)$ of its true value, allowing us to safely eliminate policies more than $\cO(\epsilon_\ell)$-suboptimal. 
Note that to satisfy \eqref{eq:alg_description_data_cond} we only need to collect data in directions for which the offline data is not sufficiently rich, and that are relevant to the active policies at epoch $\ell$, $\Pi_\ell$. To efficiently achieve this we rely on an online \emph{experiment design} procedure (\Cref{alg:regret_data_fw}) originally developed in \cite{wagenmaker2022instance}, which is able to collect covariates satisfying \eqref{eq:alg_description_data_cond} at a near-optimal rate.

\icmledit{We remark that \algname is in many respects similar to the \pedel algorithm of \cite{wagenmaker2022instance}---both rely on policy elimination strategies and on the same experiment design routine to collect data. The key difference is that \algname initializes its data buffer with the available offline data, which allows it to then focus exploration on regions not covered by the offline data. We emphasize the simplicity of this modification---efficiently incorporating offline data does not require entirely new algorithmic approaches; offline data can be naturally used to warm-start online RL algorithms and speed up learning. }
}

\section{The Cost of Verifiability}\label{sec:coverage}

\iftoggle{arxiv}{
As noted, the coverage condition implied by $\Coto^h$ is stronger than that required by many recent offline RL results. 
Rather than simply controlling the coverage of the data over the optimal policy, $\Coto^h(\frakDoff,\epsilon,T)$ scales with the coverage over \emph{every} policy in $\Pilsm$, albeit with the coverage scaled by the policy's suboptimality. This stronger coverage arises due to the difference between \emph{verifiable} learning and \emph{unverifiable} learning. Informally, an algorithm is \emph{verifiable} if, upon termination, it can \emph{guarantee} that with probability at least $1-\delta$ the returned policy is $\epsilon$-optimal. 
We contrast this with algorithms that are \emph{unverifiable}---though they may return a near-optimal policy, they cannot guarantee it is near optimal. 

In online RL, existing work typically considers the verifiable setting, with a goal of deriving a PAC algorithm. In contrast, though it is rarely explicitly stated, the offline RL literature has more recently focused on unverifiable complexity results. For example, most existing guarantees on pessimistic algorithms are unverifiable. While a pessimistic algorithm may output a policy that is $\epsilon$-optimal, it cannot verify this is the case---if the offline data happens to cover a near-optimal policy, they will return a near-optimal policy, but the data may not have enough coverage to verify that the policy is actually near optimal. 

Towards making this distinction formal, we introduce the following definitions of verifiable and unverifiable reinforcement learning, \icmledit{inspired by work in verifiable vs unverifiable learning in the bandit setting \citep{katz2020true}.}
}{
As noted, the coverage condition implied by $\Coto^h$ is stronger than that required by many recent offline RL results. This stronger coverage arises due to the difference between \emph{verifiable} learning and \emph{unverifiable} learning. Informally, an algorithm is \emph{verifiable} if, upon termination, it can \emph{guarantee} that with probability at least $1-\delta$ the returned policy is $\epsilon$-optimal. 
We contrast this with algorithms that are \emph{unverifiable}---though they may return a near-optimal policy, they cannot guarantee it is near optimal. 

In online RL, existing work typically considers the verifiable setting. In contrast, though it is rarely explicitly stated, the offline RL literature has more recently focused on unverifiable complexity results---an algorithm may output an $\epsilon$-optimal policy, but often cannot verify this is the case, unless the data coverage is sufficiently rich. Towards making this distinction formal, we introduce the following definitions of verifiable and unverifiable reinforcement learning, \icmledit{inspired by work in verifiable vs. unverifiable learning in the bandit setting \citep{katz2020true}.}
}

\begin{defn}[Verifiable RL]
We say an algorithm is $(\epsilon,\delta)$-PAC verifiable over some class of MDPs $\frakM$ if, for any MDP $\cM \in \frakM$, it terminates after $\tauver$ episodes and returns some policy $\pihat$ such that $\Pr_{\cM}[V_0^{\pihat} \ge \Vst_0 - \epsilon] \ge 1-\delta$. We let $\Exp_\cM[\tauver]$ denote the expected $(\epsilon,\delta)$-verifiable sample complexity on $\cM$. 
\end{defn}

Note that the above definition of verifiable learning coincides with the standard definition of the PAC RL problem given in online RL. We contrast this with the following definition of unverifiable learning. 

\begin{defn}[Unverifiable RL]
Consider an algorithm which at each step $k$ outputs some $\pihat_k$. Then we say that this algorithm has expected unverifiable sample complexity $\Exp_{\cM}[\tauuver]$ on instance $\cM$ if $\tauuver$ is the minimum stopping time such that $\Pr_{\cM}[\forall k \ge \tauuver \ : \ V_0^{\pihat_k} \ge \Vst_0 - \epsilon ] \ge 1 - \delta$. 
\end{defn}

\iftoggle{arxiv}{
\icmledit{Note that implicit in both definitions is that the algorithm \emph{learns}. It is possible that, according to these definitions, an algorithm is neither verifiable or unverifiable with finite sample complexity if, for example, it always outputs policies that are more that $\epsilon$-suboptimal no matter how many samples it collects.}

As a simple example illustrating the distinction between these settings, consider a multi-armed bandit instance where every arm has value at most a factor of $\epsilon$ from the value of the optimal arm. In this instance, \emph{any} policy returned is $\epsilon$-optimal, so an unverifiable algorithm does not need to take any samples from the environment in order to return a policy which is $\epsilon$-optimal---in the above definition $\pihat_k$ will automatically be $\epsilon$-optimal for each $k$, so $\Exp_{\cM}[\tauuver] = 1$. However, while the returned policy may be $\epsilon$-optimal, clearly the learner cannot guarantee that this is the case without further knowledge of the environment.
In contrast, a verifiable learner must sample each arm sufficiently many times in order to verify it is in fact $\epsilon$-optimal. The following example makes this distinction formal.
}{
\icmledit{Note that implicit in both definitions is that the algorithm \emph{learns}. It is possible that, according to these definitions, an algorithm is neither verifiable or unverifiable with finite sample complexity if, for example, it always outputs policies that are more that $\epsilon$-suboptimal no matter how many samples it collects.}
As a simple example illustrating verifiability, we consider a multi-armed bandit instance.
}

\begin{example}[Data Coverage in Multi-Armed Bandits]\label{ex:mab_coverage}
Fix $\epsilon > 0$ and consider an $A$-armed multi-armed bandit where arm 1 is optimal, and has a mean of $\mu_1 = 1$, and every other arm is suboptimal and has a mean $\mu_2 = 1 -  3 \epsilon$. Let $\frakDoff$ be a dataset containing $\poly(1/\epsilon,\log 1/\delta)$ samples from arm 1, but 0 samples from arms 2 to $A$. 

\iftoggle{arxiv}{
Note that in this example we cannot simply return an arbitrary arm, as it may be $3\epsilon$-suboptimal. However, we have $\Cst(\frakDoff) \le \frac{1}{\poly(\frac{1}{\epsilon}, \log \frac{1}{\delta})}$, so by standard guarantees on pessimistic algorithms, applying pessimism will yield a policy $\pihat$ that is $\epsilon$-optimal. However, as the following result shows, while $\frakDoff$ is sufficient to identify the best arm in an \emph{unverifiable} fashion, it is not able to identify the best arm in a \emph{verifiable} fashion:
}{
In this example we have $\Cst(\frakDoff) \le \frac{1}{\poly(\frac{1}{\epsilon}, \log \frac{1}{\delta})}$, so by standard offline RL guarantees, $\frakDoff$ is sufficient to obtain a policy $\pihat$ that is $\epsilon$-optimal. However, as the following result shows, while $\frakDoff$ is sufficient in an \emph{unverifiable} fashion, it is not able to identify the best arm in a \emph{verifiable} fashion:
}

\begin{prop}\label{prop:verification_lb}
In the instance of \Cref{ex:mab_coverage}, any $(\epsilon,\delta)$-PAC verifiable algorithm must collect at least $\Omega(\frac{1}{\epsilon^2} \cdot \log \frac{1}{\delta} )$ additional samples from every arm 2 to $A$.
\end{prop}

\Cref{prop:verification_lb} shows that in order to provide a \emph{guarantee} that $\pihat$ is $\epsilon$-optimal, we need to collect a potentially large number of additional samples. This is very intuitive---if we have no samples from arms $2$ to $A$ in $\frakDoff$, we have no way of knowing whether or not they are better or worse than arm 1. To guarantee that they are in fact worse than arm 1, we need to sample them sufficiently many times to show that they are suboptimal. Note then that in this example, the number of samples needed for verifiable learning is at least a factor of $A$ larger than that needed for unverifiable learning.
\end{example}

\Cref{ex:mab_coverage} illustrates that simply covering the optimal policy---the condition typically given for pessimistic learning---is not sufficient for verifiable learning,
and motivates the coverage condition of $\Coto^h$. Next, we provide a \emph{sufficient} condition for offline data coverage to guarantee that verifiable learning is possible.

\begin{theorem}\label{prop:offline_verifiability}
Assume that our offline data $\frakDoff$ satisfies, for each $h \in [H]$:
\begin{align}\label{eq:offline_verifiability}
\max_{\pi \in \Pilsm} \frac{\| \bphi_{\pi,h} \|_{(\bLamoff^h)^{-1}}^2}{(\Vst_0 - \Vpi_0)^2 \vee \epsilon^2} \le \frac{1}{\beta}, \quad \forall h \in [H]
\end{align}
and $\min_{h \in [H]} \lammin(\bLamoff^h) \ge   \frac{d^2}{H^2} \cdot \beta$, for 
\iftoggle{arxiv}{
\begin{align*}
\beta := d H^5 \cdot \logs \left (d, H, \Toff, \frac{1}{\epsilon}, \log \frac{1}{\delta} \right ) + c \log \frac{1}{\delta}.
\end{align*}}{
\begin{align*}
\beta := d H^5 \cdot \logs \left (d, H, \Toff, \tfrac{1}{\epsilon}, \log \tfrac{1}{\delta} \right ) + c \log \tfrac{1}{\delta}.
\end{align*}
}
Then there exists an $(\epsilon,\delta)$-PAC verifiable algorithm that returns an $\epsilon$-optimal policy with probability at least $1-\delta$ using only $\frakDoff$.
\end{theorem}

We prove \Cref{prop:offline_verifiability} and state the $(\epsilon,\delta)$-PAC verifiable algorithm that realizes it in \Cref{sec:verifiability_pfs}. \Cref{prop:offline_verifiability} applies in the purely offline setting, giving a condition when verifiable learning is possible. The condition given in \Cref{prop:offline_verifiability}, \eqref{eq:offline_verifiability}, essentially requires that every policy is covered by the offline data. However, it does not require that every policy is covered uniformly well---policies that are very suboptimal need significantly less coverage than near-optimal policies. While this is stronger than only requiring $\Cst(\frakDoff)$ is bounded, it is a much weaker condition than the uniform concentrability condition---that $\max_\pi \Cpi(\frakDoff)$ is bounded---required by many works \citep{antos2008learning,chen2019information,xie2021batch}. 
\iftoggle{arxiv}{Note that \Cref{prop:offline_verifiability} states that there \emph{exists} some procedure which is $(\epsilon,\delta)$-PAC verifiable given that \eqref{eq:offline_verifiability} holds---it does not state that \eqref{eq:offline_verifiability} is a sufficient condition for verifiability for an \emph{arbitrary} algorithm. The procedure that realizes \Cref{prop:offline_verifiability} is a variant of the \algname algorithm, but relies only on the offline data without online exploration.
}{
Our lower bound, \Cref{prop:minimax_lb}, shows that, in a minimax sense,  \eqref{eq:offline_verifiability}, is not in general improvable for verifiable learning from purely offline data. 

 }

\iftoggle{arxiv}{
Our lower bound, \Cref{prop:minimax_lb} shows that, in a minimax sense, the former condition in \eqref{eq:offline_verifiability} is necessary for verifiable learning. In order to have $\Noto^h(\frakDoff,\epsilon; c \cdot dH)=0$---in order for the offline dataset to be expressive enough to identify an $\epsilon$-optimal policy without further online exploration---we must have a sufficiently rich offline dataset so that 
\begin{align*}
\Coto^h(\frakDoff,\epsilon,0) =  \inf_{\bLambda \in \bOmega_h}  \max_{\pi \in \Pilsm} \frac{\|  \bphi_{\pi,h} \|_{(  \bLamoff^h )^{-1}}^2}{(\Vst_0 - V_0^\pi)^2 \vee \epsilon^2} \le \frac{1}{c \cdot d H},
\end{align*}
the condition required by \eqref{eq:offline_verifiability}, up to $H$ and log factors. We offer an additional instance-dependent lower bound in \Cref{sec:instance_lb} which shows that, in the $\epsilon = 0$ case, the $\log 1/\delta$ in \eqref{eq:offline_verifiability} is also necessary. \iftoggle{arxiv}{We believe the latter condition in \eqref{eq:offline_verifiability}, $\lammin(\bLamoff^h) \ge   \frac{d^2}{H^2} \cdot \beta$, is not fundamental and is an artifact of our analysis, but leave refining this for future work.
}{}
}{

}

\icmledit{We emphasize here that, in situations where the conditions of \Cref{prop:offline_verifiability} are not met, unverifiable learning may be more appropriate. \Cref{prop:offline_verifiability} simply gives a sufficient coverage condition in settings where the goal is verifiable learning. }

\subsection{Verifying the Performance of a Policy}

We turn now to a related question on the cost of verification---given some policy $\pihat$, and some data $\frakDoff$ which may or may not have been used to obtain $\pihat$, what is the cost of verifying whether or not $\pihat$ is $\epsilon$-optimal? This setting can model, for example, scenarios in which some prior information about the system may be available, allowing us to obtain a guess at a near-optimal policy, but where we wish to verify that this is indeed the case before deploying the policy in the wild. To facilitate learning in this setting, we make the following assumption on the policy we wish to verify, $\pihat$. 

\begin{asm}\label{asm:policy_set_cover}
Assume that $\pihat \in \Pi$ for some $\Pi$ which can be efficiently covered in the sense that, for any $\gamma > 0$, there exists some set $\Picov^\gamma \subseteq \Pi$ with cardinality bounded by $\Ncov(\Pi,\gamma)$ such that for any $\pi \in \Pi$, there exists $\pitil \in \Picov$ satisfying:
\iftoggle{arxiv}{
\begin{align*}
\| \bphi_{\pi,h}(s) - \bphi_{\pitil,h}(s) \|_2 \le \gamma, \quad \forall s, h.
\end{align*}
}{
$\| \bphi_{\pi,h}(s) - \bphi_{\pitil,h}(s) \|_2 \le \gamma$, for all $s,h$.
}
\end{asm}

\iftoggle{arxiv}{
As we show in \Cref{sec:policy_verification_pfs}, this assumption is met by standard policy classes. For example, the class of linear-softmax policies will satisfy this with $\log \Ncov(\Pilsm,\gamma) \le \cOtil(dH^2 \cdot \log \frac{1}{\gamma})$. Under this assumption, we obtain the following result on the cost of policy verification. 
}{
As we show in \Cref{sec:policy_verification_pfs}, this assumption is met by standard policy classes, for example the class of linear-softmax policies. We obtain the following result.
}

\iftoggle{arxiv}{
\begin{corollary}\label{cor:verification_complexity}
Fix $\epsilon  > 0$, and assume we have access to some offline dataset $\frakDoff$ satisfying \Cref{asm:offline_data}, and some policy $\pihat$ satisfying \Cref{asm:policy_set_cover}. 
Then there exists some algorithm which with probability at least $1-\delta$ will verify whether or not $\pihat$ is $\epsilon$-optimal and collect at most
\begin{align*}
 \iotaalg \cdot \sum_{h=1}^H \Nver^h(\frakDoff, \epsver; \beta, \pihat) + \frac{\Clot}{\epsver^{8/5}}
\end{align*}
online episodes, for $\epsver = \epsilon \vee (\Vst_0 - V_0^{\pihat})$, and 
\begin{align*}
\Nver^h(\frakDoff, \epsver; \beta, \pihat) := \min_T T \quad \text{s.t.} \quad \inf_{\bLambda \in \bOmega_h}  \max_{\pi \in \Pilsm \cup \{ \pihat \}} \frac{\|  \bphi_{\pi,h} \|_{( N\bLambda + \bLamoff^h )^{-1}}^2}{(\Vst_0 - V_0^\pi)^2 \vee \epsver^2} \le \frac{1}{\beta}
\end{align*}
and $\beta := d H^5 \cdot \logs \left (d, H, \Toff, \frac{1}{\lamminst}, \frac{1}{\epsilon}, \log \frac{1}{\delta} \right ) + c \log \frac{\Ncov(\Pi,\gamma)}{\delta}$ with
$\gamma = \poly(d, H, \Toff, \frac{1}{\lamminst}, \frac{1}{\epsilon}, \log \frac{1}{\delta} )$.
\end{corollary}
}{
\begin{corollary}\label{cor:verification_complexity}
Fix $\epsilon  > 0$, and assume we have access to some offline dataset $\frakDoff$ satisfying \Cref{asm:offline_data}, and some policy $\pihat$ satisfying \Cref{asm:policy_set_cover}. 
Then there exists some algorithm which with probability at least $1-\delta$ will verify whether or not $\pihat$ is $\epsilon$-optimal and collect at most
\begin{align*}
 \iotaalg \cdot \tsum_{h=1}^H \Nver^h(\frakDoff, \epsver; \beta, \pihat) + \tfrac{\Clot}{\epsver^{8/5}}
\end{align*}
online episodes, for $\epsver = \epsilon \vee (\Vst_0 - V_0^{\pihat})$, $\Nver^h(\frakDoff, \epsver; \beta, \pihat)$ the value of the optimization:
\begin{align*}
& \min_T T \quad \text{s.t.} \quad \inf_{\bLambda \in \bOmega_h}  \max_{\pi \in \Pilsm \cup \{ \pihat \}} \frac{\|  \bphi_{\pi,h} \|_{( T\bLambda + \bLamoff^h )^{-1}}^2}{(\Vst_0 - V_0^\pi)^2 \vee \epsver^2} \le \frac{1}{\beta}, \\
& \beta := d H^5 \logs  (d, H, \Toff, \tfrac{1}{\lamminst}, \tfrac{1}{\epsilon}, \log \tfrac{1}{\delta}  ) + c H^4 \log \tfrac{\Ncov(\Pi,\gamma)}{\delta},
\end{align*}
and
$\gamma = \poly(d, H, \Toff, \frac{1}{\lamminst}, \frac{1}{\epsilon}, \log \frac{1}{\delta} )$.\loose
\end{corollary}
}

\Cref{cor:verification_complexity} shows that it is possible to verify the quality of a policy with complexity similar to that given in \Cref{cor:main_complexity}, yet scaling with $\epsver$ rather than $\epsilon$. This difference in $\epsilon$ dependence arises because, if $\pihat$ is very suboptimal, we only need to learn the performance of $\pihat$ up to a tolerance that is in the order of its policy gap. Note that $\epsver$ does not need to be known by the algorithm in advance---our procedure is able to adapt to the value of $\epsver$.

\iftoggle{arxiv}{

\section{Algorithm Description}\label{sec:alg_description}
Our algorithm, \algname, is based off of the recently proposed \pedel algorithm of \cite{wagenmaker2022instance}. We provide a brief description of this algorithm in \Cref{alg:offline_pedel_pseudocode} and a full definition in \Cref{sec:upper_proof}. We refer the interested reader to \cite{wagenmaker2022instance} for an in-depth discussion of \pedel.  

\begin{algorithm}[h]
\begin{algorithmic}[1]
\State \textbf{input:} tolerance $\epsilon$, confidence $\delta$, offline data $\frakDoff$
\State $\Pi_1 \leftarrow \Pilsm$, $\bphihat^{1}_{\pi,1} \leftarrow \Exp_{a \sim \pi_1(\cdot | s_1)} [ \bphi(s_1,a)], \forall \pi \in \Pi$
\For{$\ell = 1,2,\ldots, \lceil \log \frac{4}{\epsilon} \rceil$}
\State $\epsilon_\ell \leftarrow 2^{-\ell}$, $\beta_\ell \leftarrow \cOtil \left (d H^4 \cdot \logs  (d, H, \Toff, \frac{1}{\lamminst}, \frac{1}{\epsilon}, \log \frac{1}{\delta}  ) + \log \frac{1}{\delta} \right )$
\For{$h=1,2,\ldots,H$}
\State Run \optcov (\Cref{alg:regret_data_fw}) for $K_{h,\ell}$ episodes to collect data $\frakD_{h,\ell}$ such that:
\begin{align*}
			\max_{\pi \in \Pi_{\ell}}  &\| \bphihat_{\pi,h}^\ell \|_{(\bLambda_{h,\ell} + \bLamoff^h)^{-1}}^2  \le \epsilon_\ell^2/\beta_\ell \quad \text{for} \quad \bLambda_{h,\ell} \leftarrow \tsum_{\tau = 1}^{K_{h,\ell}} \bphi_{h,\tau} \bphi_{h,\tau}^\top + d^{-1} I, \bphi_{h,\tau} = \bphi(s_{h,\tau},a_{h,\tau})
		\end{align*}
\For{$\pi \in \Pi_\ell$} \hfill {\color{blue} \texttt{// Estimate feature-visitations for active policies}}
	\State $\bphihat_{\pi,h+1}^\ell \leftarrow \Big ( \sum_{(s_h,a_h,r_h,s_{h+1}) \in \frakDoff^h \cup \frakD_{h,\ell}} \bphi_{\pi,h+1}(s_{h+1}) \bphi(s_h,a_h)^\top (\bLambda_{h,\ell} + \bLamoff^h)^{-1} \Big ) \bphihat_{\pi,h}^\ell$
\EndFor
\Statex \hspace{1.5em} {\color{blue} \texttt{// Estimate reward vectors}}
\State $\bthetahat_h^\ell \leftarrow (\bLambda_{h,\ell} + \bLamoff^h)^{-1} \sum_{(s_h,a_h,r_h,s_{h+1}) \in \frakDoff^h \cup \frakD_{h,\ell}} \bphi(s_h,a_h) r_h$
	\EndFor
	\Statex {\color{blue} \texttt{// Remove provably suboptimal policies from active policy set}}
	\State Update $\Pi_\ell$:
	\begin{align*}
	\Pi_{\ell+1} \leftarrow  \Pi_\ell \backslash \Big \{ \pi \in \Pi_\ell \ : \ \Vhat_0^\pi < \sup_{\pi' \in \Pi_\ell} \Vhat_0^{\pi'} - 2\epsilon_\ell \Big \} \quad \text{for} \quad \Vhat_0^\pi := \tsum_{h=1}^H \inner{\bphihat_{\pi,h}^\ell}{\bthetahat_h^\ell}
	\end{align*}
	\If{$|\Pi_{\ell+1}| = 1$} \textbf{return} $\pi \in \Pi_{\ell+1}$\EndIf
\EndFor
\State \textbf{return} any $\pi \in \Pi_{\ell+1}$
\end{algorithmic}
\caption{\textbf{F}ine-\textbf{T}uning \textbf{P}olicy Learning via \textbf{E}xperiment \textbf{De}sign in \textbf{L}inear MDPs (\algname)}
\label{alg:offline_pedel_pseudocode}
\end{algorithm}

\algname is a \emph{policy elimination}-based algorithm. It begins by initializing some set of policies, $\Pi_0$, to the set of linear softmax policies, $\Pilsm$. It then proceeds in epochs, gradually refining this set at each epoch---at epoch $\ell$, maintaining some set of policies $\Pi_\ell$ which are guaranteed to be at most $\cO(\epsilon_\ell)$-suboptimal. The key property of \algname is its ability to carefully explore, directing online exploration only to regions that are both relevant to learn about the set of active policies and that have not yet been sufficiently covered by the offline data. In particular, at step $h$ of epoch $\ell$, it aims to collect online covariates $\bLambda_{h,\ell}$ that satisfy
\begin{align}\label{eq:alg_description_data_cond}
\max_{\pi \in \Pi_\ell} \| \bphihat_{\pi,h}^\ell \|_{(\bLambda_{h,\ell} + \bLamoff^h)^{-1}}^2 \le \frac{\epsilon_\ell^2}{\beta_\ell}.
\end{align} 
We show that, if we have obtained covariates that satisfy \eqref{eq:alg_description_data_cond}, this is sufficient to ensure that we can accurately estimate the feature-visitations for policies in $\Pi_{\ell}$ at stage $h+1$, $\bphihat_{\pi,h+1}^\ell$. Note that to satisfy \eqref{eq:alg_description_data_cond} we only need to collect data in directions for which the offline data is not sufficiently rich, and that are relevant to the active policies at epoch $\ell$, $\Pi_\ell$. To efficiently achieve this we rely on the \optcov procedure (\Cref{alg:regret_data_fw}), originally developed in \cite{wagenmaker2022instance}. \optcov is able to collect covariates $\bLambda_{h,\ell}$ satisfying \eqref{eq:alg_description_data_cond} in only roughly
\begin{align*}
\min_T T \quad \text{s.t.} \quad \inf_{\bLambda \in \bOmega_h} \max_{\pi \in \Pi_\ell}  \| \bphihat_{\pi,h}^\ell \|_{(T \bLambda + \bLamoff^h)^{-1}}^2 \le  \frac{\epsilon_\ell^2}{\beta_\ell}
\end{align*}
episodes of exploration. Note that this rate is essentially unimprovable if we wish to collect data satisfying \eqref{eq:alg_description_data_cond}, as it quantifies the minimum number of online samples necessary to satisfy \eqref{eq:alg_description_data_cond} if we play the optimal covariates realizable on our MDP.
Achieving this requires a careful experiment-design based exploration strategy.

After collecting this data, \algname then produces an estimate of the feature-visitation for each active policy at step $h+1$, $\bphihat_{\pi,h+1}^\ell$, as well as the reward vector at step $h$, $\bthetahat_{h}^\ell$. Note that the true value of a policy $\pi$ can be written as
\begin{align*}
\Vpi_0 = \sum_{h=1}^H \inner{\bphi_{\pi,h}}{\btheta_h}.
\end{align*}
Inspired by this, \algname forms an estimate of the value of each active policy using its estimates of $\bphi_{\pi,h}$ and $\btheta_h$, $\Vhat_0^\pi := \tsum_{h=1}^H \inner{\bphihat_{\pi,h}^\ell}{\bthetahat_h^\ell}$, and eliminates policies provably suboptimal. By satisfying \eqref{eq:alg_description_data_cond} at each step, we can show that the $|\Vhat_0^\pi - \Vpi_0 | \le \cO(\epsilon_\ell)$, allowing us to eliminate policies that are more than $\cO(\epsilon_\ell)$-suboptimal. Repeating this procedure until $\ell = \cO(\log \frac{1}{\epsilon})$, we guarantee that only $\epsilon$-optimal policies remain. 

\icmledit{We remark that \algname is in many respects similar to the \pedel algorithm of \cite{wagenmaker2022instance}---both rely on policy elimination strategies and on the same experiment design routine to collect data. The key difference is that \algname initializes its data buffer with the available offline data, which allows it to then focus exploration on regions not covered by the offline data. We emphasize the simplicity of this modification---efficiently incorporating offline data does not require entirely new algorithmic approaches; offline data can be naturally used to warm-start online RL algorithms and speed up learning. }
}{}

%% file: body/conclusion.tex

\section{Conclusion}\label{sec:conclusion}

\iftoggle{arxiv}{
This work takes a first step towards understanding the statistical complexity of online RL when the learner is given access to a set of logged data. We introduce the \paradigmname setting and develop matching upper and lower bounds on the number of online episodes needed to obtain an $\epsilon$-optimal policy given access to an offline dataset, and an algorithm, \algname, able to achieve it. We believe our work opens several interesting directions for future work.

\paragraph{Improving \algname.}
\algname inherits several weakness from the original \pedel algorithm on which it is based, namely computational inefficiency and the dependence on $\lamminst$. \pedel requires enumerating a policy class, which in general will be exponentially large in $d$ and $H$, rendering it computationally infeasible for all but the smallest problems. In addition, it requires that our MDP has full rank covariates, \Cref{asm:full_rank_cov}, and our guarantee scales with largest achievable minimum eigenvalue, $\lamminst$, albeit only logarithmically in the leading-order term. Developing a computationally efficient algorithm that achieves the same complexity as given in \Cref{cor:main_complexity} but free of $\lamminst$ is an exciting direction towards making near-optimal offline-to-online RL practical.

In addition, \Cref{cor:main_complexity} exhibits a lower-order $\cO(\frac{1}{\epsilon^{8/5}})$ term, arising from the complexity required to learn to explore. While a more careful analysis can refine this term to depend explicitly on the coverage of the offline data---in some cases removing the $\cO(\frac{1}{\epsilon^{8/5}})$ scaling entirely---in general it is not clear how to remove this term completely. We believe removing this completely may require new algorithmic techniques, which we leave for future work.

\paragraph{Verifiability in RL.}
Our results point to an important distinction between verifiable and unverifiable RL, and show that verifiable RL can be significantly more difficult than unverifiable RL. We are not the first to observe this separation: \cite{katz2020true} show that in the multi-armed bandit setting the cost of unverifiable learning is smaller than verifiable learning, and very recently \cite{tirinzoni2022optimistic} show a similar result in online RL. 
We believe a more thorough exploration of verifiability in RL is an interesting direction for future work with many open questions. For example, \cite{tirinzoni2022optimistic} construct a particular example where the verifiable complexity is much larger than the unverifiable complexity---can this effect be shown to hold more broadly?

\paragraph{Towards a General Theory of Offline-to-Online RL.}
While our work answers the question of how to leverage offline data in online RL, it applies only in PAC setting and requires the underlying MDP structure to be linear.
Recently, both the offline and online RL literature has devoted significant attention to RL with more general function approximation. Developing algorithms that address the offline-to-online setting with general function approximation while preserving our optimality guarantees is of much interest. In addition, understanding how to best leverage offline data in different RL settings---for example, regret minimization---remains an open question and is an exciting direction for future work. 

\subsection*{Acknowledgements}
AW is supported in part by NSF TRIPODS II-DMS 2023166. AP would like to thank the support of the Broad Institute of MIT and Harvard.
}{
This work takes a first step towards understanding the statistical complexity of online RL when the learner is given access to a set of logged data. We introduce the \paradigmname setting and develop matching upper and lower bounds on the number of online episodes needed to obtain an $\epsilon$-optimal policy given access to an offline dataset, and an algorithm, \algname, able to achieve it. We believe our work opens several interesting directions for future work.

\paragraph{Improving \algname.}
\algname inherits several weakness from the original \pedel algorithm on which it is based: it requires enumerating a policy class, which in general will be exponentially large in $d$ and $H$, rendering it computationally infeasible for all but the smallest problems, and requires that our MDP has full rank covariates, \Cref{asm:full_rank_cov}. Developing a computationally efficient algorithm that achieves the same complexity as given in \Cref{cor:main_complexity} but free of $\lamminst$ is an exciting future direction. In addition, \Cref{cor:main_complexity} exhibits a lower-order $\cO(\frac{1}{\epsilon^{8/5}})$ term, arising from the complexity required to learn to explore. While a more careful analysis can refine this term somewhat, in general it is not clear how to remove this term completely. We believe removing this completely may require new algorithmic techniques, which we leave for future work.

\paragraph{Verifiability in RL.}
Our results point to an important distinction between verifiable and unverifiable RL, and show that verifiable RL can be significantly more difficult than unverifiable RL. We are not the first to observe this separation: \cite{katz2020true} show that in the multi-armed bandit setting the cost of unverifiable learning is smaller than verifiable learning, and very recently \cite{tirinzoni2022optimistic} show a similar result in online RL. We believe a more thorough exploration of verifiability in RL is an interesting direction for future work with many open questions. \loose 

\paragraph{Towards a General Theory of Offline-to-Online RL.}
While our work answers the question of how to leverage offline data in online RL, it applies only in PAC setting and requires the underlying MDP structure to be linear.
Recently, both the offline and online RL literature has devoted significant attention to RL with more general function approximation. Developing algorithms that address the offline-to-online setting with general function approximation while preserving our optimality guarantees is of much interest. In addition, understanding how to best leverage offline data in different RL settings---for example, regret minimization---is an exciting direction for future work. 

\subsection*{Acknowledgements}
AW is supported in part by NSF TRIPODS II-DMS 2023166. AP would like to thank the support of the Broad Institute of MIT and Harvard and Boston University.
}

%% file: body/linear_mdp_upper_bound.tex

\section{Technical Results}
\begin{prop}[Theorem 1 of \cite{abbasi2011improved}]\label{prop:self_norm}
Let $(\cF_t)_{t=0}^\infty$ be a filtration and $(\eta_t)_{t=1}^\infty$ be a real-valued stochastic process such that $\eta_t$ is $\cF_{t}$ measurable and $\eta_t | \cF_{t-1}$ is mean 0 and $\sigma^2$-subgaussian. Let $\{ \bphi_t \}_{t=1}^{\infty}$ be an $\cR^d$-valued stochastic process such that $\bphi_t$ is $\cF_{t-1}$-measurable. Let $\bLambda_0 \succ 0$ and define
\begin{align*}
\bLambda_t = \bLambda_0 + \sum_{s=1}^t \bphi_s \bphi_s^\top.
\end{align*}
Then with probability at least $1-\delta$, for all $t \ge 0$ simultaneously,
\begin{align*}
\bigg \| \sum_{s=1}^t \bphi_s \eta_s  \bigg \|_{\bLambda_t^{-1}} \le \sigma \sqrt{ \log \frac{\det \bLambda_t}{\det \bLambda_0} + 2 \log \frac{1}{\delta}}.
\end{align*}
\end{prop}

\begin{lemma}\label{lem:log_inequality}
Assume that $B,D \ge 1$. Then if
\begin{align*}
x \ge 3 C \log(3 B \max \{ C,D \}) + D,
\end{align*}
we have that $x \ge C \log (Bx) + D$. 
\end{lemma}
\begin{proof}
Take $x = 3 C \log(3 B \max \{ C,D \}) + D$. Then,
\begin{align*}
C \log(Bx) + D & = C \log(3 B C \log(3 B \max \{ C,D \}) + BD) + D \\
& \le C \log(9 B^2 C \max \{ C,D \} + BD) + D \\
& \le C \log (18 B \max \{ B C^2, B C D, D \}) + D \\
& \le C \log (18 B^2 \max \{ C^2, D^2 \}) + D \\
& \overset{(a)}{\le} C \log (3^3 B^3 \max \{ C^3, D^3 \}) + D \\
& \le 3 C \log(3B \max \{ C,D \}) + D \\
& = x .
\end{align*}
Note that $(a)$ holds even if $C < 1$, since we have assumed $D \ge 1$, so in this case $\max \{ C,D \} = D$. 
\end{proof}

\section{Proof of Upper Bound}\label{sec:upper_proof}

\begin{algorithm}[H]
\begin{algorithmic}[1]
\State  \textbf{input:} tolerance $\epsilon$, confidence $\delta$, policy set $\Pi$ (default: $\Pilsm$)
\For{$i = 1,2,3,\ldots$}
\State $\pi \leftarrow$ \algnamese$(\epsilon,\delta/2i^2,\Pi,2^i)$
\If{$\pi \neq \emptyset$}
\State \textbf{return} $\pi$
\EndIf
\EndFor
\State
\State \textbf{function:} \algnamese 
\State \textbf{input:} tolerance $\epsilon$, confidence $\delta$, policy set $\Pi$, upper bound on number of online samples $\Tonbar$ 
\State  $\Pi_{1} \leftarrow \Pi$, $\bphihat^{1}_{\pi,1} \leftarrow \Exp_{a \sim \pi_1(\cdot | s_1)} [ \bphi(s_1,a)], \forall \pi \in \Pi$, $\Ton \leftarrow 0$, $\lambda \leftarrow 1/d$
\For{$\ell = 1,2,\ldots, \lceil \log \frac{4}{\epsilon} \rceil$}
	\State $\epsilon_\ell \leftarrow 2^{-\ell}$, $\beta_\ell \leftarrow  H^4 \left ( 2 \sqrt{d \log \frac{\lambda + (\Toff + \Tonbar)/d}{\lambda}+ 2\log \frac{2 H^2 | \Pi| \ell^2}{\delta}}  + \sqrt{d \lambda}  \right )^2$
	\For{$h = 1,2,\ldots,H$}
		\State Run procedure outlined in \Cref{thm:fw_full_guarantee} with parameters
		\begin{align*}
		\epsexp \leftarrow \frac{\epsilon_\ell^2}{\beta_\ell}, \quad \delta \leftarrow \frac{\delta}{2H\ell^2}, \quad \lamun \leftarrow \log \frac{4H^2 |\Pi| \ell^2}{\delta}, \quad \Phi \leftarrow \Phi_{\ell,h} := \{ \bphihat_{\pi,h}^\ell \ : \ \pi \in \Pi_\ell \},
		\end{align*}
		\hspace{2.75em} to collect data $\frakD_{h,\ell} \leftarrow \{ (s_{h,\tau}, a_{h,\tau}, r_{h,\tau}, s_{h+1,\tau}) \}_{\tau=1}^{K_{h,\ell}}$ such that: \label{line:pedel_fw_explore}
		\begin{align*}
			\max_{\pi \in \Pi_{\ell}}  &\| \bphihat_{\pi,h}^\ell \|_{(\bLambda_{h,\ell} + \bLamoff^h)^{-1}}^2  \le \epsilon_\ell^2/\beta_\ell \quad  \text{and} \quad \lammin(\bLambda_{h,\ell} + \bLamoff^h) \ge \log \frac{4H^2 |\Pi| \ell^2}{\delta} \\
			&\text{for} \quad \bLambda_{h,\ell} \leftarrow \tsum_{\tau = 1}^{K_{h,\ell}} \bphi_{h,\tau} \bphi_{h,\tau}^\top + 1/d \cdot I \quad \text{and} \quad \bphi_{h,\tau} = \bphi(s_{h,\tau},a_{h,\tau})
		\end{align*}
		\hspace{2.75em} terminate procedure early if $\Ton +$ total number online episodes collected $> \Tonbar$
		\State $\Ton \leftarrow \Ton +$ total number online episodes collected on Line \ref{line:pedel_fw_explore}
		\If{$\Ton \ge \Tonbar$}
		\State \textbf{return} $\emptyset$
		\EndIf
		\For{$\pi \in \Pi_\ell$} \hfill {\color{blue} \texttt{// Estimate feature-visitations for active policies}}
			\State $\bphihat_{\pi,h+1}^\ell \leftarrow \Big ( \sum_{(s_h,a_h,r_h,s_{h+1}) \in \frakDoff^h \cup \frakD_{h,\ell}} \bphi_{\pi,h+1}(s_{h+1}) \bphi(s_h,a_h)^\top (\bLambda_{h,\ell} + \bLamoff^h)^{-1} \Big ) \bphihat_{\pi,h}^\ell$
		\EndFor
		\State $\bthetahat_h^\ell \leftarrow (\bLambda_{h,\ell} + \bLamoff^h)^{-1} \sum_{(s_h,a_h,r_h,s_{h+1}) \in \frakDoff^h \cup \frakD_{h,\ell}} \bphi(s_h,a_h) r_h$
	\EndFor
	\Statex {\color{blue} \texttt{// Remove provably suboptimal policies from active policy set}}
	\State Update $\Pi_\ell$:
	\begin{align*}
	\Pi_{\ell+1} \leftarrow  \Pi_\ell \backslash \Big \{ \pi \in \Pi_\ell \ : \ \Vhat_0^\pi < \sup_{\pi' \in \Pi_\ell} \Vhat_0^{\pi'} - 2\epsilon_\ell \Big \} \quad \text{for} \quad \Vhat_0^\pi := \tsum_{h=1}^H \inner{\bphihat_{\pi,h}^\ell}{\bthetahat_h^\ell}
	\end{align*}
	\If{$|\Pi_{\ell+1}| = 1$}\label{line:early_term} \textbf{return} $\pi \in \Pi_{\ell+1}$\EndIf
\EndFor
\State \textbf{return} any $\pi \in \Pi_{\ell+1}$
\end{algorithmic}
\caption{\textbf{F}ine-\textbf{T}uning \textbf{P}olicy Learning via \textbf{E}xperiment \textbf{De}sign in \textbf{L}inear MDPs (\algname)}
\label{alg:offline_pedel}
\end{algorithm}

The proof of \Cref{cor:main_complexity} follows closely the proof of Theorem 1 in \cite{wagenmaker2022instance}. As such, we omit details where appropriate and refer the interested reader to the cited results in \cite{wagenmaker2022instance}.

\subsection{Estimation of Feature Visitations}

\begin{lemma}\label{lem:linear_mdp_est_transition}
Assume that we have some dataset $\frakD = \{ (s_{h-1,\tau}, a_{h-1,\tau}, r_{h-1,\tau}, s_{h,\tau}) \}_{\tau = 1}^K$ satisfying \Cref{asm:offline_data}.
Denote $\bphi_{h-1,\tau} = \bphi(s_{h-1,\tau},a_{h-1,\tau})$ and $\bLambda_{h-1} = \sum_{\tau=1}^K \bphi_{h-1,\tau} \bphi_{h-1,\tau}^\top + \lambda I$. Fix $\pi$ and let
\begin{align*}
\cThat_{\pi,h} = \left (  \sum_{\tau=1}^K \bphi_{\pi,h}(s_{h,\tau}) \bphi_{h-1,\tau}^\top \right ) \bLambda_{h-1}^{-1}.
\end{align*}
Fix $\bu \in \R^d$ and $\bv \in \R^d$ satisfying $|\bv^\top \bphi_{\pi,h}(s)| \le 1$ for all $s$.
Then with probability at least $1-\delta$, we can bound
\begin{align*}
| \bv^\top (\cT_{\pi,h} - \cThat_{\pi,h}) \bu | \le \left ( 2 \sqrt{ \log \frac{\det \bLambda_{h-1}}{\lambda^d} + 2\log \frac{1}{\delta}} + \sqrt{\lambda} \| \cT_{\pi,h}^\top \bv \|_2 \right ) \cdot \| \bu \|_{\bLambda_{h-1}^{-1}}.
\end{align*}
\end{lemma}
\begin{proof}
Define the $\sigma$-algebra
\begin{align*}
\cF_{h,\tau} = \sigma ( \{ (s_{h,j}, a_{h,j}) \}_{j = 1}^{\tau + 1} \cup \{ (r_{h,j}, s_{h+1,j}) \}_{j=1}^{\tau} ). 
\end{align*}
Then $(s_{h,\tau},a_{h,\tau})$ is $\cF_{h,\tau-1}$-measurable, and $(r_{h,\tau},a_{h,\tau})$ is $\cF_{h,\tau}$-measurable. Since $\frakD$ satisfies \Cref{asm:offline_data}, we have that
\begin{align*}
\Exp[\bphi_{\pi,h}(s_{h}) \mid s_{h-1} = s_{h-1,\tau}, a_{h-1}= a_{h-1,\tau}] & = \Exp[\bphi_{\pi,h}(s_h) \mid \{ s_{h-1,j}, a_{h-1,j} \}_{j=1}^{\tau} \cup \{ r_{h-1,j}, s_{h,j} \}_{j=1}^{\tau - 1}] \\
& = \Exp[\bphi_{\pi,h}(s_h) \mid \cF_{h-1,\tau-1}]
\end{align*}
It follows that, by \Cref{defn:linear_mdp},
\begin{align*}
\cT_{\pi,h} & = \int \bphi_{\pi,h}(s) \rmd \bmu_{h-1}(s)^\top \\
& = \int \bphi_{\pi,h}(s) \rmd \bmu_{h-1}(s)^\top \left (  \sum_{\tau=1}^K \bphi_{h-1,\tau} \bphi_{h-1,\tau}^\top \right ) \bLambda_{h-1}^{-1} + \lambda \int \bphi_{\pi,h}(s) \rmd \bmu_{h-1}(s)^\top \bLambda_{h-1}^{-1} \\
& =  \sum_{\tau=1}^K \left ( \int \bphi_{\pi,h}(s) \rmd \bmu_{h-1}(s)^\top   \bphi_{h-1,\tau} \right ) \bphi_{h-1,\tau}^\top  \bLambda_{h-1}^{-1} + \lambda \int \bphi_{\pi,h}(s) \rmd \bmu_{h-1}(s)^\top \bLambda_{h-1}^{-1} \\
& = \sum_{\tau=1}^K \Exp[\bphi_{\pi,h}(s_{h,\tau}) | s_{h-1} = s_{h-1,\tau}, a_{h-1} = a_{h-1,\tau}] \bphi_{h-1,\tau}^\top  \bLambda_{h-1}^{-1} + \lambda \int \bphi_{\pi,h}(s) \rmd \bmu_{h-1}(s)^\top \bLambda_{h-1}^{-1} \\
& = \sum_{\tau=1}^K \Exp[\bphi_{\pi,h}(s_{h,\tau}) | \cF_{h-1,\tau-1}] \bphi_{h-1,\tau}^\top  \bLambda_{h-1}^{-1} + \lambda \int \bphi_{\pi,h}(s) \rmd \bmu_{h-1}(s)^\top \bLambda_{h-1}^{-1} \\
& = \sum_{\tau=1}^K \Exp[\bphi_{\pi,h}(s_{h,\tau}) | \cF_{h-1,\tau-1}] \bphi_{h-1,\tau}^\top  \bLambda_{h-1}^{-1} + \lambda \cT_{\pi,h} \bLambda_{h-1}^{-1}
\end{align*}
so
\begin{align*}
| \bv^\top (\cT_{\pi,h} -  \cThat_{\pi,h}) \bu | & \le \underbrace{\Big | \sum_{\tau=1}^K \bv^\top \left (  \Exp[\bphi_{\pi,h}(s_{h,\tau}) | \cF_{h-1,\tau-1}]  - \bphi_{\pi,h}(s_{h,\tau}) \right ) \bphi_{h-1,\tau}^\top  \bLambda_{h-1}^{-1} \bu \Big |}_{(a)} + \underbrace{\Big | \lambda \bv^\top \cT_{\pi,h} \bLambda_{h-1}^{-1} \bu \Big |}_{(b)} .
\end{align*}
We can bound
\begin{align*}
(a) \le \| \bu \|_{\bLambda_{h-1}^{-1}} \left \| \sum_{\tau=1}^K \bv^\top \left (  \Exp[\bphi_{\pi,h}(s_{h,\tau}) | \cF_{h-1,\tau-1}]  - \bphi_{\pi,h}(s_{h,\tau}) \right ) \bphi_{h-1,\tau}^\top  \bLambda_{h-1}^{-1/2} \right \|_2.
\end{align*}
By assumption $|\bv^\top  \bphi_{\pi,h}(s_{h,\tau}) | \le 1$, so $\bv^\top \left (  \Exp[\bphi_{\pi,h}(s_{h,\tau}) | \cF_{h-1,\tau-1}]  - \bphi_{\pi,h}(s_{h,\tau}) \right )$ is, conditioned on $\cF_{h-1,\tau-1}$, mean 0 and 2-subgaussian. \Cref{prop:self_norm} then gives that, with probability at least $1-\delta$,
\begin{align*}
\left \| \sum_{\tau=1}^K \bv^\top \left (  \Exp[\bphi_{\pi,h}(s_{h,\tau}) | \cF_{h-1,\tau-1}]  - \bphi_{\pi,h}(s_{h,\tau}) \right ) \bphi_{h-1,\tau}^\top  \bLambda_{h-1}^{-1/2} \right \|_2 \le 2 \sqrt{ \log \frac{\det \bLambda_{h-1}}{\lambda^d} + 2\log \frac{1}{\delta}}.
\end{align*}

We can also bound
\begin{align*}
(b) \le \sqrt{\lambda} \| \bu \|_{\bLambda_{h-1}^{-1}} \| \cT_{\pi,h}^\top \bv \|_2.
\end{align*}
Combining these gives the result.
\end{proof}

\begin{lemma}\label{lem:reward_est_fixed}
Assume that we have some dataset $\frakD = \{ (s_{h,\tau}, a_{h,\tau}, r_{h,\tau}, s_{h+1,\tau}) \}_{\tau = 1}^K$ satisfying \Cref{asm:offline_data}. Denote $\bphi_{h,\tau} = \bphi(s_{h,\tau},a_{h,\tau})$ and $\bLambda_{h} = \sum_{\tau=1}^K \bphi_{h,\tau} \bphi_{h,\tau}^\top + \lambda I$.
Let
\begin{align*}
\bthetahat_h = \argmin_{\btheta} \sum_{\tau = 1}^K (r_{h,\tau} - \inner{\bphi_{h,\tau}}{\btheta})^2 + \lambda \| \btheta \|_2^2
\end{align*}
and fix $\bu \in \R^d$.
Then with probability at least $1-\delta$:
\begin{align*}
| \inner{\bu}{\bthetahat_h - \btheta_h}| \le \left  ( 2 \sqrt{\log \frac{\det \bLambda_h}{\lambda^d} + 2 \log \frac{1}{\delta}} + \sqrt{d\lambda} \right ) \cdot \| \bu \|_{\bLambda_h^{-1}}.
\end{align*}
\end{lemma}
\begin{proof}
We construct a filtration as in \Cref{lem:linear_mdp_est_transition}. By construction we have
\begin{align*}
\bthetahat_h = \bLambda_h^{-1} \sum_{\tau = 1}^K \bphi_{h,\tau} r_{h,\tau}.
\end{align*}
Furthermore, by \Cref{defn:linear_mdp}:
\begin{align*}
\btheta_h & = \bLambda_h^{-1} \bLambda_h \btheta_h  = \bLambda_h^{-1} \sum_{\tau = 1}^K \bphi_{h,\tau} \Exp[r_{h,\tau}| \cF_{h,\tau-1}] + \lambda \bLambda_h^{-1} \btheta_h.
\end{align*}
Thus,
\begin{align*}
| \inner{\bu}{\bthetahat_h - \btheta_h}| & \le \underbrace{\left |  \sum_{\tau = 1}^K \bu^\top  \bLambda_h^{-1} \bphi_{h,\tau} (r_{h,\tau} - \Exp[r_{h,\tau}| \cF_{h,\tau-1}]) \right |}_{(a)} + \underbrace{\left | \lambda \bu^\top \bLambda_h^{-1} \btheta_h \right  |}_{(b)} .
\end{align*}
We can bound
\begin{align*}
(a) \le \| \bu \|_{\bLambda_h^{-1}} \left \| \sum_{\tau = 1}^K  \bLambda_h^{-1/2} \bphi_{h,\tau} (r_{h,\tau} - \Exp[r_{h,\tau}| \cF_{h,\tau-1}]) \right \|_2.
\end{align*}
Since rewards are bounded in [0,1] almost surely, we have that $r_{h,\tau} - \Exp[r_{h,\tau} | \cF_{h,\tau-1}]$ is conditionally mean 0 and 2-subgaussian. Therefore, applying Theorem 1 of \cite{abbasi2011improved}, with probability at least $1-\delta$,
\begin{align*}
\left \| \sum_{\tau = 1}^K  \bLambda_h^{-1/2} \bphi_{h,\tau} (r_{h,\tau} - \Exp[r_{h,\tau}| \cF_{h-1,\tau}]) \right \|_2  \le 2 \sqrt{\log \frac{\det \bLambda_h}{\lambda^d} + 2 \log \frac{1}{\delta}}
\end{align*}

By \Cref{defn:linear_mdp}, we can also bound
\begin{align*}
(b) \le \sqrt{\lambda} \| \bu \|_{\bLambda_h^{-1}} \| \btheta_h \|_2 \le \sqrt{d \lambda} \| \bu \|_{\bLambda_h^{-1}}.
\end{align*}
Combining these proves the result. 
\end{proof}

\begin{lemma}\label{lem:est_good_event}
Let $\cEest^{\ell,h}$ denote the event on which, for all $\pi \in \Pi_\ell$:
\begin{align*}
& |\inner{\btheta_{h+1}}{\bphihat_{\pi,h+1}^\ell - \bphi_{\pi,h+1}}| \le \sum_{i=1}^{h} \left ( 2 \sqrt{\log \frac{\det \bLambda_{i,\ell}}{\lambda^d} + 2\log \frac{2 H^2 | \Pi| \ell^2}{\delta}}  + \sqrt{d \lambda}  \right ) \cdot \| \bphihat_{\pi,i}^\ell \|_{\bLambda_{i,\ell}^{-1}}, \\
& \| \bphihat_{\pi,h+1}^\ell - \bphi_{\pi,h+1} \|_2 \le d \sum_{i=1}^{h} \left ( 2 \sqrt{\log \frac{\det \bLambda_{i,\ell}}{\lambda^d} + 2\log \frac{2 H^2 d | \Pi| \ell^2}{\delta}}  + \sqrt{d \lambda}  \right ) \cdot \| \bphihat_{\pi,i}^\ell \|_{\bLambda_{i,\ell}^{-1}} , \\
& | \inner{\bphihat_{\pi,h}^\ell}{\bthetahat_h - \btheta_h}| \le \left  ( 2 \sqrt{\log \frac{\det \bLambda_{h,\ell}}{\lambda^d} + 2 \log \frac{2 H^2 |\Pi| \ell^2}{\delta}} + \sqrt{d\lambda}  \right ) \cdot \| \bphihat_{\pi,h}^\ell \|_{\bLambda_{h,\ell}^{-1}}.
\end{align*}
Then $\Pr[(\cEest^{\ell,h})^c] \le  \frac{\delta}{2H\ell^2}$.
\end{lemma}
\begin{proof}
This follows analogously to Lemma B.5 of \cite{wagenmaker2022instance} and using \Cref{lem:linear_mdp_est_transition} and \Cref{lem:reward_est_fixed} in place of Lemma B.1 and B.4 of \cite{wagenmaker2022instance}.

We apply \Cref{lem:linear_mdp_est_transition} and \Cref{lem:reward_est_fixed} at step $h$ and round $\ell$ with the dataset
\begin{align*}
\frakD \leftarrow \frakDoff^{h-1} \cup \frakD_{h-1,\ell}.
\end{align*}
By assumption, the offline data in $\frakD$ satisfies \Cref{asm:offline_data}. In addition, it is easy to see that given our collection procedure, the online data also satisfies \Cref{asm:offline_data}. Thus, $\frakD$ satisfies \Cref{asm:offline_data}.
\end{proof}

\subsection{Policy Elimination via \algnamese}

\begin{lemma}\label{lem:exp_good_event}
Let $\cEexp^{\ell,h}$ denote the event on which:
\begin{itemize}
\item The exploration procedure on \Cref{line:pedel_fw_explore} terminates after running for at most
\begin{align*}
\max \bigg \{   \min_{\Tfw}  & \ C\cdot \Tfw \quad \text{s.t.} \quad \inf_{\bLambda \in \bOmega} \max_{\bphi \in \Phi_{\ell,h}} \bphi^\top ( \Tfw ( \bLambda + \lambar I) + \bLamoff )^{-1} \bphi  \le \frac{\epsilon_\ell^2}{6\beta_\ell}, \\
& \frac{\poly(d,H,\frac{1}{\lamminst},\log 1/\delta, \log |\Pi| )}{\epsilon_\ell^{8/5}} \bigg \}
\end{align*}
episodes
\item The covariates returned by \Cref{line:pedel_fw_explore} for any $(h,\ell)$, $\bLambda_{h,\ell}$, satisfy
\begin{align*}
& \max_{\bphi \in \Phi_{\ell,h}} \| \bphi \|_{(\bLambda_{h,\ell} + \bLamoff)^{-1}}^2 \le \frac{\epsilon_\ell^2}{\beta_\ell}, \qquad \lammin(\bLambda_{h,\ell}) \ge \log \frac{4H^2|\Pi_\ell| \ell^2}{\delta}.
\end{align*}
\end{itemize}
Then $\Pr[(\cEexp^{\ell,h})^c \cap \cEest^{\ell,h-1} \cap (\cap_{i=1}^{h-1} \cEexp^{\ell,i} )] \le  \frac{\delta}{2H\ell^2}$. 
\end{lemma}
\begin{proof}
This follows from \Cref{thm:fw_full_guarantee} and a union bound.
\end{proof}

\begin{lemma}\label{lem:good_event}
Define $\cEexp = \cap_{\ell} \cap_h \cEexp^{\ell,h}$ and $\cEest = \cap_{\ell} \cap_h \cEest^{\ell,h}$. Then $\Pr[\cEest \cap \cEexp] \ge 1-2\delta$ and on $\cEest \cap \cEexp$, for all $h,\ell$, and $\pi \in \Pi_\ell$, 
\begin{align*}
& |\inner{\btheta_{h+1}}{\bphihat_{\pi,h+1}^\ell - \bphi_{\pi,h+1}}| \le \epsilon_\ell/2H, \\
& \| \bphihat_{\pi,h+1}^\ell - \bphi_{\pi,h+1} \|_2 \le d \epsilon_\ell/2H, \\
& | \inner{\bphihat_{\pi,h}^\ell}{\bthetahat_h - \btheta_h}| \le \epsilon_\ell/2H.
\end{align*}
\end{lemma}
\begin{proof}
This follows analogously to the proof of Lemma B.8 of \cite{wagenmaker2022instance}.
\end{proof}

\begin{lemma}\label{lem:correctness}
On the event $\cEest \cap \cEexp$, for all $\ell > \ell_0$, every policy $\pi \in \Pi_\ell$ satisfies $\Vst_0(\Pi) - \Vpi_0 \le 4\epsilon_\ell $ and $\pitilst \in \Pi_\ell$, for $\pitilst = \argmax_{\pi \in \Pi} \Vpi_0$. 
\end{lemma}
\begin{proof}
This follows analogously to the proof of Lemma B.9 of \cite{wagenmaker2022instance}.
\end{proof}

\begin{lemma}\label{lem:pedel_se_bound}
With probability at least $1-2\delta$, \algnamese run with parameters $(\epsilon,\delta,\Tonbar)$ will terminate after collecting at most 
\begin{align*}
\min \left \{ C \cdot \sum_{h=1}^{H} \sum_{\ell = 1}^{\iotaalg} \Tfw_{h,\ell}(\Tonbar) +   \poly \left ( d, H, \log \frac{1}{\delta}, \frac{1}{\lamminst}, \log | \Pi|, \log \frac{1}{\epsilon} \right ) \cdot \frac{\beta_{\iotaalg}(\Tonbar)}{\epsilon^{8/5} \vee \Delmin(\Pi)^{8/5}}, \Tonbar \right \}
\end{align*}
episodes for $\iotaalg := \min \{ \lceil \log \frac{4}{\epsilon} \rceil, \log \frac{4}{\Delmin(\Pi)} \}$ and
\begin{align*}
\Tfw_{h,\ell}(\Tonbar) = \min_{N}  \  N \quad \text{s.t.} \quad   \inf_{\bLambda \in \bOmega}  \max_{\pi \in \Pi(4\epsilon_\ell)} \|  \bphi_{\pi,h} \|_{( N\bLambda + \bLamoff^h )^{-1}}^2  \le \frac{\epsilon_\ell^2}{48\beta_\ell(\Tonbar)}
\end{align*}
where here $\Pi(4\epsilon_\ell) =  \{ \pi \in \Pi \ : \ V_0^\pi \ge \max_{\pi \in \Pi} V_0^{\pi} - 4 \epsilon_\ell \}$. 
\end{lemma}
\begin{proof}
Note that \algnamese will terminate and output $\emptyset$ if the total number of online episodes it has collected reaches $\Tonbar$, so it follows that $\Tonbar$ is always an upper bound on the number of online episodes that will be collected. Henceforth we assume that we are in a situation where $\Tonbar$ is greater than the number of online episodes collected.

By \Cref{lem:good_event} the event $\cEest \cap \cEexp$ occurs with probability at least $1-2\delta$. Henceforth we assume we are on this event. 

By \Cref{lem:exp_good_event}, the complexity of exploration at round $\ell$, step $h$, is bounded by 
\begin{align*}
\max \bigg \{   \min_{\Tfw}  & \ C\cdot \Tfw \quad \text{s.t.} \quad \inf_{\bLambda \in \bOmega} \max_{\bphi \in \Phi_{\ell,h}} \bphi^\top ( \Tfw( \bLambda + \lambar I) + \bLamoff )^{-1} \bphi  \le \frac{\epsilon_\ell^2}{6\beta_\ell(\Tonbar)}, \\
& \frac{\poly(d,H,\frac{1}{\lamminst},\log 1/\delta, \log |\Pi| ) \cdot \beta_\ell(\Tonbar)^{3/4}}{\epsilon_\ell^{8/5}} \bigg \}.
\end{align*}
On $\cEest \cap \cEexp$, by \Cref{lem:good_event}, for each $\pi \in \Pi_\ell$, we have $\| \bphihat_{\pi,h}^\ell - \bphi_{\pi,h} \|_2 \le d \epsilon_\ell / 2H$. As $\Phi_{\ell,h} = \{ \bphihat_{\pi,h}^\ell : \pi \in \Pi_\ell \}$, it follows that we can upper bound
\begin{align*}
&  \inf_{\bLambda \in \bOmega}  \max_{\bphi \in \Phi_{\ell,h}} \bphi^\top ( \Tfw( \bLambda + \lambar I) + \bLamoff )^{-1} \bphi \\
& \le  \inf_{\bLambda \in \bOmega}  \max_{\pi \in \Pi_\ell} 2 \bphi_{\pi,h}^\top ( \Tfw( \bLambda + \lambar I) + \bLamoff )^{-1} \bphi_{\pi,h} + 2 (\bphihat_{\pi,h}^\ell - \bphi_{\pi,h})^\top ( \Tfw( \bLambda + \lambar I) + \bLamoff )^{-1} (\bphihat_{\pi,h}^\ell - \bphi_{\pi,h}) \\
& \le \inf_{\bLambda \in \bOmega}  \max_{\pi \in \Pi_\ell} 2 \bphi_{\pi,h}^\top ( \Tfw( \bLambda + \lambar I) + \bLamoff )^{-1} \bphi_{\pi,h} + \frac{d^2 \epsilon_\ell^2}{2 H^2 \lammin(\bLambda)} \\
& \le \inf_{\bLambda \in \bOmega}  \max_{\pi \in \Pi_\ell} 4 \bphi_{\pi,h}^\top ( \Tfw( \bLambda + \lambar I) + \bLamoff )^{-1} \bphi_{\pi,h} + \inf_{\bLambda \in \bOmega} \frac{d^2 \epsilon_\ell^2}{\Tfw H^2 \lammin(\bLambda)} \\
& \le \inf_{\bLambda \in \bOmega}  \max_{\pi \in \Pi_\ell} 4 \bphi_{\pi,h}^\top ( \Tfw( \bLambda + \lambar I) + \bLamoff )^{-1} \bphi_{\pi,h} +  \frac{d^2 \epsilon_\ell^2}{\Tfw H^2 \lamminst}\\
& \le \max \left \{ \inf_{\bLambda \in \bOmega}  \max_{\pi \in \Pi_\ell} 8 \bphi_{\pi,h}^\top ( \Tfw( \bLambda + \lambar I) + \bLamoff )^{-1} \bphi_{\pi,h},  \frac{2 d^2 \epsilon_\ell^2}{\Tfw H^2 \lamminst} \right \}.
\end{align*}
It follows that
\begin{align*}
& \min_{\Tfw}  \ C\cdot \Tfw \quad \text{s.t.} \quad \inf_{\bLambda \in \bOmega} \max_{\bphi \in \Phi_{\ell,h}} \bphi^\top ( \Tfw( \bLambda + \lambar I) + \bLamoff )^{-1} \bphi  \le \frac{\epsilon_\ell^2}{6\beta_\ell(\Tonbar)} \\
& \le \min_{\Tfw}  \ C\cdot \Tfw \quad \text{s.t.} \quad  \max \left \{ \inf_{\bLambda \in \bOmega}  \max_{\pi \in \Pi_\ell} 8 \bphi_{\pi,h}^\top ( \Tfw( \bLambda + \lambar I) + \bLamoff )^{-1} \bphi_{\pi,h},  \frac{2 d^2 \epsilon_\ell^2}{\Tfw H^2 \lamminst} \right \}  \le \frac{\epsilon_\ell^2}{6\beta_\ell(\Tonbar)} \\
&  \le \left [ \min_{\Tfw}  \ C\cdot \Tfw \quad \text{s.t.} \quad   \inf_{\bLambda \in \bOmega}  \max_{\pi \in \Pi_\ell}  \bphi_{\pi,h}^\top ( \Tfw\bLambda + \bLamoff )^{-1} \bphi_{\pi,h}  \le \frac{\epsilon_\ell^2}{48\beta_\ell(\Tonbar)} \right ] + \frac{12 \beta_\ell(\Tonbar) d^2}{H^2 \lamminst} \\
&  \le \underbrace{\left [ \min_{\Tfw}  \ C\cdot \Tfw \quad \text{s.t.} \quad   \inf_{\bLambda \in \bOmega}  \max_{\pi \in \Pi(4\epsilon_\ell)}  \bphi_{\pi,h}^\top ( N\bLambda + \bLamoff )^{-1} \bphi_{\pi,h}  \le \frac{\epsilon_\ell^2}{48\beta_\ell(\Tonbar)} \right ]}_{=: \Tfw_{h,\ell}(\Tonbar)} + \frac{12 \beta_\ell(\Tonbar) d^2}{H^2 \lamminst}
\end{align*}
where the last inequality follows since, by \Cref{lem:correctness}, for $\ell > \ell_0$, every policy $\pi \in \Pi_\ell$ will be $4\epsilon_\ell$ optimal, so we therefore have
\begin{align*}
\Pi_\ell \subseteq \{ \pi \in \Pi \ : \ V_0^\pi \ge V_0^{\pitilst} - 4 \epsilon_\ell \} =: \Pi(4\epsilon_\ell).
\end{align*}

By \Cref{lem:correctness}, if $4 \epsilon_\ell < \Delmin(\Pi)$, we must have that $\Pi_\ell = \{ \pitilst \}$, and will therefore terminate on \Cref{line:early_term} since $|\Pi_\ell| = 1$. Thus, we can bound the number of number of epochs by 
\begin{align*}
\iotaalg := \min \{ \lceil \log \frac{4}{\epsilon} \rceil, \log \frac{4}{\Delmin(\Pi)} \}.
\end{align*} 
It follows that the total complexity is bounded as
\begin{align}\label{eq:Nse}
\Nse(\Tonbar) := \sum_{h=1}^{H} \sum_{\ell = 1}^{\iotaalg} \Tfw_{h,\ell}(\Tonbar) +   \poly \left ( d, H, \log \frac{1}{\delta}, \frac{1}{\lamminst}, \log | \Pi|, \log \frac{1}{\epsilon} \right ) \cdot \frac{\beta_{\iotaalg}(\Tonbar)}{\epsilon^{8/5} \vee \Delmin(\Pi)^{8/5}}.
\end{align}
\end{proof}

\subsection{Completing the Proof: \algname}
\begin{lemma}\label{lem:complexity}
\algname will terminate after running for at most
\begin{align*}
C \cdot \sum_{h=1}^{H} \sum_{\ell = 1}^{\iotaalg} \Tfw_{h,\ell}(\Tontil) +   \poly \left ( d, H, \log \frac{1}{\delta}, \frac{1}{\lamminst}, \log | \Pi|, \log \frac{1}{\epsilon}, \log \Toff \right ) \cdot \frac{1}{\epsilon^{8/5} \vee \Delmin(\Pi)^{8/5}}
\end{align*}
episodes, for
\begin{align*}
\Tontil := \poly \left ( d, H, \log \frac{1}{\delta}, \frac{1}{\lamminst}, \log |\Pi|, \frac{1}{\epsilon}, \log(\Toff) \right ) .
\end{align*}
Furthermore, it will output a policy $\pihat$ such that
\begin{align*}
V_0^{\pihat} \ge \max_{\pi \in \Pi} V_0^\pi - \epsilon.
\end{align*}
\end{lemma}
\begin{proof}
By Lemma B.10 of \cite{wagenmaker2022instance}, we can  bound
\begin{align*}
\Tfw_{h,\ell}(\Tonbar) \le \frac{48C d \beta_\ell(\Tonbar)}{\epsilon_\ell^2} .
\end{align*}
From \Cref{lem:pedel_se_bound}, it follows that the total complexity of running \algnamese with parameters $(\epsilon,\delta/2i^2,\Tonbar^i)$, for $\Tonbar^i = 2^i$, is bounded as
\begin{align*}
& \sum_{h=1}^{H} \sum_{\ell = 1}^{\iotaalg} \left ( \frac{48C d \beta_\ell(\Tonbar^i)}{\epsilon_\ell^2}  +  \poly \left ( d, H, \log \frac{i}{\delta}, \frac{1}{\lamminst}, \log | \Pi|, \log \frac{1}{\epsilon} \right ) \cdot \frac{\beta_{\iotaalg}(\Tonbar^i)}{\epsilon^{8/5} \vee \Delmin(\Pi)^{8/5}} \right ) \\
& \le \poly \left ( d, H, \log \frac{i}{\delta}, \frac{1}{\lamminst}, \log |\Pi|, \frac{1}{\epsilon} \right ) \cdot  \sqrt{\log(\Toff + \Tonbar^i)} \\
& \le \poly \left ( d, H, \log \frac{\Tonbar^i}{\delta}, \frac{1}{\lamminst}, \log |\Pi|, \frac{1}{\epsilon} \right ) \cdot  \left ( \log(\Toff) + \log(\Tonbar^i) \right ).
\end{align*}
To ensure that $\Tonbar^i$ is sufficiently large, we then need only that
\begin{align*}
\poly \left ( d, H, \log \frac{\Tonbar^i}{\delta}, \frac{1}{\lamminst}, \log |\Pi|, \frac{1}{\epsilon} \right ) \cdot \left ( \log(\Toff) + \log(\Tonbar^i) \right ) \le \Tonbar^i.
\end{align*}
To achieve this it suffices that
\begin{align*}
\Tonbar^i \ge \poly \left ( d, H, \log \frac{1}{\delta}, \frac{1}{\lamminst}, \log |\Pi|, \frac{1}{\epsilon}, \log(\Toff) \right ) =: \Tontil.
\end{align*}

Note that \algnamese will terminate and output $\emptyset$ if the total number of online episodes it has collected reaches $\Tonbar$, so it follows that $\Tonbar$ is an upper bound on the number of online episodes that will be collected. Thus, the total number of episodes is bounded as
\begin{align*}
\min_i \sum_{j=1}^i 2^j \quad \text{s.t.} \quad \Nse(2^i) \le 2^i
\end{align*}
for $\Nse$ as defined in \eqref{eq:Nse}.
Note that a feasible solution to this is $i = \log_2(\Tontil)$, so
\begin{align*}
\min_i \sum_{j=1}^i 2^j \quad \text{s.t.} \quad \Nse(2^i) \le 2^i & = \min_{i \le \log_2(\Tontil)}  \sum_{j=1}^i 2^j \quad \text{s.t.} \quad \Nse(2^i) \le 2^i  \\
& \le \min_{i \le \log_2(\Tontil)}  \sum_{j=1}^i 2^j \quad \text{s.t.} \quad \Nse(\Tontil) \le 2^i \\
& \le 2 \Nse(\Tontil)
\end{align*}
where the first inequality uses that $\Nse(2^i)$ is increasing in $i$. The result follows.

Correctness follows by \Cref{lem:correctness}, since upon termination, $\Pi_\ell$ will only contain policies $\pi$ satisfying $V_0^{\pi} \ge \max_{\pi \in \Pi} V_0^\pi - \epsilon$ (and will contain at least 1 policy since $\pitilst \in \Pi_\ell$ for all $\ell$). Furthermore, by \Cref{lem:correctness}, if $4 \epsilon_\ell < \Delmin(\Pi)$, we must have that $\Pi_\ell = \{ \pitilst \}$, and will therefore terminate on \Cref{line:early_term} since $|\Pi_\ell| = 1$.
\end{proof}

\begin{proof}[Proof of \Cref{cor:main_complexity}]
By the definition of $\Pi(4\epsilon_\ell)$, for each $\pi \in \Pi(4\epsilon_\ell)$ we have 
\begin{align*}
\epsilon_\ell^2 = \frac{1}{16} \left ( (\Vst_0(\Pi) - V_0^\pi)^2 \vee (4\epsilon_\ell)^2 \right ).
\end{align*}
We can therefore bound
\begin{align}
\Tfw_{h,\ell}(\Tontil) & \le \min_\Tfw \ \Tfw \quad \text{s.t.} \quad   \inf_{\bLambda \in \bOmega}  \max_{\pi \in \Pi(4\epsilon_\ell)} \frac{\|  \bphi_{\pi,h} \|_{( \Tfw\bLambda + \bLamoff^h )^{-1}}^2}{(\Vst_0(\Pi) - V_0^\pi)^2 \vee \epsilon_\ell^2}  \le \frac{c}{\beta_\ell(\Tontil)} \nonumber \\
& \le \min_\Tfw \ \Tfw \quad \text{s.t.} \quad   \inf_{\bLambda \in \bOmega}  \max_{\pi \in \Pi} \frac{\|  \bphi_{\pi,h} \|_{( \Tfw\bLambda + \bLamoff^h )^{-1}}^2}{(\Vst_0(\Pi) - V_0^\pi)^2 \vee \epsilon^2 \vee \Delmin(\Pi)^2}  \le \frac{c}{\beta_\ell(\Tontil)} . \label{eq:complexity_pf_eq1}
\end{align}
We next set $\Pi$ to be the set of linear softmax policies defined in Lemma A.14 of \cite{wagenmaker2022instance}, and note that this set is guaranteed to contain a policy which is $\epsilon$-optimal as compared to the best possible policy. Furthermore, we can bound $\log |\Pi| \le \cO(dH \cdot \log \frac{1}{\epsilon})$. Using this and that $\ell \le \iotaalg$, we can then bound
\begin{align*}
\eqref{eq:complexity_pf_eq1} \le \min_\Tfw \ \Tfw \quad \text{s.t.} \quad   \inf_{\bLambda \in \bOmega}  \max_{\pi \in \Pi} \frac{\|  \bphi_{\pi,h} \|_{( \Tfw\bLambda + \bLamoff^h )^{-1}}^2}{(\Vst_0 - V_0^\pi)^2 \vee \epsilon^2 }  \le \frac{c}{\beta}
\end{align*}
for
\begin{align*}
\beta := d H^5 \cdot \logs(d, H, \Toff, \frac{1}{\lamminst}, \frac{1}{\epsilon}, \log \frac{1}{\delta} ) + H^4 \cdot \log \frac{1}{\delta}
\end{align*}
We can therefore bound the complexity given in \Cref{lem:complexity} as
\begin{align*}
C  \iotaalg \cdot \sum_{h=1}^H \Noto^h(\frakDoff,\epsilon;\beta) + \poly \left ( d, H, \log \frac{1}{\delta}, \frac{1}{\lamminst}, \log | \Pi|, \log \frac{1}{\epsilon}, \log \Toff \right ) \cdot \frac{1}{\epsilon^{8/5} \vee \Delmin(\Pi)^{8/5}}.
\end{align*}
\end{proof}

\section{Online Experiment Design}

The results in this section largely follow those presented in Appendices C and D of \cite{wagenmaker2022instance}, with several minor modifications. As such, we omit calculations that would be identical to those in \cite{wagenmaker2022instance} and refer the reader to \cite{wagenmaker2022instance} for more in-depth proofs.

\begin{algorithm}[h]
\begin{algorithmic}[1]
\State \textbf{input}: function to optimize $f$, number of iterates $T$, episodes per iterate $K$
\State Play any policy for $K$ episodes, denote collected covariates as $\bGamma_0$, collected data as $\frakD_0$
\State $\bLambda_1 \leftarrow K^{-1} \bGamma_0$
\For{$t = 1,2,\ldots,T$}
	\State Set $\gamma_t \leftarrow \frac{1}{t+1}$
	\State Run \force \citep{wagenmaker2022first} on reward $r_h^t(s,a) = \tr(\Xi_{\bLambda_t} \cdot \bphi(s,a) \bphi(s,a)^\top)/M$ for $K$ episodes, denote collected covariates as $\bGamma_t$, collected data as $\frakD_t$
	\State $\bLambda_{t+1} \leftarrow (1- \gamma_t) \bLambda_t + \gamma_t  K^{-1} \bGamma_t$
\EndFor
\State \textbf{return} $\bLambda_{T+1}$, $\cup_{t=0}^T \frakD_t$
\end{algorithmic}
\caption{Online Frank-Wolfe via Regret Minimization (\fwregret)}
\label{alg:regret_fw}
\end{algorithm}

\newcommand{\bOmegahat}{\widehat{\bOmega}}
\begin{lemma}\label{lem:regret_fw}
Consider running \Cref{alg:regret_fw} with a function $f$ satisfying Definition 5.1 of \cite{wagenmaker2022instance}. Then, we have that, with probability at least $1-\delta$,
\begin{align*}
f(\bLambda_{T+1}) - \inf_{\bLambda \in \bOmega} f(\bLambda) & \le \frac{ \beta R^2 (\log T + 3)}{2(T+1)} +  \sqrt{\frac{8 M^2 \log (4T/\delta) }{K}}  + \sqrt{\frac{c_1 M^2 d^4 H^4 \log^{3}(2HKT/\delta)}{K}} \\ 
& \qquad + \frac{c_2 M  d^4 H^3 \log^{7/2}(2HKT/\delta)}{K}
\end{align*}
for $R = \sup_{\bLambda,\bLambda' \in \bOmegahat} \| \bLambda - \bLambda' \|$ and $\bOmegahat = \{ \Exp_{(s,a) \sim \omega}[\bphi(s,a) \bphi(s,a)^\top] \ : \ \omega \in \simplex_{\cS \times \cA} \}$. 
\end{lemma}
\begin{proof}
This follows analogously to the proof of Lemma C.3 of \cite{wagenmaker2022instance}, but without requiring that $K$ be large enough to upper bound the $\cO(1/\sqrt{K})$ and $\cO(1/K)$ terms by $\frac{ \beta R^2 (\log T + 3)}{2(T+1)} $, and instantiating \textsc{RegMin} with the \textsc{Force} algorithm of \cite{wagenmaker2022first}.
\end{proof}

\begin{algorithm}[h]
\begin{algorithmic}[1]
\State \textbf{input}: functions to optimize $( f_i )_i$, constraint tolerance $\epsilon$, confidence $\delta$
\For{$i = 1,2,3,\ldots$}
	\State $T_i \leftarrow 2^i$, $K_i \leftarrow 2^{i}$
	\State $\bLamhat_i, \frakD_i \leftarrow$ \textsc{FWRegret}($f_i,T_i-1,K_i$)
	\If{$f_i(\bLamhat_i) \le K_i T_i \epsilon$}\label{line:data_fw_if_eps}
		\State \textbf{return} $\bLamhat$, $K_i T_i$, $\frakD_i$
	\EndIf
\EndFor
\end{algorithmic}
\caption{Collect Optimal Covariates (\optcov)}
\label{alg:regret_data_fw}
\end{algorithm}

\begin{lemma}\label{thm:regret_data_fw}
Let 
\begin{align*}
f_i(\bLambda) = \frac{1}{\eta_i} \log \left ( \sum_{\bphi \in \Phi} e^{\eta_i \| \bphi \|_{\bA_i(\bLambda)^{-1}}^2} \right ), \quad \bA_i(\bLambda) = \bLambda + \frac{1}{T_i K_i} \bLambda_{0,i} + \frac{1}{T_i K_i} \bLamoff
\end{align*}
for some $\bLambda_{0,i}$ satisfying $\bLambda_{0,i} \succeq \bLambda_0$ for all $i$, and $\eta_i = 2^{2i/5}$. 
Let $(\beta_i,M_i)$ denote the smoothness and magnitude constants for $f_i$. Let $(\beta,M)$ be some values such that $\beta_i \le  \eta_i \beta, M_i \le M$ for all $i$.

Then, if we run \Cref{alg:regret_data_fw} on $(f_i)_i$ with constraint tolerance $\epsilon$ and confidence $\delta$, we have that with probability at least $1-\delta$, it will run for at most
\begin{align*}
\max \bigg \{   \min_{N}  & \ 16N \quad \text{s.t.} \quad \inf_{\bLambda \in \bOmega} \max_{\bphi \in \Phi} \bphi^\top ( N \bLambda + \bLambda_{0} + \bLamoff )^{-1} \bphi  \le \frac{\epsilon}{6}, \\
& \frac{\poly(\beta,R,d,H,M,\log 1/\delta, \log |\Phi| )}{\epsilon^{4/5}} \bigg \}.
\end{align*}
episodes, and will return data $\{ \bphi_\tau \}_{\tau =1}^N$ with covariance $\bSighat_N = \sum_{\tau=1}^N \bphi_\tau \bphi_\tau^\top$ such that 
\begin{align*}
f_{\ihat}(N^{-1} \bSighat_N) \le N \epsilon,
\end{align*}
where $\ihat$ is the iteration on which \optcov terminates.
\end{lemma}
\begin{proof}
Our goal is to simply find a setting of $i$ that is sufficiently large to guarantee the condition $f_i(\bLamhat_i) \le K_i T_i \epsilon$ is met. By \Cref{lem:regret_fw}, we have with probability at least $1-\delta/(2i^2)$:
\begin{align*}
f_i(\bLamhat_i) & \le \inf_{\bLambda \in \bOmega} f_i(\bLambda)  + \frac{\beta_i R^2 (\log T_i + 3)}{2 T_i}  + \sqrt{\frac{4 M^2 \log (8 i^2T_i/\delta) }{K_i}}  \\
& \qquad + \sqrt{\frac{c_1 M^2 d^4 H^4 \log^{3}(8 i^2HK_i T_i/\delta)}{K_i}} + \frac{c_2 M  d^4 H^3 \log^{7/2}(4 i^2HK_iT_i/\delta)}{K_i} \\
& \le 3 \max \Bigg \{  \inf_{\bLambda \in \bOmega} f_i(\bLambda) , \frac{\beta_i R^2 (\log T_i + 3)}{2 T_i} ,  \\
& \qquad \sqrt{\frac{4 M^2 \log (8 i^2T_i/\delta) }{K_i}} + \sqrt{\frac{c_1 M^2 d^4 H^4 \log^{3}(8 i^2HK_i T_i/\delta)}{K_i}} + \frac{c_2 M  d^4 H^3 \log^{7/2}(4 i^2HK_iT_i/\delta)}{K_i} \Bigg \} .
\end{align*}  
So a sufficient condition for $f_i(\bLamhat_i) \le K_i T_i \epsilon$ is that
\begin{align}\label{eq:linear_mdp_suff_KT}
\begin{split}
K_i T_i \ge &  \frac{3}{\epsilon} \max \Bigg \{  \inf_{\bLambda \in \bOmega} f_i(\bLambda) , \frac{\beta_i R^2 (\log T_i + 3)}{2 T_i} ,  \\
& \qquad \sqrt{\frac{4 M^2 \log (8 i^2T_i/\delta) }{K_i}} + \sqrt{\frac{c_1 M^2 d^4 H^4 \log^{3}(8 i^2HK_i T_i/\delta)}{K_i}} + \frac{c_2 M  d^4 H^3 \log^{7/2}(4 i^2HK_iT_i/\delta)}{K_i} \Bigg \} .
\end{split}
\end{align}

Recall that 
\begin{align*}
f_i(\bLambda) = \frac{1}{\eta_i} \log \left ( \sum_{\bphi \in \Phi} e^{\eta_i \| \bphi \|_{\bA_i(\bLambda)^{-1}}^2} \right ), \quad \bA_i(\bLambda) = \bLambda + \frac{1}{T_i K_i} \bLambda_{0,i} + \frac{1}{T_i K_i} \bLamoff.
\end{align*}
By Lemma D.1 of \cite{wagenmaker2022instance}, we can bound
\begin{align*}
\max_{\bphi \in \Phi} \| \bphi \|_{\bA_i(\bLambda)^{-1}}^2 \le f_i(\bLambda) \le \max_{\bphi \in \Phi} \| \bphi \|_{\bA_i(\bLambda)^{-1}}^2 + \frac{\log |\Phi|}{\eta_i}.
\end{align*}
Thus,
\begin{align*}
\inf_{\bLambda \in \bOmega} f_i(\bLambda) & \le \inf_{\bLambda \in \bOmega} \max_{\bphi \in \Phi} \| \bphi \|_{\bA_i(\bLambda)^{-1}}^2 + \frac{\log |\Phi|}{\eta_i} \\
& = \inf_{\bLambda \in \bOmega} \max_{\bphi \in \Phi} T_i K_i \bphi^\top ( T_i K_i \bLambda + \bLambda_{0,i} + \bLamoff )^{-1} \bphi +  \frac{\log |\Phi|}{\eta_i} 
\end{align*}
By our choice of $\eta_i = 2^{2i/5}$, and $K_i = 2^{i}$, $T_i = 2^i$, we can ensure that
\begin{align*}
K_i T_i \ge \frac{6}{\epsilon} \frac{\log |\Phi|}{\eta_i}
\end{align*}
as long as $i \ge \frac{2}{5} \log_2 [ \frac{6\log |\Phi|}{\epsilon} ]$. To ensure that
\begin{align*}
T_i K_i \ge \frac{6}{\epsilon}  \inf_{\bLambda \in \bOmega} \max_{\bphi \in \Phi} T_i K_i \bphi^\top ( T_i K_i \bLambda + \bLambda_{0,i} + \bLamoff )^{-1} \bphi 
\end{align*}
it suffices to take 
\begin{align*}
i \ge \argmin_{i} i \quad \text{s.t.} \quad \inf_{\bLambda \in \bOmega} \max_{\bphi \in \Phi} \bphi^\top ( 2^{3i} \bLambda + \bLambda_{0,i} + \bLamoff )^{-1} \bphi  \le \frac{\epsilon}{6}.
\end{align*}
Since we assume that we can lower bound $\bLambda_{0,i} \succeq \bLambda_0$ for each $i$, so this can be further simplified to
\begin{align}\label{eq:linear_mdp_i_design_suff}
i \ge \argmin_{i} i \quad \text{s.t.} \quad \inf_{\bLambda \in \bOmega} \max_{\bphi \in \Phi} \bphi^\top ( 2^{3i} \bLambda + \bLambda_{0} + \bLamoff )^{-1} \bphi  \le \frac{\epsilon}{6}.
\end{align}

We next want to show that
\begin{align*}
T_i K_i \ge \frac{3}{\epsilon} \cdot \frac{\beta_i R^2 (\log T_i + 3)}{2 T_i}.
\end{align*}
Bounding $\beta_i \le \eta_i \beta$, a sufficient condition for this is that 
\begin{align*}
i \ge \frac{2}{5} \left ( \log_2 (12\beta R^2 i) + \log_2 \frac{1}{\epsilon} \right ).
\end{align*}
By \Cref{lem:log_inequality}, it suffices to take
\begin{align}\label{eq:fw_i_cond1}
i \ge \frac{6}{5} \log_2 ( 9 \beta R^2 \log_2 \frac{1}{\epsilon}) +  \frac{2}{5} \log_2 \frac{1}{\epsilon}
\end{align}
to meet this condition (this assumes that $12 \beta R^2 \ge 1$ and $\frac{1}{4} \log_2 \frac{1}{\epsilon} \ge 1$---if either of these is not the case we can just replace them with 1 without changing the validity of the final result).

Finally, we want to ensure that
\begin{align*}
T_i K_i \ge \frac{3}{\epsilon} \left ( \sqrt{\frac{4 M^2 \log (8 i^2T_i/\delta) }{K_i}} + \sqrt{\frac{c_1 M^2 d^4 H^4 \log^{3}(8 i^2HK_i T_i/\delta)}{K_i}} + \frac{c_2 M  d^4 H^3 \log^{7/2}(4 i^2HK_iT_i/\delta)}{K_i} \right ).
\end{align*}
To guarantee this, it suffices that
\begin{align*}
2^{5i/2} \ge \frac{c}{\epsilon} \sqrt{M^2 d^4 H^4 i^3 \log^3(i H/\delta)}, \quad 2^{3i} \ge \frac{c}{\epsilon} \cdot M d^4 H^3 i^{7/2} \log^{7/2}(i H/\delta)
\end{align*}
or, equivalently, 
\begin{align*}
i \ge \frac{4}{5} \log_2( c M d H i \log( H/\delta)) + \frac{2}{5} \log_2 \frac{1}{\epsilon}, \quad i \ge \frac{4}{3} \log_2 ( c M d H \log(H/\delta)) + \frac{1}{3} \log_2 \frac{1}{\epsilon}.
\end{align*}
By \Cref{lem:log_inequality}, it then suffices to take
\begin{align}\label{eq:fw_i_cond2}
\begin{split}
& i \ge  \frac{12}{5} \log(c M d H \log(H/\delta) \log_2 1/\epsilon) + \frac{2}{5} \log_2 \frac{1}{\epsilon}, \\
& i \ge 4 \log_2 ( c M d H \log(H/\delta) \log_2 1/\epsilon) + \frac{1}{3} \log_2 \frac{1}{\epsilon}
\end{split}
\end{align}

Thus, a sufficient condition to guarantee \eqref{eq:linear_mdp_suff_KT} is that $i$ is large enough to satisfy \eqref{eq:linear_mdp_i_design_suff}, \eqref{eq:fw_i_cond1}, and \eqref{eq:fw_i_cond2} and $i \ge \frac{2}{5} \log_2 [ \frac{6\log |\Phi|}{\epsilon} ]$.

If $\ihat$ is the final round, the total complexity scales as
\begin{align*}
\sum_{i=1}^{\ihat} T_i K_i = \sum_{i=1}^{\ihat} 2^{3i} \le 2 \cdot 2^{3 \ihat} .
\end{align*}
Using the sufficient condition on $i$ given above, we can bound the total complexity as 
\begin{align*}
\max \bigg \{   \min_{N}  & \ 16N \quad \text{s.t.} \quad \inf_{\bLambda \in \bOmega} \max_{\bphi \in \Phi} \bphi^\top ( N \bLambda + \bLambda_{0} + \bLamoff )^{-1} \bphi  \le \frac{\epsilon}{6}, \\
& \frac{\poly(\beta,R,d,H,M,\log 1/\delta, \log |\Phi| )}{\epsilon^{4/5}} \bigg \}.
\end{align*}
\end{proof}

\newcommand{\gamphi}{\gamma_{\Phi}}
\newcommand{\Gopts}{\widetilde{XY}}
\newcommand{\Gopt}{XY}

\begin{theorem}\label{thm:fw_full_guarantee}
Consider running \optcov with some $\epsexp > 0$ and functions $f_i$ as defined in \Cref{thm:regret_data_fw}, for $\bLambda_{i,0}$ the matrix returned by running \condcov \citep{wagenmaker2022instance} with $N \leftarrow T_i K_i$, $\delta \leftarrow \delta/(2i^2)$, and some $\lamun \ge 0$. 

Then with probability $1-2\delta$, this procedure will collect at most
\begin{align*}
\max \bigg \{   \min_{N}  & \ C \cdot N \quad \text{s.t.} \quad \inf_{\bLambda \in \bOmega} \max_{\bphi \in \Phi} \bphi^\top ( N( \bLambda + \lambar I) + \bLamoff )^{-1} \bphi  \le \frac{\epsexp}{6}, \\
& \frac{\poly(d,H,\frac{1}{\lamminst},\log 1/\delta, \lamun, \log |\Phi| )}{\epsexp^{4/5}} \bigg \}
\end{align*}
episodes, where 
$$\lambar =  \min \left \{ \frac{(\lamminst)^2}{d}, \frac{\lamminst}{d^3 H^3 \log^{7/2} 1/\delta} \right \} \cdot  \poly\log \left ( \frac{1}{\lamminst}, d, H , \lamun, \log \frac{1}{\delta} \right )^{-1} , $$ 
and will produce covariates $\bSighat$ such that
\begin{align*}
\max_{\bphi \in \Phi} \| \bphi \|_{(\bSighat + \bLambda_{i,0} + \bLamoff)^{-1}}^2 \le \epsexp
\end{align*}
and
\begin{align*}
\lammin(\bSighat + \bLambda_{i,0} + \bLamoff) \ge \max \{ d \log 1/\delta, \lamun \}.
\end{align*}
\end{theorem}
\begin{proof}
This proof follows closely the proof of Theorem 9 of \cite{wagenmaker2022instance}.

Note that the total failure probability of our calls to \condcov is at most
\begin{align*}
\sum_{i=1}^\infty \frac{\delta}{2i^2} = \frac{\pi^2}{12} \delta \le \delta.
\end{align*}
For the remainder of the proof, we will then assume that we are on the success event of \condcov, as defined in Lemma D.8 of \cite{wagenmaker2022instance}.

By Lemma D.5 of \cite{wagenmaker2022instance}, it suffices to bound the smoothness constants of $f_i(\bLambda)$ by
\begin{align*}
& L_i = \| \bLambda_{i,0}^{-1} \|_\op^2, \quad \beta_i = 2 \| \bLambda_{i,0}^{-1}\|_\op^3 ( 1 + \eta_i \| \bLambda_{i,0}^{-1}\|_\op), \quad M_i = \| \bLambda_{i,0}^{-1} \|_\op^2 .
\end{align*}

By Lemma D.8 of \cite{wagenmaker2022instance}, on the success event of \condcov we have that 
\begin{align*}
\lammin(\bLambda_{i,0}) \ge \min \left \{ \frac{(\lamminst)^2}{d}, \frac{\lamminst}{d^3 H^3 \log^{7/2} 1/\delta} \right \} \cdot  \poly\log \left ( \frac{1}{\lamminst}, d, H , \lamun, i, \log \frac{1}{\delta} \right )^{-1} =: \lambar.
\end{align*}
Thus, we can bound, for all $i$:
\begin{align*}
& L_i = M_i \le  \max \left \{ \frac{d^2}{(\lamminst)^4}, \frac{d^6 H^6 \log^{7} 1/\delta} {(\lamminst)^2}\right \} \cdot  \poly\log \left ( \frac{1}{\lamminst}, d, H , \lamun, i, \log \frac{1}{\delta} \right ), \\
& \beta_i \le \eta_i \cdot  \poly \left ( d, H, \log 1/\delta, \frac{1}{\lamminst}, \lamun, i, \log | \Phi | \right ) .
\end{align*}
Assume that the termination condition of \optcov is met for $\ihat$ satisfying
\begin{align}\label{eq:xy_design_ihat_bound}
\ihat \le \log \left ( \poly \left ( \frac{1}{\epsexp}, d, H, \log 1/\delta, \frac{1}{\lamminst}, \lamun, \log |\Phi| \right ) \right ).
\end{align}
We assume this holds and justify it at the conclusion of the proof. For notational convenience, define
\begin{align*}
\iota := \poly \left ( \log \frac{1}{\epsexp}, d, H, \log 1/\delta, \frac{1}{\lamminst}, \lamun, \log |\Phi| \right ) .
\end{align*}
Given this upper bound on $\ihat$, set 
\begin{align*}
& L = M := \max \left \{ \frac{d^2}{(\lamminst)^4}, \frac{d^6 H^6 \log^{7} 1/\delta} {(\lamminst)^2}\right \} \cdot  \poly\log \iota, \qquad \beta := \iota.
\end{align*}
With this choice of $L,M,\beta$, we have $L_i \le L, M_i \le M,\beta_i \le \eta_i \beta$ for all $i \le \ihat$. 

Since on the success event of \condcov we have have $\bLambda_{i,0} \succeq \lambar \cdot I$, we can apply \Cref{thm:regret_data_fw} with $\bLambda_0 = \lambar \cdot I$
and get that, with probability at least $1-\delta$, \optcov terminates after at most
\begin{align*}
\max \bigg \{   \min_{N}  & \ 16N \quad \text{s.t.} \quad \inf_{\bLambda \in \bOmega} \max_{\bphi \in \Phi} \bphi^\top ( N \bLambda + \lambar \cdot I + \bLamoff )^{-1} \bphi  \le \frac{\epsexp}{6}, \\
& \frac{\poly(d,H,\lamun,1/\lamminst, \log 1/\epsexp, \log 1/\delta, \log |\Phi| )}{\epsexp^{4/5}} \bigg \}.
\end{align*}
episodes, and returns data $\{ \bphi_\tau \}_{\tau =1}^N$ with covariance $\bSighat = \sum_{\tau=1}^N \bphi_\tau \bphi_\tau^\top$ such that 
\begin{align*}
f_{\ihat}(N^{-1} \bSighat) \le N \epsexp,
\end{align*}
where $\ihat$ is the iteration on which \optcov terminates, and choosing $\bLambda_0 = \lambar \cdot I$

By Lemma D.1 of \cite{wagenmaker2022instance} we have
\begin{align*}
N \cdot \max_{\bphi \in \Phi} \| \bphi \|^2_{(\bSighat + \bLambda_{\ihat,0} + \bLamoff)^{-1}} \le f_{\ihat}(N^{-1} \bSighat).
\end{align*}

The final upper bound on the number of episodes collected and the lower bound on the minimum eigenvalue of the covariates follows from Lemma D.8 of \cite{wagenmaker2022instance}.

It remains to justify our bound on $\ihat$, \eqref{eq:xy_design_ihat_bound}. Note that by definition of $\optcov$, if we run for a total of $\bar{N}$ episodes, we can bound $\ihat \le \frac{1}{4} \log_2(\bar{N})$. However, we see that the bound on $\ihat$ given in \eqref{eq:xy_design_ihat_bound} upper bounds $\frac{1}{4} \log_2(\bar{N})$ for $\bar{N}$ the upper bound on the number of samples collected by \optcov stated above. Thus, our bound on $\ihat$ is valid. 
\end{proof}

%% file: body/examples.tex

\section{Leveraging Offline Data Yields a Provable Improvement}\label{sec:improvement}

\begin{proof}[Proof of \Cref{prop:ex_offline_helps}]
Let $\cM^i$, $i \in \{1,2\}$, denote the MDP with three states---an initial state $s_0$ and two other states $s_1,s_2$, and three actions---and:
\begin{align*}
& r_0(s_0,a_1) = 1, \quad r_0(s_0,a_2) = r_0(s_0,a_3) = 0 \\
& P_0(s_1|s_0,a_1) = 1-p, P_0(s_2|s_0,a_1)= p, \\
& P_0(s_1|s_0,a_2) = 1, P_0(s_2|s_0,a_2) = 0 \\
& P_0(s_1|s_0,a_3) = 0, P_0(s_2|s_0,a_3) = 1
\end{align*}
and
\begin{align*}
& r_1(s_1,a_1) = 1/2 + \Delta, \quad r_1(s_1,a_2) = r_1(s_1,a_3) = 1/2 \\
& r_1(s_2,a_i) = 1/2, \quad r_1(s_2,a_j) = 0, j \neq i.
\end{align*}
After taking an action at $h = 1$, the MDP terminates. We set $p = \sqrt{\epsilon}$ and $\Delta = 6 \epsilon$. As this is a tabular MDP, we can simply take the feature vectors to be the standard basis vectors and will have $d = 9$. In addition, as every state and action are easily reachable, we have $\lamminst = \Omega(1)$. 

Let $\frakDoff$ be a dataset formed by playing the logging policy $\pilog$ for $\poly(1/\epsilon)$ episodes, for $\pilog$ specified as:
\begin{align*}
& \pilog_0(a_1|s_0) = \pilog_0(a_2|s_0) = \pilog_0(a_3|s_0) = 1/3, \\
& \pilog_1(a_1|s_1) = \pilog_1(a_2|s_1) = \pilog_1(a_3|s_1) = 1/3, \\
& \pilog_1(a_1|s_2) = \pilog_1(a_2|s_2) = 0, \pilog_1(a_3|s_2) = 1.
\end{align*}
Note that with this logging policy $\frakDoff$ will contain $\poly(1/\epsilon)$ samples from every state-action pair except for $(s_2,a_1)$ and $(s_2,a_2)$, for which it contains 0 samples. Note also that $\Pr_{\cM^1,\pilog}[\cE] = \Pr_{\cM^2,\pilog}[\cE]$ for any event $\cE$---the measures induced by playing $\pilog$ on $\cM^1$ and $\cM^2$ are identical. 

The first two conclusions are then an immediate consequence of this construction, \Cref{lem:ex_online_insuff}, and \Cref{lem:ex_offline_insuff}.

Note that on both $\cM^1$ and $\cM^2$, the suboptimality of any policy $\pi$ is at least $1-\pi_0(a_1|s_0)$. Furthermore, we can bound
\begin{align*}
[\bphi_{\pi,1}]_{(s_2,a)} \le p \pi_0(a_1|s_0) + 1 - \pi_0(a_1|s_0) \le \sqrt{\epsilon} + 1 - \pi_0(a_1|s_0) \le 2 \max \{ \sqrt{\epsilon}, 1 - \pi_0(a_1|s_0)  \}.
\end{align*}
Thus, using our setting of $\frakDoff$, we have, for any $\pi$, 
\begin{align*}
\frac{\| \bphi_{\pi,1} \|_{(N \bLambda + \bLamoff^1)^{-1}}^2}{(\Vst_0 - \Vpi_0)^2 \vee \epsilon^2} & \le \frac{4\max \{ \epsilon, (1 - \pi_0(a_1|s_0))^2 \}}{\min_i N_{s_2,a_i}} \cdot \frac{1}{(1 - \pi_0(a_1|s_0))^2 \vee \epsilon^2} + \cO(1) \\
& \le \frac{4 }{\min_i N_{s_2,a_i}} \cdot \max \{ 1, \frac{1}{\epsilon} \} + \cO(1).
\end{align*}
To ensure that this is $\cO(1)$, we therefore only need to make $N_{s_2,a_i} = \Omega(1/\epsilon)$ for each $i$. Therefore, the additional online samples required to find an $\epsilon$-optimal policy is $\cO(1/\epsilon)$. The final conclusion then follows from \Cref{cor:main_complexity}. 
\end{proof}

\begin{lemma}\label{lem:ex_online_insuff}
Fix $\epsilon \le 1/4$ and assume that $\Delta > \frac{16 \epsilon}{3}$. For $\cM^1$ and $\cM^2$ constructed as in the proof of \Cref{prop:ex_offline_helps}, any purely online algorithm must take at least $\Omega(\frac{1}{\Delta^2} \cdot \log \frac{1}{2.4 \delta})$ episodes to identify an $\epsilon$-optimal policy with probability at least $1-\delta$.
\end{lemma}
\begin{proof}
Assume our algorithm has returned some policy $\pihat$. Let $\ihat = \argmax_a \pihat_1(s_1,a)$. We show that if $\pihat$ is $\epsilon$-optimal, then it must be the case that $\ihat = a_1$.

Note that for $\epsilon \le 1/4$, for $\pihat$ to be $\epsilon$-optimal we must have that $\pihat_0(a_1|s_0) \ge 3/4$ since, even if $\pihat$ plays optimally in $s_1$ and $s_2$, if $\pihat_0(a_1|s_0) < 3/4$ it will be at least $1/4$-suboptimal. 

Now assume that $\pihat$ is suboptimal in state $s_1$ by $\xi$. Then it follows that the total contribution to the suboptimality from state $s_1$ is at least $\frac{3}{4} (1-p) \xi$. Thus, if we assume the policy is $\epsilon$-optimal, it follows that $\xi \le \frac{4 \epsilon}{3 (1-p)} \le \frac{8 \epsilon}{3}$ where the last inequality holds assuming that $p \le 1/2$. Note that
\begin{align*}
\xi = \Delta(1 - \pihat_1(a_1|s_1))
\end{align*}
so it follows that
\begin{align*}
\Delta ( 1 - \pihat_1(a_1|s_1)) \le \frac{8 \epsilon}{3} \iff \pihat_1(a_1|s_1) \ge 1 - \frac{8 \epsilon}{3 \Delta}.
\end{align*}
It follows that for $\Delta > \frac{16 \epsilon}{3}$, we have $\pihat_1(a_1|s_1) > 1/2$, which implies that $\ihat = a_1$. In other words, if $\pihat$ is $\epsilon$-optimal with probability at least $1-\delta$, $\ihat = a_1$ with probability at least $1-\delta$.

By standard lower bounds on bandits (see e.g. \cite{kaufmann2016complexity}), the complexity of identifying the best action in state $s_1$ with probability at least $1-\delta$ scales as $\Omega(\frac{1}{\Delta^2} \cdot \log \frac{1}{2.4 \delta})$. As the above procedure is able to identify the best arm in $s_1$ with probability at least $1-\delta$, it follows that we must have collected at least $\Omega(\frac{1}{\Delta^2} \cdot \log \frac{1}{2.4 \delta})$ samples from $s_1$, which serves as a lower on the total complexity. 
\end{proof}

\begin{lemma}\label{lem:ex_offline_insuff}
For $\cM^1$, $\cM^2$, and $\frakDoff$ constructed as in the proof of \Cref{prop:ex_offline_helps}, any algorithm which returns some policy $\pihat$ without further exploration must have:
\begin{align*}
\max_{i \in \{ 1,2 \}} \Exp_{\frakDoff \sim \cM^i}[\Vst_0(\cM^i) - V_0^{\pihat}(\cM^i)] \ge \frac{3 \sqrt{\epsilon}}{16} .
\end{align*}
\end{lemma}
\begin{proof}
Note that for $\epsilon \le 1/4$, for $\pihat$ to be $\epsilon$-optimal we must have that $\pihat_0(a_1|s_0) \ge 3/4$ since, even if $\pihat$ plays optimally in $s_1$ and $s_2$, if $\pihat_0(a_1|s_0) < 3/4$ it will be at least $1/4$-suboptimal. With $p = \sqrt{\epsilon}$, this implies that we will transition to $s_2$ with probability at least $\frac{3}{4} p = \frac{3}{4} \sqrt{\epsilon}$. Note that on $\cM^i$ the suboptimality of policy $\pihat$ in state $s_2$ is given by
\begin{align*}
\frac{1}{2} ( 1- \pihat_1(a_i|s_2))
\end{align*}
so the total suboptimality that $s_2$ contributes is at least
\begin{align*}
\frac{3}{8} \sqrt{\epsilon} ( 1- \pihat_1(a_i|s_2)).
\end{align*}

It follows that we can lower bound
\begin{align*}
\max_{i \in \{ 1,2 \}} \Exp_{\frakDoff \sim \cM^i}[\Vst_0(\cM^i) - V_0^{\pihat}(\cM^i)] \ge \frac{3 \sqrt{\epsilon}}{16} \Big ( \Exp_{\frakDoff \sim \cM^1}[ 1- \pihat_1(a_1|s_2)] + \Exp_{\frakDoff \sim \cM^2}[ 1- \pihat_1(a_2|s_2)] \Big ).
\end{align*}
Define the function $\psi(\frakDoff) = \I \{ a \neq a_1 \}$ for $a \sim \pihat_1(\cdot \mid s_2)$. Then
\begin{align*}
\Exp_{\frakDoff \sim \cM^1}[ 1- \pihat_1(a_1|s_2)] = \Exp_{\frakDoff \sim \cM^1}[\psi(\frakDoff) = 1]
\end{align*}
and
\begin{align*}
\Exp_{\frakDoff \sim \cM^2}[ 1- \pihat_1(a_2|s_2)] \ge \Exp_{\frakDoff \sim \cM^2}[\psi(\frakDoff) = 0],
\end{align*}
so the total suboptimality is lower bounded by
\begin{align*}
& \frac{3 \sqrt{\epsilon}}{16} \Big ( \Exp_{\frakDoff \sim \cM^1}[\psi(\frakDoff) = 1] + \Exp_{\frakDoff \sim \cM^2}[\psi(\frakDoff) = 0] \Big )  \ge \frac{3 \sqrt{\epsilon}}{16} \Big ( 1 - \TV(\Pr_{\frakDoff \sim \cM^1}, \Pr_{\frakDoff \sim \cM^2}) \Big ).
\end{align*}
However, by the construction of $\frakDoff$, the distribution of $\frakDoff$ is identical under $\cM^1$ and $\cM^2$, so $\TV(\Pr_{\frakDoff \sim \cM^1}, \Pr_{\frakDoff \sim \cM^2}) = 0$, which proves the result. 
\end{proof}

\section{The Cost of Verifiability}\label{sec:verifiability_pfs}

\begin{proof}[Proof of \Cref{prop:verification_lb}]

Assume that some procedure returns a possibly random policy $\pihat$. Let $\ihat = \argmax_i \pihat(i)$. Then if $\pihat$ is $\epsilon$-optimal, it follows that $\ihat = 1$, since the suboptimality of a policy is given by $(1-\pihat(1)) \cdot 3 \epsilon$, so if $\pihat$ is $\epsilon$-optimal, it follows that $\pihat(1) \ge 2\epsilon/3$. 

We can then apply standard lower bounds for multi-armed bandits (see e.g. \cite{kaufmann2016complexity}) to get that any $(\epsilon,\delta)$-PAC algorithm must collect at least $\Omega(\frac{1}{9\epsilon^2} \log \frac{1}{2.4 \delta})$ samples from each arm $2-A$. 
\end{proof}

\subsection{Proof of \Cref{prop:offline_verifiability}}

\begin{algorithm}
\begin{algorithmic}[1]
\State \textbf{input:} tolerance $\epsilon$, confidence $\delta$, policy set $\Pi$
\State  $\Pi_{1} \leftarrow \Pi$, $\bphihat^{1}_{\pi,1} \leftarrow \Exp_{a \sim \pi_1(\cdot | s_1)} [ \bphi(s_1,a)], \forall \pi \in \Pi$, $\lambda \leftarrow 1/d$
\For{$h=1,2,\ldots,H$}
\For{$\pi \in \Pi_{1}$}
			\State $\bphihat_{\pi,h+1} \leftarrow \Big ( \sum_{(s_h,a_h,r_h,s_{h+1}) \in \frakDoff^h } \bphi_{\pi,h+1}(s_{h+1}) \bphi(s_h,a_h)^\top ( \bLamoff^h)^{-1} \Big ) \bphihat_{\pi,h}$
		\EndFor
		\State $\bthetahat_h \leftarrow ( \bLamoff^h)^{-1} \sum_{(s_h,a_h,r_h,s_{h+1}) \in \frakDoff^h } \bphi(s_h,a_h) r_h$
\EndFor
\For{$\ell = 1,2, \ldots, \lceil \log \frac{4}{\epsilon} \rceil$}
	\State $\epsilon_\ell \leftarrow 2^{-\ell}$, $\beta_\ell \leftarrow  H^4 \left ( 2 \sqrt{d \log \frac{\lambda + \Toff /d}{\lambda}+ 2\log \frac{2 H^2 | \Pi| \ell^2}{\delta}}  + \sqrt{d \lambda}  \right )^2$
	\For{$h = 1,2,\ldots,H$}
		\If{$\max_{\pi \in \Pi_\ell} \| \bphihat_{\pi,h} \|_{(\bLamoff^h)^{-1}}^2 > \epsilon_\ell^2/\beta_\ell$ or $\lammin(\bLamoff^h) < \log \frac{4 H^2 |\Pi| \ell^2}{\delta}$}\label{line:full_offline_check_data}
		\State \textbf{return} $\emptyset$
		\EndIf
	\EndFor
	\State Set
	\begin{align*}
	\Pi_{\ell+1} \leftarrow  \Pi_\ell \backslash \Big \{ \pi \in \Pi_\ell \ : \ \Vhat_0^\pi < \sup_{\pi' \in \Pi_\ell} \Vhat_0^{\pi'} - 2\epsilon_\ell \Big \} \quad \text{for} \quad \Vhat_0^\pi := \tsum_{h=1}^H \inner{\bphihat_{\pi,h}}{\bthetahat_h}
	\end{align*}
	\If{$|\Pi_{\ell+1}| = 1$} \textbf{return} $\pi \in \Pi_{\ell+1}$\EndIf
\EndFor
\State \textbf{return} any $\pi \in \Pi_{\ell+1}$
\end{algorithmic}
\caption{Verifiable Offline RL}
\label{alg:full_offline_pedel}
\end{algorithm}

\begin{proof}[Proof of \Cref{prop:offline_verifiability}]
We prove \Cref{prop:offline_verifiability} for the algorithm outlined in \Cref{alg:full_offline_pedel}.

Note that the condition
\begin{align*}
\max_{\pi \in \Pi_\ell} \| \bphihat_{\pi,h}^\ell \|_{(\bLamoff^h)^{-1}}^2 \le \epsilon_\ell^2/\beta_\ell \quad \text{and} \quad \lammin(\bLamoff^h) \ge \log \frac{4 H^2 |\Pi| \ell^2}{\delta}
\end{align*}
is precisely the condition required on \Cref{line:pedel_fw_explore} of epoch $\ell$ and step $h$ of \algname. Thus, if \Cref{alg:full_offline_pedel} returns some policy $\pihat \neq \emptyset$, it is $\epsilon$-optimal with probability at least $1-\delta$ by an argument identical to the proof of \Cref{cor:main_complexity}. We omit details for the sake of brevity. As we do not need to run any additional exploration, the $\cO(1/\epsilon^{3/2})$ term present in \Cref{cor:main_complexity} is no longer incurred here. 

It remains to show that the coverage condition given in \Cref{prop:offline_verifiability}, \eqref{eq:offline_verifiability}, suffices to ensure that the if statement on \Cref{line:full_offline_check_data} is never true. That $\lammin(\bLamoff^h) \ge \log \frac{4 H^2 |\Pi| \ell^2}{\delta}$ for all $h$ is immediate by assumption. We next prove that $\max_{\pi \in \Pi_\ell} \| \bphihat_{\pi,h}^\ell \|_{(\bLamoff^h)^{-1}}^2 \le \epsilon_\ell^2/\beta_\ell$ inductively. The base case is immediate by assumption. Now assume that, for all $h' < h$,
\begin{align*}
\max_{\pi \in \Pi_\ell} \| \bphihat_{\pi,h'}^\ell \|_{(\bLamoff^{h'})^{-1}}^2 \le \epsilon_\ell^2/\beta_\ell.
\end{align*}
Then by \Cref{lem:est_good_event}, we have that on the event $\cEest^{\ell,h}$,
\begin{align*}
\| \bphihat_{\pi,h}^\ell - \bphi_{\pi,h} \|_2 & \le d \sum_{i=1}^{h} \left ( 2 \sqrt{\log \frac{\det \bLambda_{i,\ell}}{\lambda^d} + 2\log \frac{2 H^2 d | \Pi| \ell^2}{\delta}}  + \sqrt{d \lambda}  \right ) \cdot \frac{\epsilon_\ell^2}{\beta_\ell}  \le \frac{d \epsilon_\ell}{2H}
\end{align*}
where the last inequality follows by our choice of $\beta_\ell$. 

Now note that
\begin{align*}
\max_{\pi \in \Pi_\ell} \| \bphihat_{\pi,h}^\ell \|_{(\bLamoff^h)^{-1}}^2 & \le \max_{\pi \in \Pi_\ell} 2\| \bphi_{\pi,h} \|_{(\bLamoff^h)^{-1}}^2 + 2\| \bphihat_{\pi,h}^\ell  - \bphi_{\pi,h} \|_{(\bLamoff^h)^{-1}}^2 \\
& \le \max_{\pi \in \Pi_\ell} 2\| \bphi_{\pi,h} \|_{(\bLamoff^h)^{-1}}^2 + \frac{2\| \bphihat_{\pi,h}^\ell  - \bphi_{\pi,h} \|_2^2}{\lammin(\bLamoff^h)} \\
& \le \max_{\pi \in \Pi_\ell} 2\| \bphi_{\pi,h} \|_{(\bLamoff^h)^{-1}}^2 + \frac{d^2 \epsilon_\ell^2}{2 H^2\lammin(\bLamoff^h)} \\
& \le \max_{\pi \in \Pi(4\epsilon_\ell)} 2\| \bphi_{\pi,h} \|_{(\bLamoff^h)^{-1}}^2 + \frac{d^2 \epsilon_\ell^2}{2 H^2\lammin(\bLamoff^h)}
\end{align*}
where the last inequality follows by \Cref{lem:correctness}, which gives that $\Pi_\ell \subseteq \Pi(4\epsilon_\ell)$. It follows that as long as
\begin{align*}
\max_{\pi \in \Pi(4\epsilon_\ell)} 2\| \bphi_{\pi,h} \|_{(\bLamoff^h)^{-1}}^2 \le \frac{\epsilon_\ell^2}{2\beta_\ell} \quad \text{and} \quad \frac{d^2 \epsilon_\ell^2}{2 H^2\lammin(\bLamoff^h)} \le \frac{\epsilon_\ell^2}{2\beta_\ell}
\end{align*}
the inductive hypothesis holds. The latter condition follows by assumption. The former also holds since
\begin{align*}
\max_{\pi \in \Pi(4\epsilon_\ell)} 2\| \bphi_{\pi,h} \|_{(\bLamoff^h)^{-1}}^2 \le \frac{\epsilon_\ell^2}{2\beta_\ell} &  \iff \max_{\pi \in \Pi(4\epsilon_\ell)} \frac{\| \bphi_{\pi,h} \|_{(\bLamoff^h)^{-1}}^2}{\epsilon_\ell^2} \le \frac{1}{4 \beta_\ell} \\
&  \iff \max_{\pi \in \Pi(4\epsilon_\ell)} \frac{\| \bphi_{\pi,h} \|_{(\bLamoff^h)^{-1}}^2}{\epsilon_\ell^2 \vee (\Vst_0 - \Vpi_0)^2/16} \le \frac{1}{4 \beta_\ell}
\end{align*}
but this is implied by our coverage assumption, \eqref{eq:offline_verifiability}. By a union bound, it follows that \eqref{eq:offline_verifiability} is sufficient to guarantee that for each $h$ and $\ell$, with high probability,
\begin{align*}
\max_{\pi \in \Pi_\ell} \| \bphihat_{\pi,h}^\ell \|_{(\bLamoff^h)^{-1}}^2 \le \epsilon_\ell^2/\beta_\ell \quad \text{and} \quad \lammin(\bLamoff^h) \ge \log \frac{4 H^2 |\Pi| \ell^2}{\delta}.
\end{align*}
\end{proof}

\subsection{Policy Verification}\label{sec:policy_verification_pfs}

\begin{lemma}\label{lem:linear_mdp_est_transition_verification}

Assume that we have some dataset $\frakD = \{ (s_{h-1,\tau}, a_{h-1,\tau}, r_{h-1,\tau}, s_{h,\tau}) \}_{\tau = 1}^K$ satisfying \Cref{asm:offline_data}.
Denote $\bphi_{h-1,\tau} = \bphi(s_{h-1,\tau},a_{h-1,\tau})$ and $\bLambda_{h-1} = \sum_{\tau=1}^K \bphi_{h-1,\tau} \bphi_{h-1,\tau}^\top + \lambda I$.

Assume that $\Pi$ satisfies \Cref{asm:policy_set_cover}, and let $\Picov^\gamma$ denote the corresponding cover of $\Pi$. For each $\pi$, let 
\begin{align*}
\cThat_{\pi,h} = \left (  \sum_{\tau=1}^K \bphi_{\pi,h}(s_{h,\tau}) \bphi_{h-1,\tau}^\top \right ) \bLambda_{h-1}^{-1}.
\end{align*}
Consider some $\pihat \in \Pi$, where $\pihat$ might be correlated with $\frakD$. 
Fix $\bu \in \R^d$ and $\bv \in \R^d$ satisfying $|\bv^\top \bphi_{\pi,h}(s)| \le 1$ for all $s$.
Then with probability at least $1-\delta$, we can bound, for all $\pi \in \Picov^\gamma \cup \{ \pihat \}$:
\begin{align*}
| \bv^\top (\cT_{\pi,h} - \cThat_{\pi,h}) \bu | \le \left ( 2 \sqrt{ \log \frac{\det \bLambda_{h-1}}{\lambda^d} + 2\log \frac{\Ncov(\Pi,\gamma)}{\delta}} + \sqrt{\lambda} \| \cT_{\pi,h}^\top \bv \|_2 \right ) \cdot \| \bu \|_{\bLambda_{h-1}^{-1}} + \frac{2K \cdot \gamma}{\sqrt{\lambda}} .
\end{align*}
\end{lemma}
\begin{proof}
Construct a filtration as in \Cref{lem:linear_mdp_est_transition}.
Let $\cE$ denote the event that, for each $\pi \in \Picov^\gamma$, 
\begin{align*}
\left \| \sum_{\tau=1}^K \bv^\top \left (  \Exp[\bphi_{\pi,h}(s_{h,\tau}) | \cF_{h-1,\tau-1}]  - \bphi_{\pi,h}(s_{h,\tau}) \right ) \bphi_{h-1,\tau}^\top  \bLambda_{h-1}^{-1/2} \right \|_2 \le 2 \sqrt{ \log \frac{\det \bLambda_{h-1}}{\lambda^d} + 2\log \frac{\Ncov(\Pi,\gamma)}{\delta}}.
\end{align*}
Then using the same argument as in the proof of \Cref{lem:linear_mdp_est_transition}, we have $\Pr[\cE] \ge 1-\delta$, and on $\cE$, we can immediately bound
\begin{align*}
| \bv^\top (\cT_{\pi,h} - \cThat_{\pi,h}) \bu | \le \left ( 2 \sqrt{ \log \frac{\det \bLambda_{h-1}}{\lambda^d} + 2\log \frac{\Ncov(\Pi,\gamma)}{\delta}} + \sqrt{\lambda} \| \cT_{\pi,h}^\top \bv \|_2 \right ) \cdot \| \bu \|_{\bLambda_{h-1}^{-1}}, \quad \forall \pi \in \Picov^\gamma.
\end{align*}
 It remains to bound the error in the estimate $\cThat_{\pihat,h}$. Following the proof of \Cref{lem:linear_mdp_est_transition}, we have that
\begin{align*}
| \bv^\top (\cT_{\pihat,h} -  \cThat_{\pihat,h}) \bu | & \le \| \bu \|_{\bLambda_{h-1}^{-1}} \cdot \underbrace{\left \| \sum_{\tau=1}^K \bv^\top \left (  \Exp[\bphi_{\pihat,h}(s_{h,\tau}) | \cF_{h-1,\tau-1}]  - \bphi_{\pihat,h}(s_{h,\tau}) \right ) \bphi_{h-1,\tau}^\top  \bLambda_{h-1}^{-1/2} \right \|_2}_{(a)} + \Big | \lambda \bv^\top \cT_{\pihat,h} \bLambda_{h-1}^{-1} \bu \Big | .
\end{align*}
Let $\pitil \in \Picov^\gamma$ denote the policy satisfying
\begin{align*}
\| \bphi_{\pihat,h}(s) - \bphi_{\pitil,h}(s) \|_2 \le \gamma, \quad \forall s, h.
\end{align*}
Note that such a $\pitil$ is guaranteed to exist under \Cref{asm:policy_set_cover}. 
We can bound
\begin{align*}
(a) & = \bigg \| \sum_{\tau=1}^K \bv^\top \left (  \Exp[\bphi_{\pitil,h}(s_{h,\tau}) | \cF_{h-1,\tau-1}]  - \bphi_{\pitil,h}(s_{h,\tau}) \right ) \bphi_{h-1,\tau}^\top  \bLambda_{h-1}^{-1/2} \\
& \qquad + \sum_{\tau=1}^K \bv^\top \left (  \Exp[\bphi_{\pihat,h}(s_{h,\tau}) - \bphi_{\pitil,h}(s_{h,\tau}) | \cF_{h-1,\tau-1}]  - \bphi_{\pihat,h}(s_{h,\tau}) + \bphi_{\pitil,h}(s_{h,\tau}) \right ) \bphi_{h-1,\tau}^\top  \bLambda_{h-1}^{-1/2} \bigg \|_2 \\
& \le  \bigg \| \sum_{\tau=1}^K \bv^\top \left (  \Exp[\bphi_{\pitil,h}(s_{h,\tau}) | \cF_{h-1,\tau-1}]  - \bphi_{\pitil,h}(s_{h,\tau}) \right ) \bphi_{h-1,\tau}^\top  \bLambda_{h-1}^{-1/2} \bigg \|_2 \\
& \qquad  + \| \bLambda_{h-1}^{-1/2} \|_\op \sum_{\tau = 1}^K \Big (\Exp[\| \bphi_{\pihat,h}(s_{h,\tau}) - \bphi_{\pitil,h}(s_{h,\tau}) \|_2 | \cF_{h-1,\tau-1}]  + \|  \bphi_{\pihat,h}(s_{h,\tau}) - \bphi_{\pitil,h}(s_{h,\tau})  \|_2 \Big ) \\
& \le  \bigg \| \sum_{\tau=1}^K \bv^\top \left (  \Exp[\bphi_{\pitil,h}(s_{h,\tau}) | \cF_{h-1,\tau-1}]  - \bphi_{\pitil,h}(s_{h,\tau}) \right ) \bphi_{h-1,\tau}^\top  \bLambda_{h-1}^{-1/2} \bigg \|_2   + \frac{2\gamma K}{\sqrt{\lambda}} .
\end{align*}
The result then follows since, on $\cE$, this is bounded, since $\pitil \in \Picov^\gamma$. 
\end{proof}

\begin{proof}[Proof of \Cref{cor:verification_complexity}]
\Cref{cor:verification_complexity} can be proved almost identically to \Cref{cor:main_complexity}, but using \Cref{lem:linear_mdp_est_transition_verification} to bound the error in the estimate of $\bphihat_{\pihat,h}$ in place of \Cref{lem:linear_mdp_est_transition}. In addition, for the $i$th call to \algnamese, we set $\gamma = \frac{\sqrt{\lambda}}{2 (\Toff + \Tonbar^i)} \cdot \epsilon$, as $\Toff + \Tonbar^i$ will upper bound the total number of samples in the $i$th call, resulting in the approximation error in \Cref{lem:linear_mdp_est_transition_verification} to be $\epsilon$. 

As the goal is simply to determine whether or not $\pihat$ is $\epsilon$-optimal, we terminate early if it is determined to be suboptimal, which yields the modified scaling in the policy gap.

\subsubsection{Common Policy Classes Satisfy \Cref{asm:policy_set_cover}}
We remark briefly on \Cref{asm:policy_set_cover} and which policy classes satisfy it. A common policy class we might consider is one parameterized by some vectors $\bw := (\bw_h)_{h=1}^H$, and given by
\begin{align*}
\pi_h^{\bw}(s) = \argmax_{a \in \cA} \ \inner{\bphi(s,a)}{\bw_h}.
\end{align*}
Note that the optimal policy takes this form \citep{jin2020provably}. Policies of this form are in fact non-smooth, and therefore do not obviously satisfy \Cref{asm:policy_set_cover}. However, as shown in Section A.3 of \cite{wagenmaker2022instance}, they can be approximated arbitrarily well with linear softmax policies. As linear softmax policies are smooth, it is straightforward to show that they satisfy \Cref{asm:policy_set_cover} with $\Ncov(\Pi,\gamma) = \cOtil(dH^2 \cdot \log \frac{1}{\gamma})$.

Another common policy class found in the literature (see e.g. \cite{jin2020provably} or \cite{jin2021pessimism}) takes the form
\begin{align*}
\pi_h^{\bw, \bLambda}(s) = \argmax_{a \in \cA} \ \inner{\bphi(s,a)}{\bw_h} + \beta \| \bphi(s,a) \|_{\bLambda_h^{-1}},
\end{align*}
for some $\bLambda_h \succeq 0$. This policy is again non-smooth, but can similarly be approximated arbitrarily well using a softmax policy class (where now the softmax is taken over $\{ \inner{\bphi(s,a)}{\bw_h} + \beta \| \bphi(s,a) \|_{\bLambda_h^{-1}} \}_{a \in \cA}$). As this policy class has $\cO(d^2)$ parameters, we now have that $\Ncov(\Pi,\gamma)$ scales with $d^2$. 

More generally, any policy class that is ``smooth'' in $\bphi(s,a)$, or can be approximated by a policy class that is smooth in $\bphi(s,a)$, and that has a finite number of parameters will satisfy \Cref{asm:policy_set_cover}.

\end{proof}

%% file: body/lower_bounds.tex

\section{Lower Bounds}\label{sec:lower_bounds}

\subsection{Instance-Dependent Lower Bound}\label{sec:instance_lb}

\begin{theorem}\label{prop:lb_instance_no_diff}
Let $\frakM$ denote the set of linear MDPs with reward vectors in $\R^d$. 
Fix some $h  \in [H]$ and consider some algorithm that is $(0,\delta)$-PAC on $\frakM$ and has knowledge of the dynamics of each MDP, and let $\tau$ denote the number of online samples collected by this algorithm. 
Then there exists some MDP $\cM \in \frakM$ with $\cO(1)$ states and actions such that, for any set of offline data $\frakDoff$ satisfying \Cref{asm:offline_data} on $\cM$, we must have $\Exp^{\cM}[\tau] \ge \Omega(N_h^{\mathrm{det}}(\frakDoff,0; \log \frac{1}{\delta}))$, for
\begin{align*}
N_h^{\mathrm{det}}(\frakDoff) :=  \min_N N \quad \text{s.t.} \quad \inf_{\bLambda \in \bOmega_h} \max_{\pi \in \Pidet, \pi \neq \pist} \frac{\| \bphi_{\pi,h} \|_{(N \bLambda + \bLamoff^h)^{-1}}^2}{(\Vst_0 - \Vpi_0)^2} \le \frac{1}{\log 1/\delta} ,
\end{align*}
where $\Pidet$ is the set of all deterministic policies. 
\end{theorem}

\begin{remark}[$\Pidet$ vs $\Pilsm$ Dependence]
Note that \Cref{prop:lb_instance_no_diff} depends on $\Pidet$ and not $\Pilsm$, the policy class our upper bound \Cref{cor:main_complexity} scales with. Our upper bounds were in fact proved for a generic policy class $\Pi$, and provide a guarantee on finding a policy $\pihat$ that is $\epsilon$-optimal with respect to the best policy in $\Pi$. In the case of linear MDPs, to construct a policy class that is guaranteed to have a policy $\epsilon$-optimal with respect to the best possible policy on the MDP, we rely on the linear softmax policy construction. However, it is well known that in tabular MDPs, there always exists a deterministic optimal policy. Thus, in tabular settings, it suffices to choose $\Pi \leftarrow \Pidet$ to guarantee we have a globally $\epsilon$-optimal policy in our class. Given this, in the setting of \Cref{prop:lb_instance_no_diff} with a finite number of states and actions, we can simply run \algname with $\Pi \leftarrow \Pidet$ and will obtain the same guarantee as in \Cref{cor:main_complexity} but with $\Pilsm$ replaced by $\Pidet$, matching the guarantee of \Cref{prop:lb_instance_no_diff}.
\end{remark}

\subsubsection{Linear Bandits with Random Arms}
Towards proving a lower bound for linear MDPs, we first consider the setting of linear bandits with random arms. In this setting, assume we have some set of arms $\cA$. 
When arm $a \in \cA$ is played, a vector $\bz \in \cZ \subseteq \R^d$ is sampled from some distribution $\xi_a$, which we assume is known to the learner. The learner then observes $\bz$ as well as $y = \bz^\top \thetast + \eta$ for some noise $\eta$, and unknown $\bthetast \in \R^d$. 

Note that this can simply be thought of as a multi-armed bandit where arm $a$ has expected reward $\Exp_{\eta}\Exp_{\bz \sim \xi_a}[\bz^\top \bthetast + \eta]$. Define $\bLambda_a := \Exp_{\bz \sim \xi_a}[\bz \bz^\top]$. For some set $\cX \subseteq \R^d$, our goal is to identify $\xst(\bthetast) \in \cX$ such that
\begin{align*}
\xst(\thetast)^\top \thetast \ge \max_{\bx \in \cX} \bx^\top \thetast.
\end{align*}
In the offline-to-online setting, we assume the learner has access to some set of data $\{ (a_s,\bz_s,y_s) \}_{s=1}^{\Toff}$.

Finally, we say that an algorithm is $\delta$-PAC if it returns $\xst(\bthetast)$ with probability at least $1-\delta$. 

\begin{prop}\label{thm:lb_lin_band_rand_arms}
Assume we are in the linear bandit with random arms settings defined above, for noise distribution $\eta \sim \cN(0,1)$. Then any $\delta$-PAC strategy with online stopping time $\tau$ must have
\begin{align*}
\Exp_{\thetast}[\tau] \ge \min_t \sum_{a \in \cA} t_{a} \quad \text{s.t.} \quad \max_{\bx \neq \xst(\thetast) } \frac{\| \bx - \xst(\thetast) \|_{\bA(t)^{-1}}^2}{((\xst(\bthetast)-\bx)^\top \thetast)^2} \le \frac{1}{\log \frac{1}{2.4\delta}}
\end{align*}
for $\bA(t) = \sum_{a \in \cA} t_{a} \bLambda_a + \sum_{s=1}^{\Toff} \bz_s \bz_s^\top$. 
\end{prop}
\begin{proof}
We follow the proof of Theorem 1 of \cite{fiez2019sequential}, with some small modifications. Let $\Thetaaltst = \{ \btheta \in \R^d \ : \ \xst(\btheta) \neq \xst(\thetast) \}$ denote the set of alternate instances. Let $\nu_{\btheta,\bz} = \cN(\btheta^\top \bz, 1)$, $\nu_{\btheta,a}$ the reward distribution of action $a$, and let $\ton_a$ the total number of (online) pulls of arm $a$.

Given some alternate $\btheta$, the KL divergence between $\btheta$ and $\bthetast$ is given by
\begin{align*}
\sum_{s=1}^{\Toff} \KL(\nu_{\bthetast,\bz_s},\nu_{\btheta,\bz_s}) + \sum_{a \in \cA} \Exp_{\bthetast}[\ton_{a}] \KL(\nu_{\bthetast,a},\nu_{\btheta,a}).
\end{align*}
 Then Lemma 1 of \cite{kaufmann2016complexity} and the fact that the algorithm is $\delta$-PAC give that, for any $\btheta \in \Thetaaltst$,
\begin{align*}
\sum_{s=1}^{\Toff} \KL(\nu_{\bthetast,\bz_s},\nu_{\btheta,\bz_s}) + \sum_{a \in \cA} \Exp_{\bthetast}[\ton_{a}] \KL(\nu_{\bthetast,a},\nu_{\btheta,a}) \ge \log \frac{1}{2.4 \delta}.
\end{align*}
Note that, by the convexity of the KL divergence, we have
\begin{align*}
\KL(\nu_{\bthetast,a},\nu_{\btheta,a}) & \le \sum_{\bz \in \cZ} \xi_a(\bz) \KL(\nu_{\bthetast,\bz},\nu_{\btheta,\bz}) \\
& = \sum_{\bz \in \cZ} \xi_a(\bz) (\bz^\top (\bthetast - \btheta))^2 \\
& = (\bthetast - \btheta)^\top \left ( \sum_{\bz \in \cZ} \xi_a(\bz) \bz \bz^\top \right ) (\bthetast - \btheta) \\
& = (\bthetast - \btheta)^\top \bLambda_a (\bthetast - \btheta).
\end{align*}
Since $\tau = \sum_{a \cA} \ton_a$, we have $\Exp_{\bthetast}[\tau] \ge \sum_{a \in \cA} t_a$ for $t$ the solution to
\begin{align}\label{eq:lin_rand_arms_lb_eq1}
\min_t \sum_{\bz \in \cZ} t_{\bz} \quad \text{s.t.} \quad \min_{\btheta \in \Thetaaltst} \sum_{s=1}^{\Toff} \KL(\nu_{\bthetast,\bz_s},\nu_{\btheta,\bz_s}) + \sum_{a \in \cA} t_a (\bthetast - \btheta)^\top \bLambda_a (\bthetast - \btheta) \ge \log \frac{1}{2.4\delta}.
\end{align}

Fix $\epsilon > 0$ and $t$ and, for $\bx \in \cX$, $\bx \neq \xst(\thetast)$, let
\begin{align*}
\btheta_{\bx}(\epsilon,t) = \thetast - \frac{(\bu_{\bx}^\top \thetast + \epsilon) \bA(t)^{-1} \bu_{\bx}}{\bu_{\bx}^\top \bA(t)^{-1} \bu_{\bx}}
\end{align*}
for $\bu_{\bx} = \xst(\bthetast) - \bx$. Note that $\bu_{\bx}^\top \btheta_{\bx}(\epsilon,t) = -\epsilon < 0$ so $\btheta_{\bx}(\epsilon,t) \in \Thetaaltst$. Furthermore, we have
\begin{align*}
\KL(\nu_{\thetast,\bz}, \nu_{\btheta_{\bx}(\epsilon,t),\bz}) = \bu_{\bx}^\top \bA(t)^{-1} \frac{(\bu_{\bx}^\top \thetast + \epsilon)^2 \bz \bz^\top}{(\bu_{\bx}^\top \bA(t)^{-1} \bu_{\bx})^2} \bA(t)^{-1} \bu_{\bx}
\end{align*}
and
\begin{align*}
(\bthetast - \btheta)^\top \bLambda_a (\bthetast - \btheta) =  \bu_{\bx}^\top \bA(t)^{-1} \frac{(\bu_{\bx}^\top \thetast + \epsilon)^2 \bLambda_a}{(\bu_{\bx}^\top \bA(t)^{-1} \bu_{\bx})^2} \bA(t)^{-1} \bu_{\bx}.
\end{align*}
It follows that
\begin{align*}
\eqref{eq:lin_rand_arms_lb_eq1} & \ge \min_t \sum_{\bz \in \cZ} t_{\bz} \quad \text{s.t.} \quad \min_{\btheta \in \Thetaaltst} \frac{(\bu_{\bx}^\top \thetast + \epsilon)^2}{(\bu_{\bx}^\top \bA(t)^{-1} \bu_{\bx})^2} \cdot \bu_{\bx}^\top \bA(t)^{-1} \Big ( \sum_{s=1}^{\Toff} \bz_s \bz_s^\top + \sum_{a \in \cA} t_a \bLambda_a \Big ) \bA(t)^{-1} \bu_{\bx}  \ge \log \frac{1}{2.4\delta} \\
& = \min_t \sum_{\bz \in \cZ} t_{\bz} \quad \text{s.t.} \quad \min_{\btheta \in \Thetaaltst} \frac{(\bu_{\bx}^\top \thetast + \epsilon)^2}{\bu_{\bx}^\top \bA(t)^{-1} \bu_{\bx}}   \ge \log \frac{1}{2.4\delta}.
\end{align*}
Taking $\epsilon \rightarrow 0$ and rearranging proves the result. 
\end{proof}

\subsubsection{Linear MDPs}
To obtain a lower bound for linear MDPs, we show a reduction from linear MDPs to linear bandits with random arms. In the following we will denote $\Piexp$ the set of all possible exploration policies (this could include any possible policy, though in practice it suffices to just take a policy cover). We will also let $\Pidet$ denote the set of all deterministic policies.

Consider the following linear bandit with random arms construction:
\begin{itemize}
\item Choose some set of transition kernels $\{ P_h \}_{h=1}^H$ on an MDP with $|\cS| =: S< \infty$ states and $|\cA| =: A < \infty$ actions. 
\item Choose some set of feature vectors $\bphi(s,a)$ defined on $\cS \times \cA$.
\item Let the set of actions in the linear bandit be $\Piexp$, the set of all possible policies on the MDP constructed above.
\item For some $h$, let $\xi_\pi$ denote the distribution of $\bphi(s_h,a_h)$ under policy $\pi \in \Piexp$. 
\item Set $\bthetast$ as desired, and $\cX = \{ \bphi_{\pi,h} \ : \ \pi \in \Pidet  \}$. 
\end{itemize}
We furthermore assume that there is a unique optimal policy in $\Pidet$. This is a well-specified linear bandit with random arms, so the lower bound from \Cref{thm:lb_lin_band_rand_arms} applies to finding the best decision on $\cX$. To make this reduction precise, we need the following result.

\begin{lemma}\label{lem:visit_det_convex}
For any (possibly random) policy $\pitil$ and fixed $h$, there exists some $p \in \simplex_{\Pidet}$ such that $\sum_{\pi \in \Pidet} p_{\pi} \bphi_{\pi,h} = \bphi_{\pitil,h}$.
\end{lemma}
\begin{proof}
First, note that for any policy $\pi$, we can write 
\begin{align*}
w_{h+1}^\pi = \pi_{h+1} P_{h} w_{h}^\pi
\end{align*}
for $w_h^\pi \in \R^{SA}$ the visitation probabilities, $P_h \in \R^{S \times SA}$ the transition matrix at step $h$, and $\pi_{h+1} \in \R^{SA \times S}$ the policy's probabilities, where $[\pi_{h+1}]_{(sa),s} = \pi_{h+1}(a|s)$ and 0 otherwise.

Fix some policy $\pi$. Let us assume that $w_h^\pi = \sum_{\pitil \in \Pidet} p_{\pitil}^h w_h^{\pitil}$ for some $p^h \in \simplex_{\Pidet}$. Note that there exists some $\ptil \in \simplex_{\Pidet}$ such that $\pi_{h+1} = \sum_{\pi' \in \Pidet} \ptil_{\pi'} \pi_{h+1}'$. Then we have
\begin{align*}
w_{h+1}^\pi & = \pi_{h+1} P_h w_h  = \sum_{\pi' \in \Pidet} \ptil_{\pi'} \pi'_h \cdot P_h \sum_{\pitil \in \Pidet} p_{\pitil}^h w_h^{\pitil}  = \sum_{\pi',\pitil \in \Pidet} \ptil_{\pi'} p_{\pitil}^h \pi_h' P_h w_h^{\pitil}.
\end{align*}
Note that $\pi_{h+1}'  P_h w_h^{\pitil} = w_{h+1}^{\pi''}$ for some $\pi'' \in \Pidet$ (namely the concatenation of $\pitil$ up to step $h$ with $\pi'$ at step $h-1$). It follows that the above can be written as
\begin{align*}
\sum_{\pi',\pitil \in \Pidet} \ptil_{\pi'} p_{\pitil}^h  w_{h+1}^{\pi''(\pitil,\pi')}
\end{align*}
where $\pi''(\pitil,\pi')$ denotes the concatenation of $\pitil$ and $\pi'$ given above. However, this itself can be written in terms of some $p'$ as $\sum_{\pi'' \in \Pidet} p'_{\pi''} w_{h+1}^{\pi''}$, which proves the inductive hypothesis. 

The result then follows since $\bphi_{\pi,h}$ is completely specified by the visitation probabilities of policy $\pi$. 
\end{proof}

Assume we have defined a linear bandit with random actions as above and that we have access to some linear MDP algorithm that will return an $\epsilon$-optimal policy with probability at least $1-\delta$. (possibly one that uses knowledge of the dynamics). Consider running our linear bandit on the linear MDP with transitions $\{ P_h \}_{h=1}^H$ and rewards $\btheta_1 = \ldots = \btheta_{h-1} = \btheta_{h+1} = \ldots = \btheta_H = \bm{0}$ and $\btheta_h = \bthetast$, for reward distribution $r_h(s,a) \sim \cN(\bphi(s,a)^\top \bthetast,1)$\footnote{Note that allowing rewards to be normally distributed violates the linear MDP definition as rewards could now fall outside $[0,1]$. This is done only for simplicity---at the expense of a more complicated calculation, all results in this section can be shown to hold for slightly different constants if the noise is instead Bernoulli (see e.g. Lemma E.1 of \cite{wagenmaker2022instance} for an example).}. Note that this linear MDP can be completely simulated by running our linear bandit with random actions. Assume that our linear MDP algorithm returns some policy $\pihat$ that is $\epsilon$-optimal. Consider the following procedure:
\begin{enumerate}
\item Find weights $p \in \Pidet$ such that $\bphi_{\pihat,h} = \sum_{\pi \in \Pidet} p_\pi \bphi_{\pi,h}$ (note that this is possible by \Cref{lem:visit_det_convex} and since we assume the dynamics are known). 
\item Set $\pitil = \argmax_{\pi \in \Pidet} p_\pi$.
\end{enumerate}
We have the following result.

\begin{lemma}\label{lem:reduction_correctness}
Assume that $\epsilon < \Delmin / 2$ for $\Delmin =  \bphi_{\pist,h}^\top \bthetast - \max_{\pi \in \Pidet, \pi \neq \pist} \bphi_{\pi,h}^\top \bthetast$ and $\pist = \argmax_{\pi \in \Pi} \bphi_{\pi,h}^\top \bthetast$. Then $\pitil = \pist$ as long as $\pihat$ is $\epsilon$-optimal.
\end{lemma}
\begin{proof}
Let $\Vst = \bphi_{\pist,h}^\top \bthetast$ denote the value of the optimal policy. Note that
\begin{align*}
\sum_{\pi \in \Pidet} p_\pi \bphi_{\pi,h}^\top \bthetast = \bphi_{\pihat,h}^\top \bthetast \ge \Vst - \epsilon. 
\end{align*}
However, we also have that for any $\pi \neq \pist$, $\bphi_{\pi,h}^\top \bthetast \le \Vst - \Delmin$, so it follows that
\begin{align*}
\sum_{\pi \in \Pidet} p_\pi \bphi_{\pi,h}^\top \bthetast \le p_{\pist} \Vst + (1-p_{\pist}) (\Vst - \Delmin) = \Vst - (1-p_{\pist})\Delmin.
\end{align*}
Putting these together we have
\begin{align*}
& \Vst - (1-p_{\pist})\Delmin \ge \Vst - \epsilon \implies p_{\pist} \ge 1 - \frac{\epsilon}{\Delmin}. 
\end{align*}
Thus, it follows that if $\epsilon < \Delmin/2$, then $p_{\pist} > 1/2$, which implies that $\pitil = \pist$. 
\end{proof}

\begin{lemma}\label{lem:lin_mdp_lb}
Consider the linear MDP setting outlined above, and assume there is a unique optimal policy in $\Pidet$. For $\epsilon < \Delmin/2$, any $(\epsilon,\delta)$-PAC algorithm given access to some dataset $\frakDoff$ must have
\begin{align*}
\Exp[\tau] \ge \min_t \sum_{\pi \in \Piexp} t_\pi \quad \text{s.t.} \quad \max_{\pi \neq \pist, \pi \in \Pidet} \frac{\| \bphi_{\pi,h} - \bphi_{\pist,h} \|_{\bA(t)^{-1}}^2}{((\bphi_{\pist,h} - \bphi_{\pi,h})^\top \btheta_h)^2} \le \frac{1}{\log 1/2.4\delta}
\end{align*}
where $\pist = \argmax_{\pi \in \Pi} V_0^\pi$ and $\bA(t) = \sum_{\pi \in \Piexp} t_\pi \bLambda_{\pi,h} + \sum_{n=1}^{\Toff} \bphi(s_h^n,a_h^n) \bphi(s_h^n,a_h^n)^\top$. 
\end{lemma}
\begin{proof}
By \Cref{lem:reduction_correctness}, we can simulate a linear MDP with a linear bandit with random arms, and use any $(\epsilon,\delta)$-PAC algorithm to identify the optimal arm in the linear bandit with probability at least $1-\delta$, as long as $\epsilon < \Delmin/2$. It follows that the lower bound from \Cref{thm:lb_lin_band_rand_arms} applies with action set $\cA= \Piexp$, and $\cX$ the set $\{ \bphi_{\pi,h} \}_{\pi \in \Pidet}$.
\end{proof}

\subsubsection{Removing Differences}
The upper bound we obtain scales as $\| \bphi_{\pi,h} \|_{\bA(t)^{-1}}^2$ instead of $\| \bphi_{\pi,h} - \bphi_{\pist,h} \|_{\bA(t)^{-1}}^2$. The following result shows that we can construct linear MDPs where this difference is not too significant.

\begin{lemma}\label{lem:lb_instance_no_diff}
Let
\begin{align*}
\tst = \argmin_t \sum_{\pi \in \Piexp} t_\pi \quad \text{s.t.} \quad \max_{\pi \neq \pist, \pi \in \Pidet} \frac{\| \bphi_{\pi,h} - \bphi_{\pist,h} \|_{\bA(t)^{-1}}^2}{((\bphi_{\pist,h} - \bphi_{\pi,h})^\top \btheta_h)^2} \le \frac{1}{\log 1/2.4\delta}.
\end{align*}
Assume that every state in level $h$ is reached with probability at least $p$ by every policy. Then there exists settings of the feature vectors such that
\begin{align*}
(1 + \frac{1}{p^2}) \| \tst \|_1 \ge \min_t \sum_{\pi \in \Piexp} t_\pi \quad \text{s.t.} \quad \max_{\pi \neq \pist, \pi \in \Pidet} \frac{\| \bphi_{\pi,h} \|_{\bA(t)^{-1}}^2}{((\bphi_{\pist,h} - \bphi_{\pi,h})^\top \btheta_h)^2} \le \frac{1}{\log 1/2.4\delta} . 
\end{align*}
\end{lemma}
\begin{proof}
Consider, for example, the case where for each $s$, $\bphi(s,\pist_h(s)) = \be_1$, and for every $a \neq \pist_h(s)$, $\bphi(s,a)^\top \be_1 = 0$. In this case, $\bA(t)$ will have the form $a \be_1 \be_1^\top + \bB(t)$, for some $a$ and $\bB(t)$ whose first column and row are entirely 0. It follows that $\bphi_{\pist,h} = \be_1$. Note that $|[\bphi_{\pi,h}]_1| \le 1$ for any other $\pi$ as well.

Assume that every reachable state at level $h$ is reached with probability at least $p$ for every roll-in policy. Then we have that $|[\bphi_{\pist,h}]_1 - [\bphi_{\pi,h}]_1| \ge p$ for every $\pi \in \Pidet$. We have
\begin{align*}
\| \bphi_{\pi,h} - \bphi_{\pist,h} \|_{\bA(\tst)^{-1}}^2 = \frac{([\bphi_{\pi,h}]_1 - [\bphi_{\pist,h}])^2}{a} + \| \bphitil_{\pi,h}  \|_{\bB(\tst)^{-1}}^2, \quad \| \bphi_{\pi,h}  \|_{\bA(\tst)^{-1}}^2 = \frac{([\bphi_{\pi,h}]_1 )^2}{a} + \| \bphitil_{\pi,h}  \|_{\bB(\tst)^{-1}}^2.
\end{align*}
Consider playing $\tst + t$ for some $t$ such that every sample goes to $\be_1$. In other words, $\bA(\tst + t) = \bA(\tst) + t \be_1 \be_1^\top$. Then if we take $t \ge \frac{a}{p^2} - a$, we have
\begin{align*}
\| \bphi_{\pi,h}  \|_{\bA(\tst+t)^{-1}}^2 & = \frac{([\bphi_{\pi,h}]_1 )^2}{a + t} + \| \bphitil_{\pi,h}  \|_{\bB(\tst)^{-1}}^2  \\
& \le \frac{1}{a + t} + \| \bphitil_{\pi,h}  \|_{\bB(\tst)^{-1}}^2 \\
& \le \frac{p}{a}  + \| \bphitil_{\pi,h}  \|_{\bB(\tst)^{-1}}^2 \\
& \le \frac{([\bphi_{\pi,h}]_1 - [\bphi_{\pist,h}])^2}{a} + \| \bphitil_{\pi,h}  \|_{\bB(\tst)^{-1}}^2.
\end{align*}
It follows that $\tst + t$ is a feasible solution. Since $\| \tst \|_1 \ge a$, the result follows.

\end{proof}

\begin{proof}[Proof of \Cref{prop:lb_instance_no_diff}]
This follows directly from \Cref{lem:lin_mdp_lb} and \Cref{lem:lb_instance_no_diff}. To make the construction explicit, we take $\bthetast = \be_1$, and construct the feature vectors as in \Cref{lem:lb_instance_no_diff}. The transition kernel can then be any transition kernel such that, regardless of the policy we play, we end up in each state $s \in \cS$ at step $h$ with probability $\Omega(1)$. 
\end{proof}

\subsection{Minimax Lower Bound (\Cref{prop:minimax_lb})}\label{sec:minimax_lb_pf}
Consider the following \emph{multi-dimensional linear bandit} setting. Let $\btheta_h \in \Theta := \{ -\mu, \mu \}^d$ for some $\mu$ to be chosen, and $\btheta = [\btheta_1,\ldots,\btheta_H]$. At step $t$, we can query points $\bphi_{ht} \in \cS^{d-1}$ for each $h \in [H]$, and observe
\begin{align*}
y_{ht} \sim \bern(\inner{\bphi_{ht}}{\btheta_h} + 1/2).
\end{align*}
Our goal will then be to find some $( \bphihat_h )_{h=1}^H$, $\bphihat_h \in \cS^{d-1}$, which is $\epsilon$-optimal in the sense that
\begin{align*}
\sum_{h=1}^H \inner{\bphihat_h}{\btheta_h} \ge \sup_{\bphi_1,\ldots,\bphi_H \in \cS^{d-1}} \sum_{h=1}^H \inner{\bphi_h}{\btheta_h} - \epsilon = \sum_{h=1}^H \| \btheta_h \|_2 - \epsilon.
\end{align*}
We start with the follow regression lower bound.

\begin{lemma}\label{lem:minimax_reg_lb}
Assume that $\mu \in (0, \frac{1}{20 \sqrt{d}}]$ and consider the multi-dimensional linear bandit setting outlined above. Let $\bthetahat$ be some estimator of $\btheta$, and $\pi$ some query strategy. Then, if we query for $T$ steps, we have
\begin{align*}
\inf_{\bthetahat,\pi} \sup_{\btheta \in \Theta^H} \Exp_{\btheta}[\| \btheta - \bthetahat \|_2^2] \ge  \frac{d H \mu^2}{2} \left ( 1 - \sqrt{\frac{20 T \mu^2}{d} } \right ).
\end{align*}
\end{lemma}
\begin{proof}
The proof of this result follows closely the proof of Theorem 5 of \cite{wagenmaker2022reward}, which is itself based on the proof of Theorem 3 of \cite{shamir2013complexity}.

We have
\begin{align*}
\sup_{\btheta \in \Theta^H} \Exp_{\btheta}[\| \btheta - \bthetahat \|_2^2] & = \sup_{\btheta_1,\ldots,\btheta_H \in \Theta} \sum_{h=1}^H \Exp_{\btheta} \left [ \sum_{i=1}^d ( \btheta_{hi} - \bthetahat_{hi})^2 \right ] \\
& \ge \Exp_{\btheta_1,\ldots,\btheta_H \sim \unif(\Theta)} \sum_{h=1}^H \Exp_{\btheta} \left [ \sum_{i=1}^d ( \btheta_{hi} - \bthetahat_{hi})^2 \right ] \\
& \ge  \sum_{h=1}^H \Exp_{\btheta_1,\ldots,\btheta_H \sim \unif(\Theta)} \Exp_{\btheta} \left [ \mu^2 \sum_{i=1}^d \I \{ \btheta_{hi} \bthetahat_{hi}  < 0 \} \right ].
\end{align*}

\begin{lem}[Lemma 4 of \cite{shamir2013complexity}]\label{lem:reg_lb_shamir}
Let $\btheta_h$ be a random vector, none of whose coordinates is supported on 0, and let $(r_1^t,\ldots,r_H^t)_{t=1}^T$ be a sequence of observations obtained by a query strategy where $\bphi_{h}^t$ is a deterministic function of $\cF_{h,t} := \{ (r_1^s,\ldots,r_H^s)_{s=1}^{t-1}, (\bphi_1^s,\ldots,\bphi_H^s)_{s=1}^{t-1}, (r_1^t,\ldots,r_{h-1}^t),(\bphi_1^t,\ldots,\bphi_{h-1}^t) \}$. Let $\bthetahat_h$ be some estimator that is a deterministic function of $(r_1^t,\ldots,r_H^t)_{t=1}^T$ and $(\bphi_1^t,\ldots,\bphi_H^t)_{t=1}^T$. Then we have
\begin{align*}
\Exp_{\btheta_1,\ldots,\btheta_H \sim \unif(\Theta)} \Exp_{\btheta} \left [ \sum_{i=1}^d \I \{ \btheta_{hi} \bthetahat_{hi}  < 0 \} \right ] \ge \frac{d}{2} \left ( 1 - \sqrt{\frac{1}{d} \sum_{i=1}^d \sum_{t=1}^T U_{t,i}^h} \right )
\end{align*}
where
\begin{align*}
U_{t,i}^h = \sup_{\btheta_{hj}, j \neq i; \btheta_{h'}, h' \neq h} \KL \Big (\Pr(r_h^t | \btheta_i > 0, \{ \btheta_{hj} \}_{j \neq i}, \{ \btheta_{h'} \}_{h' \neq h}, \cF_{h,t} ) || \Pr(r_h^t | \btheta_i < 0, \{ \btheta_{hj} \}_{j \neq i}, \{ \btheta_{h'} \}_{h' \neq h}, \cF_{h,t}) \Big ) .
\end{align*}
\end{lem}

Since our rewards are Bernoulli, we have
\begin{align*}
U_{t,i}^h = \sup_{\btheta_{hj}, j \neq i} \KL \Big (\bern(1/2 + \sum_{j \neq i} \btheta_{hj} \bphi_{hj}^t + \mu \bphi_{hi}^t) || \bern(1/2 + \sum_{j \neq i} \btheta_{hj} \bphi_{hj}^t - \mu \bphi_{hi}^t) \Big ) .
\end{align*}
Note that:
\begin{lem}[Lemma 2.7 of \cite{tsybakov2009introduction}]\label{lem:bern_kl}
\begin{align*}
\KL(\bern(p) || \bern(q)) \le \frac{(p-q)^2}{q (1 - q)} .
\end{align*}
\end{lem}
Applying this, we have
\begin{align*}
U_{t,i}^h & \le \frac{(2 \mu \bphi_{hi}^t)^2}{(1/2 + \sum_{j \neq i} \btheta_{hj} \bphi_{hj}^t - \mu \bphi_{hi}^t)(1 - 1/2 - \sum_{j \neq i} \btheta_{hj} \bphi_{hj}^t + \mu \bphi_{hi}^t)}  \le 20 \mu^2 (\bphi_{hi}^t)^2
\end{align*}
where we make use of the fact that $\mu \le \frac{1}{20\sqrt{d}}$ and $\bphi_{h}^t \in \cS^{d-1}$, which implies that $1/2 + \sum_{j \neq i} \btheta_{hj} \bphi_{hj}^t - \mu \bphi_{hi}^t \ge 9/20$ and $1 - 1/2 - \sum_{j \neq i} \btheta_{hj} \bphi_{hj}^t + \mu \bphi_{hi}^t \ge 9/20$. 

We can then lower bound
\begin{align*}
\sum_{h=1}^H \Exp_{\btheta_1,\ldots,\btheta_H \sim \unif(\Theta)} \Exp_{\btheta} \left [ \mu^2 \sum_{i=1}^d \I \{ \btheta_{hi} \bthetahat_{hi}  < 0 \} \right ] & \ge  \sum_{h = 1}^H \frac{d \mu^2}{2} \left ( 1 - \sqrt{\frac{1}{d} \sum_{i=1}^d \sum_{t=1}^T U_{t,i}^h} \right ) \\
& \ge  \sum_{h = 1}^H \frac{d \mu^2}{2} \left ( 1 - \sqrt{\frac{1}{d} \sum_{i=1}^d \sum_{t=1}^T 20 \mu^2 (\bphi_{hi}^t)^2} \right ) \\
& =  \frac{d H \mu^2}{2} \left ( 1 - \sqrt{\frac{20 T \mu^2}{d} } \right )
\end{align*}
where we have used that $\bphi_{ht} \in \cS^{d-1}$. This proves the result. 
\end{proof}

\begin{lemma}\label{lem:minimax_lb_bandit_est_to_policy}
Assume that $( \bphihat_h )_{h=1}^H$ is $\epsilon$-optimal. Then,
\begin{align*}
\sum_{h=1}^H \| \sqrt{d \mu^2} \cdot \bphihat_h - \btheta_h \|_2^2 \le 2 \sqrt{d \mu^2} \cdot \epsilon
\end{align*}
\end{lemma}
\begin{proof}
Denote $\bphihat = [\bphihat_1,\ldots,\bphihat_H]$, so that $\sum_{h=1}^H \inner{\bphihat_h}{\btheta_h}  = \inner{\bphihat}{\btheta}$. If $\bphihat$ is $\epsilon$-optimal, then this implies that
\begin{align}\label{eq:minimax_lb_policy_good}
\inner{\bphihat}{\btheta} \ge \sum_{h=1}^H \| \btheta_h \|_2 - \epsilon = \sqrt{H} \| \btheta \|_2 - \epsilon
\end{align}
where the last inequality follows since, by construction of $\btheta_h$, we have $\| \btheta_h \|_2 = \sqrt{d \mu^2}$ and $\| \btheta \|_2 = \sqrt{d H \mu^2}$. Now note that
\begin{align*}
\| \frac{1}{\sqrt{H}} \| \btheta \|_2 \bphihat - \btheta \|_2^2  & = \frac{1}{H} \| \btheta \|_2^2 \| \bphihat \|_2^2 + \| \btheta \|_2^2 - \frac{2}{\sqrt{H}} \| \btheta \|_2 \cdot \inner{\bphihat}{\btheta} \\
& \overset{(a)}{=}  2\| \btheta \|_2^2 - \frac{2}{\sqrt{H}} \| \btheta \|_2 \cdot \inner{\bphihat}{\btheta} \\
& \overset{(b)}{\le}  \frac{2}{\sqrt{H}} \| \btheta \|_2 \cdot \epsilon \\
& \overset{(c)}{=} 2 \sqrt{d \mu^2} \cdot \epsilon
\end{align*}
where $(a)$ uses that $\| \bphihat \|_2^2 = \sum_{h=1}^H \| \bphihat_h \|_2^2 = H$, $(b)$ follows from \eqref{eq:minimax_lb_policy_good}, and $(c)$ uses $\| \btheta \|_2 = \sqrt{d H \mu^2}$.
\end{proof}

\begin{lemma}\label{lem:minimax_lb_bandit}
Fix $\epsilon >0$, $d > 1$, and $T \ge d^2$. Consider running some (possibly adaptive) algorithm which stops at some (possibly random) stopping time $\tau$ and outputs $(\bphihat_h)_{h=1}^H$. Let $\cE$ denote the event
\begin{align*}
\cE := \{ \tau \le T \text{ and $(\bphihat_h)_{h=1}^H$ is $\epsilon$-optimal} \}.
\end{align*}
Then unless $T \ge c \cdot \frac{d^2 H^2}{\epsilon^2}$, there exists some instance $\btheta \in \Theta^H$ such that $\Pr_{\btheta}[\cE] \ge 1/10$. 
\end{lemma}
\begin{proof}
Set $\mu = \sqrt{d/700T}$ and let $\bthetahat_h = \sqrt{d \mu^2} \cdot \bphihat_h$. Note that $\| \bthetahat \|_2 \le d \sqrt{H/700T}$. Then
\begin{align*}
\Exp_{\btheta}[\| \bthetahat - \btheta \|_2^2] & = \Exp_{\btheta}[\| \bthetahat - \btheta \|_2^2 \cdot \I \{ \cE \} + \| \bthetahat - \btheta \|_2^2 \cdot \I \{ \cE^c \}] \\
& \le 2 \sqrt{\frac{d^2}{700T}} \cdot \epsilon + \frac{4 d^2 H}{700 T} \cdot \Pr_{\btheta}[\cE^c]
\end{align*}
where the inequality follows from \Cref{lem:minimax_lb_bandit_est_to_policy} since on $\cE$, $\bphihat$ is $\epsilon$-optimal. However, by \Cref{lem:minimax_reg_lb}, we know that there exists some $\btheta$ such that
\begin{align*}
\Exp_{\btheta}[\| \btheta - \bthetahat \|_2^2] \ge  \frac{d H \mu^2}{2} \left ( 1 - \sqrt{\frac{20 T \mu^2}{d} } \right ) \ge 0.00059 \cdot \frac{d^2 H}{T}. 
\end{align*}
However, this is a contradiction unless
\begin{align*}
2 \sqrt{\frac{d^2}{700T}} \cdot \epsilon + \frac{4 d^2 H}{700 T}\cdot \Pr_{\btheta}[\cE^c] \ge 0.00059 \cdot \frac{d^2 H}{T} \iff \Pr_{\btheta}[\cE^c]  \ge 0.10325 - \frac{\sqrt{700 T}}{2dH} \cdot \epsilon.
\end{align*}
It follows that it
\begin{align*}
0.10325 - \frac{\sqrt{700 T}}{2dH} \cdot \epsilon \ge 0.1 \iff ( \frac{2 \cdot 0.00325}{\sqrt{700}} )^2 \cdot \frac{d^2 H^2}{\epsilon^2} \ge T
\end{align*}
we have that $\Pr_{\btheta}[\cE^c] \ge 0.1$.
\end{proof}

\subsubsection{Mapping to Linear MDPs}
Consider a single-state, $H$-step, $d+1$-dimensional linear MDP specified by:
\begin{align*}
& \bphi(s,\ba) = [\ba/2, 1/2], \quad \ba \in \cS^{d-1} \cup \{ \zeros \} \\
& \bthetatil_h = [\btheta_h, 1], \quad \btheta_h \in \Theta := \{ -\mu, \mu \}^d \\
& \bmu_h(s) = [\zeros, 2].
\end{align*}
We take the reward distribution to be $r_h(s,a) \sim \bern(\inner{\bphi(s,a)}{\btheta_h})$.
It is straightforward to see that this satisfies the definition of a linear MDP, as given in \Cref{defn:linear_mdp}.

Assume we have access to the regression setting defined above. Then we can simulate this linear MDP with the regression setting by simply, at episode $k$ step $h$, choosing some action $\ba$, and using the regression setting to observe a reward $r \sim \bern(\inner{\ba}{\btheta_h} + 1/2)$, and then transitioning to state $s$ and step $h+1$. The following lemma shows that a policy which is near-optimal on the linear MDP induces a near-optimal $\bphihat$ in the regression setting.

\begin{lemma}\label{lem:minimax_bandit_to_mdp}
Assume that $\pihat$ is $\epsilon$-optimal in the linear MDP defined above, with rewards $(\bthetatil_h)_{h=1}^H$. Then $(\bphihat_{\pihat,h})_{h=1}^H$ is $\epsilon$-optimal in the multi-dimensional linear bandit setting with rewards $(\btheta_h)_{h=1}^H$.
\end{lemma}
\begin{proof}
Note that the value of any policy $\pi$ is given by
\begin{align*}
\sum_{h=1}^H \inner{\bphi_{\pi,h}}{\bthetatil_h} = H/2 + \frac{1}{2} \sum_{h=1}^H \inner{\ba_{\pi,h}}{\btheta_h}
\end{align*}
where $\ba_{\pi,h}$ denotes the first $d$ coordinates of $\bphi_{\pi,h}$. It follows that it $\pihat$ is $\epsilon$-optimal, it must be the case that
\begin{align*}
\frac{1}{2} \sum_{h=1}^H \inner{\ba_{\pihat,h}}{\btheta_h} \ge \sup_{\ba_1,\ldots,\ba_H \in \cS^{d-1}} \frac{1}{2} \sum_{h=1}^H \inner{\ba_{h}}{\btheta_h} - \epsilon.
\end{align*}
However, this is precisely the definition of a $2\epsilon$-optimal policy in the multi-dimensional linear bandit setting with rewards $(\btheta_h)_{h=1}^H$.
\end{proof}

\begin{lemma}\label{lem:minimax_lb_mdp}
Fix $\epsilon >0$, $d > 1$, and $K \ge d^2$. Consider running some (possibly adaptive) algorithm which stops at some (possibly random) stopping time $\tau$ in a $(d+1)$-dimensional linear MDP, and outputs some policy $\pihat$. Let $\cE$ denote the event
\begin{align*}
\cE := \{ \tau \le K \text{ and $\pihat$ is $\epsilon$-optimal} \}.
\end{align*}
Then unless $K \ge c \cdot \frac{d^2 H^2}{\epsilon^2}$, there exists some MDP $\cM$ such that $\Pr_{\cM}[\cE] \ge 1/10$. 
\end{lemma}
\begin{proof}
As noted, we can simulate the linear MDP defined above using only access to a multi-dimensional linear bandit setting. Since \Cref{lem:minimax_bandit_to_mdp} implies that finding an $\epsilon$-optimal policy in our linear MDP is equivalent to find a $2\epsilon$-optimal set of vectors in our multi-dimensional linear bandit, the lower bound for multi-dimensional linear bandits, \Cref{lem:minimax_lb_bandit}, must hold here, which immediately gives the result.
\end{proof}

\begin{proof}[Proof of \Cref{prop:minimax_lb}]
Let $\frakDoff$ be the dataset obtained from running $\Toff$ times on the linear MDP defined above, and at each step $h$ playing each $[\be_i/2,1/2]$ $\Toff/2d$ times, and $\be_{d+1}/2$ $\Toff/2d$ times. Note that in this case we have 
\begin{align*}
\bLamoff^h = \frac{\Toff}{8d} \begin{bmatrix} I & \bm{1} \\ \bm{1}^\top & 2d \end{bmatrix}.
\end{align*}

Note that
\begin{align}
N_h(\frakDoff) & \le \min_{N \ge 0} N \quad \text{s.t.} \quad \inf_{\bLambda \in \bOmega_h} \sup_\pi \| \bphi_{\pi,h} \|_{(N \bLambda + \bLamoff^h)^{-1}}^2 \le \frac{c' \cdot \epsilon^2}{d H} \nonumber \\
& \le \min_{N \ge 0} N \quad \text{s.t.} \quad \inf_{\bLambda \in \bOmega_h} \sup_\pi \| \bphi_{\pi,h} \|_2^2 \| (N \bLambda + \bLamoff^h)^{-1} \|_\op \le \frac{c' \cdot \epsilon^2}{d H} \nonumber \\
& \le \min_{N \ge 0} N \quad \text{s.t.} \quad \inf_{\bLambda \in \bOmega_h}  \| (N \bLambda + \bLamoff^h)^{-1} \|_\op \le \frac{c' \cdot \epsilon^2}{d H}. \label{eq:minimax_lb_eq1}
\end{align}
Now take 
\begin{align*}
\bLambda = \frac{1}{8d} \begin{bmatrix} I & \bm{1} \\ \bm{1}^\top & 2d \end{bmatrix},
\end{align*}
(note that this is a valid setting of $\bLambda$, and can be constructed analogously to $\frakDoff$), and with this choice we can bound
\begin{align}\label{eq:minimax_lb_eq2}
\eqref{eq:minimax_lb_eq1} \le \min_{N \ge 0} N \quad \text{s.t.} \quad \frac{8d}{N + \Toff} \left \|  \begin{bmatrix} I & \bm{1} \\ \bm{1}^\top & 2d \end{bmatrix}^{-1} \right \|_\op \le \frac{c' \cdot \epsilon^2}{d H}.
\end{align}
Using the formula for the inverse of a block matrix, we have
\begin{align*}
\begin{bmatrix} I & \bm{1} \\ \bm{1}^\top & 2d \end{bmatrix}^{-1} = \begin{bmatrix} I + \frac{1}{d} \ones \ones^\top & - \frac{1}{d} \ones \\
-\frac{1}{d} \ones^\top & \frac{1}{d} \end{bmatrix}.
\end{align*}
Take $\bv \in \cS^{d-1}$, and write $\bv = [\bvtil, b]$. Then,
\begin{align*}
\bv^\top \begin{bmatrix} I + \frac{1}{d} \ones \ones^\top & - \frac{1}{d} \ones \\
-\frac{1}{d} \ones^\top & \frac{1}{d} \end{bmatrix} \bv = \bvtil^\top \bvtil + \frac{1}{d} (\ones^\top \bvtil)^2 - \frac{2b}{d} \ones^\top \bvtil + \frac{b^2}{d} \le 5.
\end{align*}
Using this, we can bound
\begin{align*}
\eqref{eq:minimax_lb_eq2} \le \min_{N \ge 0} N \quad \text{s.t.} \quad \frac{40d}{N + \Toff}  \le \frac{c' \cdot \epsilon^2}{d H} = \max \{ \frac{40 d^2 H}{c' \epsilon^2} - \Toff, 0 \}.
\end{align*}
It follows that 
\begin{align*}
 \sum_{h=1}^H N_h(\frakDoff)  \le  H \max \{ \frac{40 d^2 H}{c' \epsilon^2} - \Toff, 0 \}  = \max \{ \frac{40 d^2 H^2}{c' \epsilon^2} - H \Toff, 0 \}  \le \max \{ \frac{40 d^2 H^2}{c' \epsilon^2} -  \Toff, 0 \}  .
\end{align*}
By \Cref{lem:minimax_lb_mdp}, we know that we need to collect at least $\frac{c d^2 H^2}{\epsilon^2}$ episodes on some MDP in our class or we will fail to return an $\epsilon$-optimal policy with constant probability. Note that this lower bound is agnostic to how these episodes were collected---they could be either offline (which we can think of as equivalent to just running a non-adaptive query policy) or online. Thus, if we have already collected $\Toff$ offline episodes, we must collect at least $\max \{ \frac{c d^2 H^2}{\epsilon^2} - \Toff, 0 \}$ online episodes. From what we have just shown, though, 
\begin{align*}
 \sum_{h=1}^H N_h(\frakDoff)   \le \max \{ \frac{40 d^2 H^2}{c' \epsilon^2} -  \Toff, 0 \}  ,
\end{align*}
so the result follows by \Cref{lem:minimax_lb_mdp} and proper setting of $c'$.
\end{proof}

%% file: main.bbl
\begin{thebibliography}{85}
\providecommand{\natexlab}[1]{#1}
\providecommand{\url}[1]{\texttt{#1}}
\expandafter\ifx\csname urlstyle\endcsname\relax
  \providecommand{\doi}[1]{doi: #1}\else
  \providecommand{\doi}{doi: \begingroup \urlstyle{rm}\Url}\fi

\bibitem[Abbasi-Yadkori et~al.(2011)Abbasi-Yadkori, P{\'a}l, and
  Szepesv{\'a}ri]{abbasi2011improved}
Abbasi-Yadkori, Y., P{\'a}l, D., and Szepesv{\'a}ri, C.
\newblock Improved algorithms for linear stochastic bandits.
\newblock \emph{Advances in neural information processing systems}, 24, 2011.

\bibitem[Agarwal et~al.(2021)Agarwal, Chaudhuri, Jain, Nagaraj, and
  Netrapalli]{agarwal2021online}
Agarwal, N., Chaudhuri, S., Jain, P., Nagaraj, D., and Netrapalli, P.
\newblock Online target q-learning with reverse experience replay: Efficiently
  finding the optimal policy for linear mdps.
\newblock \emph{arXiv preprint arXiv:2110.08440}, 2021.

\bibitem[Agrawal \& Jia(2017)Agrawal and Jia]{agrawal2017optimistic}
Agrawal, S. and Jia, R.
\newblock Optimistic posterior sampling for reinforcement learning: worst-case
  regret bounds.
\newblock \emph{Advances in Neural Information Processing Systems}, 30, 2017.

\bibitem[Antos et~al.(2008)Antos, Szepesv{\'a}ri, and Munos]{antos2008learning}
Antos, A., Szepesv{\'a}ri, C., and Munos, R.
\newblock Learning near-optimal policies with bellman-residual minimization
  based fitted policy iteration and a single sample path.
\newblock \emph{Machine Learning}, 71\penalty0 (1):\penalty0 89--129, 2008.

\bibitem[Auer et~al.(2008)Auer, Jaksch, and Ortner]{auer2008near}
Auer, P., Jaksch, T., and Ortner, R.
\newblock Near-optimal regret bounds for reinforcement learning.
\newblock \emph{Advances in neural information processing systems}, 21, 2008.

\bibitem[Ayoub et~al.(2020)Ayoub, Jia, Szepesvari, Wang, and
  Yang]{ayoub2020model}
Ayoub, A., Jia, Z., Szepesvari, C., Wang, M., and Yang, L.
\newblock Model-based reinforcement learning with value-targeted regression.
\newblock In \emph{International Conference on Machine Learning}, pp.\
  463--474. PMLR, 2020.

\bibitem[Azar et~al.(2017)Azar, Osband, and Munos]{azar2017minimax}
Azar, M.~G., Osband, I., and Munos, R.
\newblock Minimax regret bounds for reinforcement learning.
\newblock In \emph{International Conference on Machine Learning}, pp.\
  263--272. PMLR, 2017.

\bibitem[Ball et~al.(2023)Ball, Smith, Kostrikov, and
  Levine]{ball2023efficient}
Ball, P.~J., Smith, L., Kostrikov, I., and Levine, S.
\newblock Efficient online reinforcement learning with offline data.
\newblock \emph{arXiv preprint arXiv:2302.02948}, 2023.

\bibitem[Brafman \& Tennenholtz(2002)Brafman and Tennenholtz]{brafman2002r}
Brafman, R.~I. and Tennenholtz, M.
\newblock R-max-a general polynomial time algorithm for near-optimal
  reinforcement learning.
\newblock \emph{Journal of Machine Learning Research}, 3\penalty0
  (Oct):\penalty0 213--231, 2002.

\bibitem[Brown et~al.(2020)Brown, Mann, Ryder, Subbiah, Kaplan, Dhariwal,
  Neelakantan, Shyam, Sastry, Askell, et~al.]{brown2020language}
Brown, T., Mann, B., Ryder, N., Subbiah, M., Kaplan, J.~D., Dhariwal, P.,
  Neelakantan, A., Shyam, P., Sastry, G., Askell, A., et~al.
\newblock Language models are few-shot learners.
\newblock \emph{Advances in neural information processing systems},
  33:\penalty0 1877--1901, 2020.

\bibitem[Chang et~al.(2021)Chang, Uehara, Sreenivas, Kidambi, and
  Sun]{chang2021mitigating}
Chang, J., Uehara, M., Sreenivas, D., Kidambi, R., and Sun, W.
\newblock Mitigating covariate shift in imitation learning via offline data
  with partial coverage.
\newblock \emph{Advances in Neural Information Processing Systems},
  34:\penalty0 965--979, 2021.

\bibitem[Chen \& Jiang(2019)Chen and Jiang]{chen2019information}
Chen, J. and Jiang, N.
\newblock Information-theoretic considerations in batch reinforcement learning.
\newblock In \emph{International Conference on Machine Learning}, pp.\
  1042--1051. PMLR, 2019.

\bibitem[Chen \& Jiang(2022)Chen and Jiang]{chen2022offline}
Chen, J. and Jiang, N.
\newblock Offline reinforcement learning under value and density-ratio
  realizability: the power of gaps.
\newblock \emph{arXiv preprint arXiv:2203.13935}, 2022.

\bibitem[Dann et~al.(2017)Dann, Lattimore, and Brunskill]{dann2017unifying}
Dann, C., Lattimore, T., and Brunskill, E.
\newblock Unifying pac and regret: Uniform pac bounds for episodic
  reinforcement learning.
\newblock \emph{Advances in Neural Information Processing Systems}, 30, 2017.

\bibitem[Dann et~al.(2021)Dann, Marinov, Mohri, and Zimmert]{dann2021beyond}
Dann, C., Marinov, T.~V., Mohri, M., and Zimmert, J.
\newblock Beyond value-function gaps: Improved instance-dependent regret bounds
  for episodic reinforcement learning.
\newblock \emph{Advances in Neural Information Processing Systems},
  34:\penalty0 1--12, 2021.

\bibitem[Du et~al.(2021)Du, Kakade, Lee, Lovett, Mahajan, Sun, and
  Wang]{du2021bilinear}
Du, S., Kakade, S., Lee, J., Lovett, S., Mahajan, G., Sun, W., and Wang, R.
\newblock Bilinear classes: A structural framework for provable generalization
  in rl.
\newblock In \emph{International Conference on Machine Learning}, pp.\
  2826--2836. PMLR, 2021.

\bibitem[Fiez et~al.(2019)Fiez, Jain, Jamieson, and
  Ratliff]{fiez2019sequential}
Fiez, T., Jain, L., Jamieson, K.~G., and Ratliff, L.
\newblock Sequential experimental design for transductive linear bandits.
\newblock \emph{Advances in neural information processing systems}, 32, 2019.

\bibitem[Foster et~al.(2021)Foster, Kakade, Qian, and
  Rakhlin]{foster2021statistical}
Foster, D.~J., Kakade, S.~M., Qian, J., and Rakhlin, A.
\newblock The statistical complexity of interactive decision making.
\newblock \emph{arXiv preprint arXiv:2112.13487}, 2021.

\bibitem[Fujimoto et~al.(2019)Fujimoto, Meger, and Precup]{fujimoto2019off}
Fujimoto, S., Meger, D., and Precup, D.
\newblock Off-policy deep reinforcement learning without exploration.
\newblock In \emph{International conference on machine learning}, pp.\
  2052--2062. PMLR, 2019.

\bibitem[Hao et~al.(2021)Hao, Lattimore, Szepesv{\'a}ri, and
  Wang]{hao2021online}
Hao, B., Lattimore, T., Szepesv{\'a}ri, C., and Wang, M.
\newblock Online sparse reinforcement learning.
\newblock In \emph{International Conference on Artificial Intelligence and
  Statistics}, pp.\  316--324. PMLR, 2021.

\bibitem[He et~al.(2020)He, Zhou, and Gu]{he2020logarithmic}
He, J., Zhou, D., and Gu, Q.
\newblock Logarithmic regret for reinforcement learning with linear function
  approximation.
\newblock \emph{arXiv preprint arXiv:2011.11566}, 2020.

\bibitem[Hester et~al.(2018)Hester, Vecerik, Pietquin, Lanctot, Schaul, Piot,
  Horgan, Quan, Sendonaris, Osband, et~al.]{hester2018deep}
Hester, T., Vecerik, M., Pietquin, O., Lanctot, M., Schaul, T., Piot, B.,
  Horgan, D., Quan, J., Sendonaris, A., Osband, I., et~al.
\newblock Deep q-learning from demonstrations.
\newblock In \emph{Proceedings of the AAAI Conference on Artificial
  Intelligence}, volume~32, 2018.

\bibitem[Jiang \& Huang(2020)Jiang and Huang]{jiang2020minimax}
Jiang, N. and Huang, J.
\newblock Minimax value interval for off-policy evaluation and policy
  optimization.
\newblock \emph{Advances in Neural Information Processing Systems},
  33:\penalty0 2747--2758, 2020.

\bibitem[Jin et~al.(2018)Jin, Allen-Zhu, Bubeck, and Jordan]{jin2018q}
Jin, C., Allen-Zhu, Z., Bubeck, S., and Jordan, M.~I.
\newblock Is q-learning provably efficient?
\newblock \emph{Advances in neural information processing systems}, 31, 2018.

\bibitem[Jin et~al.(2020)Jin, Yang, Wang, and Jordan]{jin2020provably}
Jin, C., Yang, Z., Wang, Z., and Jordan, M.~I.
\newblock Provably efficient reinforcement learning with linear function
  approximation.
\newblock In \emph{Conference on Learning Theory}, pp.\  2137--2143. PMLR,
  2020.

\bibitem[Jin et~al.(2021{\natexlab{a}})Jin, Liu, and
  Miryoosefi]{jin2021bellman}
Jin, C., Liu, Q., and Miryoosefi, S.
\newblock Bellman eluder dimension: New rich classes of rl problems, and
  sample-efficient algorithms.
\newblock \emph{arXiv preprint arXiv:2102.00815}, 2021{\natexlab{a}}.

\bibitem[Jin et~al.(2021{\natexlab{b}})Jin, Yang, and Wang]{jin2021pessimism}
Jin, Y., Yang, Z., and Wang, Z.
\newblock Is pessimism provably efficient for offline rl?
\newblock In \emph{International Conference on Machine Learning}, pp.\
  5084--5096. PMLR, 2021{\natexlab{b}}.

\bibitem[Kakade(2003)]{kakade2003sample}
Kakade, S.~M.
\newblock \emph{On the sample complexity of reinforcement learning}.
\newblock University of London, University College London (United Kingdom),
  2003.

\bibitem[Katz-Samuels \& Jamieson(2020)Katz-Samuels and Jamieson]{katz2020true}
Katz-Samuels, J. and Jamieson, K.
\newblock The true sample complexity of identifying good arms.
\newblock In \emph{International Conference on Artificial Intelligence and
  Statistics}, pp.\  1781--1791. PMLR, 2020.

\bibitem[Kaufmann et~al.(2016)Kaufmann, Capp{\'e}, and
  Garivier]{kaufmann2016complexity}
Kaufmann, E., Capp{\'e}, O., and Garivier, A.
\newblock On the complexity of best-arm identification in multi-armed bandit
  models.
\newblock \emph{The Journal of Machine Learning Research}, 17\penalty0
  (1):\penalty0 1--42, 2016.

\bibitem[Kearns \& Singh(2002)Kearns and Singh]{kearns2002near}
Kearns, M. and Singh, S.
\newblock Near-optimal reinforcement learning in polynomial time.
\newblock \emph{Machine learning}, 49\penalty0 (2):\penalty0 209--232, 2002.

\bibitem[Kidambi et~al.(2020)Kidambi, Rajeswaran, Netrapalli, and
  Joachims]{kidambi2020morel}
Kidambi, R., Rajeswaran, A., Netrapalli, P., and Joachims, T.
\newblock Morel: Model-based offline reinforcement learning.
\newblock \emph{Advances in neural information processing systems},
  33:\penalty0 21810--21823, 2020.

\bibitem[Kober et~al.(2013)Kober, Bagnell, and Peters]{kober2013reinforcement}
Kober, J., Bagnell, J.~A., and Peters, J.
\newblock Reinforcement learning in robotics: A survey.
\newblock \emph{The International Journal of Robotics Research}, 32\penalty0
  (11):\penalty0 1238--1274, 2013.

\bibitem[Kumar et~al.(2019)Kumar, Fu, Soh, Tucker, and
  Levine]{kumar2019stabilizing}
Kumar, A., Fu, J., Soh, M., Tucker, G., and Levine, S.
\newblock Stabilizing off-policy q-learning via bootstrapping error reduction.
\newblock \emph{Advances in Neural Information Processing Systems}, 32, 2019.

\bibitem[Kumar et~al.(2020)Kumar, Zhou, Tucker, and
  Levine]{kumar2020conservative}
Kumar, A., Zhou, A., Tucker, G., and Levine, S.
\newblock Conservative q-learning for offline reinforcement learning.
\newblock \emph{Advances in Neural Information Processing Systems},
  33:\penalty0 1179--1191, 2020.

\bibitem[Levine et~al.(2020)Levine, Kumar, Tucker, and Fu]{levine2020offline}
Levine, S., Kumar, A., Tucker, G., and Fu, J.
\newblock Offline reinforcement learning: Tutorial, review, and perspectives on
  open problems.
\newblock \emph{arXiv preprint arXiv:2005.01643}, 2020.

\bibitem[Liu et~al.(2020)Liu, Swaminathan, Agarwal, and
  Brunskill]{liu2020provably}
Liu, Y., Swaminathan, A., Agarwal, A., and Brunskill, E.
\newblock Provably good batch off-policy reinforcement learning without great
  exploration.
\newblock \emph{Advances in neural information processing systems},
  33:\penalty0 1264--1274, 2020.

\bibitem[Munos \& Szepesv{\'a}ri(2008)Munos and
  Szepesv{\'a}ri]{munos2008finite}
Munos, R. and Szepesv{\'a}ri, C.
\newblock Finite-time bounds for fitted value iteration.
\newblock \emph{Journal of Machine Learning Research}, 9\penalty0 (5), 2008.

\bibitem[Nair et~al.(2018)Nair, McGrew, Andrychowicz, Zaremba, and
  Abbeel]{nair2018overcoming}
Nair, A., McGrew, B., Andrychowicz, M., Zaremba, W., and Abbeel, P.
\newblock Overcoming exploration in reinforcement learning with demonstrations.
\newblock In \emph{2018 IEEE international conference on robotics and
  automation (ICRA)}, pp.\  6292--6299. IEEE, 2018.

\bibitem[Nakamoto et~al.(2023)Nakamoto, Zhai, Singh, Mark, Ma, Finn, Kumar, and
  Levine]{nakamoto2023cal}
Nakamoto, M., Zhai, Y., Singh, A., Mark, M.~S., Ma, Y., Finn, C., Kumar, A.,
  and Levine, S.
\newblock Cal-ql: Calibrated offline rl pre-training for efficient online
  fine-tuning.
\newblock \emph{arXiv preprint arXiv:2303.05479}, 2023.

\bibitem[Ok et~al.(2018)Ok, Proutiere, and Tranos]{ok2018exploration}
Ok, J., Proutiere, A., and Tranos, D.
\newblock Exploration in structured reinforcement learning.
\newblock \emph{Advances in Neural Information Processing Systems}, 31, 2018.

\bibitem[Pacchiano et~al.(2021)Pacchiano, Ball, Parker-Holder, Choromanski, and
  Roberts]{pacchiano2021towards}
Pacchiano, A., Ball, P., Parker-Holder, J., Choromanski, K., and Roberts, S.
\newblock Towards tractable optimism in model-based reinforcement learning.
\newblock In \emph{Uncertainty in Artificial Intelligence}, pp.\  1413--1423.
  PMLR, 2021.

\bibitem[Rajeswaran et~al.(2017)Rajeswaran, Kumar, Gupta, Vezzani, Schulman,
  Todorov, and Levine]{rajeswaran2017learning}
Rajeswaran, A., Kumar, V., Gupta, A., Vezzani, G., Schulman, J., Todorov, E.,
  and Levine, S.
\newblock Learning complex dexterous manipulation with deep reinforcement
  learning and demonstrations.
\newblock \emph{arXiv preprint arXiv:1709.10087}, 2017.

\bibitem[Ramesh et~al.(2021)Ramesh, Pavlov, Goh, Gray, Voss, Radford, Chen, and
  Sutskever]{ramesh2021zero}
Ramesh, A., Pavlov, M., Goh, G., Gray, S., Voss, C., Radford, A., Chen, M., and
  Sutskever, I.
\newblock Zero-shot text-to-image generation.
\newblock In \emph{International Conference on Machine Learning}, pp.\
  8821--8831. PMLR, 2021.

\bibitem[Rashidinejad et~al.(2021)Rashidinejad, Zhu, Ma, Jiao, and
  Russell]{rashidinejad2021bridging}
Rashidinejad, P., Zhu, B., Ma, C., Jiao, J., and Russell, S.
\newblock Bridging offline reinforcement learning and imitation learning: A
  tale of pessimism.
\newblock \emph{Advances in Neural Information Processing Systems},
  34:\penalty0 11702--11716, 2021.

\bibitem[Ren et~al.(2022)Ren, Zhang, Szepesv{\'a}ri, and Dai]{ren2022free}
Ren, T., Zhang, T., Szepesv{\'a}ri, C., and Dai, B.
\newblock A free lunch from the noise: Provable and practical exploration for
  representation learning.
\newblock In \emph{Uncertainty in Artificial Intelligence}, pp.\  1686--1696.
  PMLR, 2022.

\bibitem[Ross \& Bagnell(2012)Ross and Bagnell]{ross2012agnostic}
Ross, S. and Bagnell, J.~A.
\newblock Agnostic system identification for model-based reinforcement
  learning.
\newblock \emph{arXiv preprint arXiv:1203.1007}, 2012.

\bibitem[Shamir(2013)]{shamir2013complexity}
Shamir, O.
\newblock On the complexity of bandit and derivative-free stochastic convex
  optimization.
\newblock In \emph{Conference on Learning Theory}, pp.\  3--24. PMLR, 2013.

\bibitem[Silver et~al.(2016)Silver, Huang, Maddison, Guez, Sifre, Van
  Den~Driessche, Schrittwieser, Antonoglou, Panneershelvam, Lanctot,
  et~al.]{silver2016mastering}
Silver, D., Huang, A., Maddison, C.~J., Guez, A., Sifre, L., Van Den~Driessche,
  G., Schrittwieser, J., Antonoglou, I., Panneershelvam, V., Lanctot, M.,
  et~al.
\newblock Mastering the game of go with deep neural networks and tree search.
\newblock \emph{nature}, 529\penalty0 (7587):\penalty0 484--489, 2016.

\bibitem[Simchowitz \& Jamieson(2019)Simchowitz and
  Jamieson]{simchowitz2019non}
Simchowitz, M. and Jamieson, K.~G.
\newblock Non-asymptotic gap-dependent regret bounds for tabular mdps.
\newblock \emph{Advances in Neural Information Processing Systems}, 32, 2019.

\bibitem[Song et~al.(2022)Song, Zhou, Sekhari, Bagnell, Krishnamurthy, and
  Sun]{song2022hybrid}
Song, Y., Zhou, Y., Sekhari, A., Bagnell, J.~A., Krishnamurthy, A., and Sun, W.
\newblock Hybrid rl: Using both offline and online data can make rl efficient.
\newblock \emph{arXiv preprint arXiv:2210.06718}, 2022.

\bibitem[Tennenholtz et~al.(2021)Tennenholtz, Shalit, Mannor, and
  Efroni]{tennenholtz2021bandits}
Tennenholtz, G., Shalit, U., Mannor, S., and Efroni, Y.
\newblock Bandits with partially observable confounded data.
\newblock In \emph{Uncertainty in Artificial Intelligence}, pp.\  430--439.
  PMLR, 2021.

\bibitem[Tewari \& Bartlett(2007)Tewari and Bartlett]{tewari2007optimistic}
Tewari, A. and Bartlett, P.
\newblock Optimistic linear programming gives logarithmic regret for
  irreducible mdps.
\newblock \emph{Advances in Neural Information Processing Systems}, 20, 2007.

\bibitem[Tirinzoni et~al.(2022)Tirinzoni, Al-Marjani, and
  Kaufmann]{tirinzoni2022optimistic}
Tirinzoni, A., Al-Marjani, A., and Kaufmann, E.
\newblock Optimistic pac reinforcement learning: the instance-dependent view.
\newblock \emph{arXiv preprint arXiv:2207.05852}, 2022.

\bibitem[Tsybakov(2009)]{tsybakov2009introduction}
Tsybakov, A.~B.
\newblock Introduction to nonparametric estimation., 2009.

\bibitem[Uehara \& Sun(2021)Uehara and Sun]{uehara2021pessimistic}
Uehara, M. and Sun, W.
\newblock Pessimistic model-based offline reinforcement learning under partial
  coverage.
\newblock \emph{arXiv preprint arXiv:2107.06226}, 2021.

\bibitem[Wagenmaker \& Foster(2023)Wagenmaker and
  Foster]{wagenmaker2023instance}
Wagenmaker, A. and Foster, D.~J.
\newblock Instance-optimality in interactive decision making: Toward a
  non-asymptotic theory.
\newblock \emph{arXiv preprint arXiv:2304.12466}, 2023.

\bibitem[Wagenmaker \& Jamieson(2022)Wagenmaker and
  Jamieson]{wagenmaker2022instance}
Wagenmaker, A. and Jamieson, K.
\newblock Instance-dependent near-optimal policy identification in linear mdps
  via online experiment design.
\newblock \emph{arXiv preprint arXiv:2207.02575}, 2022.

\bibitem[Wagenmaker et~al.(2022{\natexlab{a}})Wagenmaker, Chen, Simchowitz, Du,
  and Jamieson]{wagenmaker2022first}
Wagenmaker, A.~J., Chen, Y., Simchowitz, M., Du, S., and Jamieson, K.
\newblock First-order regret in reinforcement learning with linear function
  approximation: A robust estimation approach.
\newblock In \emph{International Conference on Machine Learning}, pp.\
  22384--22429. PMLR, 2022{\natexlab{a}}.

\bibitem[Wagenmaker et~al.(2022{\natexlab{b}})Wagenmaker, Chen, Simchowitz, Du,
  and Jamieson]{wagenmaker2022reward}
Wagenmaker, A.~J., Chen, Y., Simchowitz, M., Du, S., and Jamieson, K.
\newblock Reward-free rl is no harder than reward-aware rl in linear markov
  decision processes.
\newblock In \emph{International Conference on Machine Learning}, pp.\
  22430--22456. PMLR, 2022{\natexlab{b}}.

\bibitem[Wagenmaker et~al.(2022{\natexlab{c}})Wagenmaker, Simchowitz, and
  Jamieson]{wagenmaker2022beyond}
Wagenmaker, A.~J., Simchowitz, M., and Jamieson, K.
\newblock Beyond no regret: Instance-dependent pac reinforcement learning.
\newblock In \emph{Conference on Learning Theory}, pp.\  358--418. PMLR,
  2022{\natexlab{c}}.

\bibitem[Weisz et~al.(2021)Weisz, Amortila, and
  Szepesv{\'a}ri]{weisz2021exponential}
Weisz, G., Amortila, P., and Szepesv{\'a}ri, C.
\newblock Exponential lower bounds for planning in mdps with
  linearly-realizable optimal action-value functions.
\newblock In \emph{Algorithmic Learning Theory}, pp.\  1237--1264. PMLR, 2021.

\bibitem[Wu et~al.(2019)Wu, Tucker, and Nachum]{wu2019behavior}
Wu, Y., Tucker, G., and Nachum, O.
\newblock Behavior regularized offline reinforcement learning.
\newblock \emph{arXiv preprint arXiv:1911.11361}, 2019.

\bibitem[Xie \& Jiang(2021)Xie and Jiang]{xie2021batch}
Xie, T. and Jiang, N.
\newblock Batch value-function approximation with only realizability.
\newblock In \emph{International Conference on Machine Learning}, pp.\
  11404--11413. PMLR, 2021.

\bibitem[Xie et~al.(2021{\natexlab{a}})Xie, Cheng, Jiang, Mineiro, and
  Agarwal]{xie2021bellman}
Xie, T., Cheng, C.-A., Jiang, N., Mineiro, P., and Agarwal, A.
\newblock Bellman-consistent pessimism for offline reinforcement learning.
\newblock \emph{Advances in neural information processing systems},
  34:\penalty0 6683--6694, 2021{\natexlab{a}}.

\bibitem[Xie et~al.(2021{\natexlab{b}})Xie, Jiang, Wang, Xiong, and
  Bai]{xie2021policy}
Xie, T., Jiang, N., Wang, H., Xiong, C., and Bai, Y.
\newblock Policy finetuning: Bridging sample-efficient offline and online
  reinforcement learning.
\newblock \emph{Advances in neural information processing systems},
  34:\penalty0 27395--27407, 2021{\natexlab{b}}.

\bibitem[Xie et~al.(2022)Xie, Foster, Bai, Jiang, and Kakade]{xie2022role}
Xie, T., Foster, D.~J., Bai, Y., Jiang, N., and Kakade, S.~M.
\newblock The role of coverage in online reinforcement learning.
\newblock \emph{arXiv preprint arXiv:2210.04157}, 2022.

\bibitem[Xu et~al.(2021)Xu, Ma, and Du]{xu2021fine}
Xu, H., Ma, T., and Du, S.
\newblock Fine-grained gap-dependent bounds for tabular mdps via adaptive
  multi-step bootstrap.
\newblock In \emph{Conference on Learning Theory}, pp.\  4438--4472. PMLR,
  2021.

\bibitem[Yang et~al.(2021)Yang, Yang, and Du]{yang2021q}
Yang, K., Yang, L., and Du, S.
\newblock Q-learning with logarithmic regret.
\newblock In \emph{International Conference on Artificial Intelligence and
  Statistics}, pp.\  1576--1584. PMLR, 2021.

\bibitem[Yang \& Wang(2020)Yang and Wang]{yang2020reinforcement}
Yang, L. and Wang, M.
\newblock Reinforcement learning in feature space: Matrix bandit, kernels, and
  regret bound.
\newblock In \emph{International Conference on Machine Learning}, pp.\
  10746--10756. PMLR, 2020.

\bibitem[Yin et~al.(2021)Yin, Bai, and Wang]{yin2021near}
Yin, M., Bai, Y., and Wang, Y.-X.
\newblock Near-optimal offline reinforcement learning via double variance
  reduction.
\newblock \emph{Advances in neural information processing systems},
  34:\penalty0 7677--7688, 2021.

\bibitem[Yin et~al.(2022)Yin, Duan, Wang, and Wang]{yin2022near}
Yin, M., Duan, Y., Wang, M., and Wang, Y.-X.
\newblock Near-optimal offline reinforcement learning with linear
  representation: Leveraging variance information with pessimism.
\newblock \emph{arXiv preprint arXiv:2203.05804}, 2022.

\bibitem[Yu et~al.(2021)Yu, Liu, Nemati, and Yin]{yu2021reinforcement}
Yu, C., Liu, J., Nemati, S., and Yin, G.
\newblock Reinforcement learning in healthcare: A survey.
\newblock \emph{ACM Computing Surveys (CSUR)}, 55\penalty0 (1):\penalty0 1--36,
  2021.

\bibitem[Yu et~al.(2020)Yu, Thomas, Yu, Ermon, Zou, Levine, Finn, and
  Ma]{yu2020mopo}
Yu, T., Thomas, G., Yu, L., Ermon, S., Zou, J.~Y., Levine, S., Finn, C., and
  Ma, T.
\newblock Mopo: Model-based offline policy optimization.
\newblock \emph{Advances in Neural Information Processing Systems},
  33:\penalty0 14129--14142, 2020.

\bibitem[Zanette \& Brunskill(2019)Zanette and Brunskill]{zanette2019tighter}
Zanette, A. and Brunskill, E.
\newblock Tighter problem-dependent regret bounds in reinforcement learning
  without domain knowledge using value function bounds.
\newblock In \emph{International Conference on Machine Learning}, pp.\
  7304--7312. PMLR, 2019.

\bibitem[Zanette et~al.(2020{\natexlab{a}})Zanette, Brandfonbrener, Brunskill,
  Pirotta, and Lazaric]{zanette2020frequentist}
Zanette, A., Brandfonbrener, D., Brunskill, E., Pirotta, M., and Lazaric, A.
\newblock Frequentist regret bounds for randomized least-squares value
  iteration.
\newblock In \emph{International Conference on Artificial Intelligence and
  Statistics}, pp.\  1954--1964. PMLR, 2020{\natexlab{a}}.

\bibitem[Zanette et~al.(2020{\natexlab{b}})Zanette, Lazaric, Kochenderfer, and
  Brunskill]{zanette2020learning}
Zanette, A., Lazaric, A., Kochenderfer, M., and Brunskill, E.
\newblock Learning near optimal policies with low inherent bellman error.
\newblock In \emph{International Conference on Machine Learning}, pp.\
  10978--10989. PMLR, 2020{\natexlab{b}}.

\bibitem[Zanette et~al.(2020{\natexlab{c}})Zanette, Lazaric, Kochenderfer, and
  Brunskill]{zanette2020provably}
Zanette, A., Lazaric, A., Kochenderfer, M.~J., and Brunskill, E.
\newblock Provably efficient reward-agnostic navigation with linear value
  iteration.
\newblock \emph{Advances in Neural Information Processing Systems},
  33:\penalty0 11756--11766, 2020{\natexlab{c}}.

\bibitem[Zanette et~al.(2021)Zanette, Wainwright, and
  Brunskill]{zanette2021provable}
Zanette, A., Wainwright, M.~J., and Brunskill, E.
\newblock Provable benefits of actor-critic methods for offline reinforcement
  learning.
\newblock \emph{Advances in neural information processing systems},
  34:\penalty0 13626--13640, 2021.

\bibitem[Zhan et~al.(2022)Zhan, Huang, Huang, Jiang, and Lee]{zhan2022offline}
Zhan, W., Huang, B., Huang, A., Jiang, N., and Lee, J.
\newblock Offline reinforcement learning with realizability and single-policy
  concentrability.
\newblock In \emph{Conference on Learning Theory}, pp.\  2730--2775. PMLR,
  2022.

\bibitem[Zhang et~al.(2022{\natexlab{a}})Zhang, Ren, Yang, Gonzalez,
  Schuurmans, and Dai]{zhang2022making}
Zhang, T., Ren, T., Yang, M., Gonzalez, J., Schuurmans, D., and Dai, B.
\newblock Making linear mdps practical via contrastive representation learning.
\newblock In \emph{International Conference on Machine Learning}, pp.\
  26447--26466. PMLR, 2022{\natexlab{a}}.

\bibitem[Zhang et~al.(2022{\natexlab{b}})Zhang, Chen, Zhu, and
  Sun]{zhang2022corruption}
Zhang, X., Chen, Y., Zhu, X., and Sun, W.
\newblock Corruption-robust offline reinforcement learning.
\newblock In \emph{International Conference on Artificial Intelligence and
  Statistics}, pp.\  5757--5773. PMLR, 2022{\natexlab{b}}.

\bibitem[Zheng et~al.(2023)Zheng, Luo, Wei, Song, Li, and
  Jiang]{zheng2023adaptive}
Zheng, H., Luo, X., Wei, P., Song, X., Li, D., and Jiang, J.
\newblock Adaptive policy learning for offline-to-online reinforcement
  learning.
\newblock \emph{arXiv preprint arXiv:2303.07693}, 2023.

\bibitem[Zhou et~al.(2020)Zhou, Gu, and Szepesvari]{zhou2020nearly}
Zhou, D., Gu, Q., and Szepesvari, C.
\newblock Nearly minimax optimal reinforcement learning for linear mixture
  markov decision processes.
\newblock \emph{arXiv preprint arXiv:2012.08507}, 2020.

\bibitem[Zhou et~al.(2021)Zhou, He, and Gu]{zhou2021provably}
Zhou, D., He, J., and Gu, Q.
\newblock Provably efficient reinforcement learning for discounted mdps with
  feature mapping.
\newblock In \emph{International Conference on Machine Learning}, pp.\
  12793--12802. PMLR, 2021.

\end{thebibliography}
